\documentclass[twoside]{article}

\usepackage[preprint]{aistats2027}

\usepackage{dsfont}
\usepackage[utf8]{inputenc} % allow utf-8 input
\usepackage[T1]{fontenc}    % use 8-bit T1 fonts
\usepackage{hyperref}       % hyperlinks
\usepackage{url}            % simple URL typesetting
\usepackage{booktabs}       % professional-quality tables
\usepackage{amsfonts}       % blackboard math symbols
\usepackage{nicefrac}       % compact  symbols for 1/2, etc.
\usepackage{microtype}      % microtypography
\usepackage{xcolor}         % colors

\usepackage{amsmath}
\usepackage{amssymb}
\usepackage{mathtools}
\usepackage{amsthm}

\usepackage{amsmath,amsfonts,bm}

\def\eqref#1{equation~\ref{#1}}
\def\1{\bm{1}}

\DeclareMathAlphabet{\mathsfit}{\encodingdefault}{\sfdefault}{m}{sl}
\SetMathAlphabet{\mathsfit}{bold}{\encodingdefault}{\sfdefault}{bx}{n}

\newcommand{\E}{\mathbb{E}}

\newcommand{\R}{\mathbb{R}}

\usepackage[utf8]{inputenc} 
\usepackage[T1]{fontenc}

\usepackage{color}
\usepackage{xcolor}
\usepackage{tcolorbox}
\usepackage{varwidth}
\usepackage{tabularx}       % Automatic column widths
\usepackage{threeparttable} 
\usepackage[utf8]{inputenc} % allow utf-8 input
\usepackage[T1]{fontenc}    % use 8-bit T1 fonts
\usepackage{hyperref}       % hyperlinks
\usepackage{url}            % simple URL typesetting
\usepackage{booktabs}       % professional-quality tables
\usepackage{amsfonts}       % blackboard math symbols
\usepackage{nicefrac}       % compact symbols for 1/2, etc.
\usepackage{microtype}      % microtypography
\usepackage{comment}
\usepackage{enumerate}
\usepackage{enumitem}
\usepackage{graphicx}
\usepackage[dvipsnames]{xcolor}
\usepackage{wrapfig}
\usepackage{mathabx}
\usepackage{subcaption}
\usepackage{enumitem}
\usepackage{pifont}% http://ctan.org/pkg/pifont
\usepackage{multirow, varwidth}
\usepackage[ruled,vlined]{algorithm2e}
\usepackage{colortbl}
\usepackage{nccmath}
\usepackage{bm}
\usepackage{booktabs, colortbl, xcolor, multirow, array}
\definecolor{thispaper}{RGB}{220, 235, 255}   % soft blue for "This paper" rows
\definecolor{sectionbg}{RGB}{245, 245, 245}   % light gray for assumption rows

\usepackage{xspace} % for algorithm names
\usepackage{multirow, varwidth}
\usepackage{float}
\usepackage{booktabs, makecell}
\usepackage{diagbox}
\usepackage{graphicx}
\usepackage{listings}
\usepackage{subcaption}
\usepackage{stfloats}
\usepackage{xfrac}

\usepackage{tikz}
\usetikzlibrary{tikzmark,calc}

\definecolor{myred}{HTML}{FFDAC1}
\definecolor{myblue}{HTML}{a2edf5}

\newtheorem{proposition}{Proposition}
\newtheorem*{proposition*}{Proposition}
\newtheorem{theorem}{Theorem}
\newtheorem*{theorem*}{Theorem}
\newtheorem{remark}{Remark}
\newtheorem{lemma}{Lemma}
\newtheorem*{lemma*}{Lemma}
\newtheorem{corollary}[theorem]{Corollary}
\newtheorem*{corollary*}{Corollary}

\newcommand{\al}{\alpha}

\definecolor{highlightpink}{RGB}{255, 228, 225}
\definecolor{highlightblue}{RGB}{173, 216, 230}

\newtcbox{\highlightbox}{on line, arc=0pt, outer arc=0pt, colback=highlightpink, colframe=white, boxsep=1mm, left=1mm, right=1mm, top=0.5pt, bottom=0.5pt}

\newtcbox{\highlightboxblue}{on line, arc=0pt, outer arc=0pt, colback=highlightblue, colframe=white, boxsep=1mm, left=1mm, right=1mm, top=0.5pt, bottom=0.5pt}
\usepackage[round]{natbib}

\begin{document}



\twocolumn[

\aistatstitle{Asynchronous Heavy-Tailed Optimization}

\aistatsauthor{ Junfei Sun \And Dixi Yao \And Xuchen Gong}
\aistatsaddress{ University of Chicago \And University of Chicago \And University of Chicago }

\aistatsauthor{ Tahseen Rabbani \And Manzil Zaheer \And Tian Li }
\aistatsaddress{ University of Chicago \And Meta \And University of Chicago } ]

\runningauthor{Junfei Sun, Dixi Yao, Xuchen Gong, Tahseen Rabbani, Manzil Zaheer, Tian Li}

\begin{abstract}
  Heavy-tailed stochastic gradient noise, commonly observed in transformer models, can destabilize the optimization process. Recent works mainly focus on developing and understanding approaches to address heavy-tailed noise in the centralized or distributed, \textit{synchronous} setting, while asynchronous optimization under such noise remains underexplored, especially in the practical settings with local updates.
  In this work, we first investigate two asynchronous communication schemes (server- and client-centric) that handle stragglers under heavy-tailed gradient noise. We further propose and theoretically analyze algorithmic modifications based on delay-aware learning rate scheduling and delay compensation to enhance the performance of asynchronous algorithms. Our convergence guarantees achieve the optimal rate under heavy-tailed noise and improve delay tolerance compared with existing asynchronous approaches. Empirically, our approaches outperform prior synchronous and asynchronous methods in terms of accuracy/runtime trade-offs and are more robust to hyperparameters in both image and language tasks.
\end{abstract}
\section{Introduction}
\label{sec:intro}
\vspace{-0.4em}
% heavy-tailed noise is bad and important, but prior works focus on synchronous settings
In stochastic optimization, heavy-tailed stochastic gradient noise is known to destabilize convergence or even cause divergence both empirically and theoretically~\citep[e.g.,][]{zhang2020adaptive,HighProbAdaGConv}. %\tr{this is the core problem so you should add 2-3 more cited works also documenting this issue} 
They are commonly observed in transformer models that currently dominate the state-of-the-art large generative model architectures. Given the scale of modern datasets and transformer-based models, they are usually optimized (pre-trained or fine-tuned) in a distributed fashion by default. There is increasing interest %\tr{cite examples} 
in addressing the negative impacts of heavy-tailed noise in distributed, \textit{synchronous} settings~\citep{lee2025efficient}, where the progress of worker\footnote{Throughout this paper, we use `worker' and `client' interchangeably.} nodes is aggregated to update the global model.

% it is not clear how it interacts with asynchronous training and if prior clipping-based methods can help asynchronous
Due to the inherently heterogeneous nature of large-scale training infrastructure~\citep[e.g.,][]{meta}, asynchronous aggregation is a promising alternative to synchronous variants to handle stragglers or slow networks~\citep{xie2019asynchronous, zheng2020asynchronousstochasticgradientdescent}. Hence, it has the potential to scale up model training at a much larger scale. Recently, \citep{wu2026optimalasynchronousstochasticnonconvex} explored normalized SGD with momentum that discards updates upon a certain delay threshold in asynchronous training with heavy-tailed noise. However, practical training of large models increasingly relies on local updates \citep{liu2024asynchronous,kim2025halos,douillard2023diloco}, and it remains open how to incorporate information from stale (i.e., delayed)
workers to the global model, under heavy-tailed noise.

% However, it remains underexplored both empirically and theoretically how heavy-tailed noise interplays with asynchronous training, and whether prior clipping-based methods that are developed for synchronous settings can generalize to the asynchronous case. More broadly, it remains open how to incorporate information from stale (i.e., delayed) workers to the global model, under heavy-tailed noise.

% additionally, we propose two modifications (+ local updates)

In this work, we investigate the general setting of asynchronous optimization under heavy-tailed noise. Our framework adopts clipping to handle heavy-tailed noise and also allows worker nodes to apply multiple gradient steps locally and send accumulated model updates to the server node, instead of sending a one-step gradient immediately. We propose two techniques to enhance the performance of asynchronous heavy-tailed optimization. First, model updates from extreme stragglers are obtained from a stale snapshot of the model many rounds ago; hence, naively incorporating them into the current global model (maintained on the server node) is suboptimal. Instead of directly dropping the stale information, we perform delay-aware aggregation, softly rescaling the updates based on the amount of staleness during aggregation. Second, we explicitly recover fresh model updates (accumulated gradient updates) from stale ones sent by the slow clients as if there were no delay. Throughout our analysis and experiments, we consider the common \textit{buffered} setting for asynchronous optimization where each server-side update waits for $M$ client updates ($1 \leq M \leq N$ where $N$ is the total number of clients)~\citep{nguyen2022federated}. When $M=1$, it reduces to the fully asynchronous case. We provide convergence guarantees of the proposed enhancements without assuming bounded gradient variance and show that our approach offers a
better delay tolerance under heavy-tailed noise (we use  `staleness' and `delay' interchangeably in this paper).

\textbf{Contributions.} Our contributions are summarized as follows.
{(1)} We study the problem of client-centric and server-centric asynchronous training with heavy-tailed noise. We propose delay-aware downplaying and delay compensation strategies to better incorporate stale updates from stragglers. {(2)} Theoretically, we present the first convergence results for asynchronous
training with local updates and clipping-based optimizers under heavy-tailed noise that achieve the
optimal rate $O(T^{-\frac{\alpha-1}{3\alpha-2}})$ in terms of the expected (non-squared) gradient norm. Moreover, we provide convergence results of our proposed approach with improved dependencies on the amount of delay. {(3)} We empirically show that our method outperforms vanilla asynchronous learning in terms of accuracy/runtime tradeoffs as well as ease of hyperparameter tuning across benchmark datasets. 
    
\vspace{-0.5em}
\section{Related Work}
\label{sec:related}
\vspace{-0.5em}
\textbf{Heavy-Tailed Optimization.} Heavy-tailed stochastic gradient distributions have been demonstrated to  destabilize the training process both empirically and theoretically~\citep{inprobsupreme,lee2025efficient}, and can often be addressed by variants of clipping-based approaches~\citep{heavytail2,CentralClip,HighProbAdaGConv, zhang2020adaptive, gorbunov2024highprobabilityconvergencecompositedistributed}.
% However, to the best of our knowledge, the effects of heavy-tailed gradients (e.g., with unbounded variance) have not been explored for asynchronous training, which is critical when the infrastructure has heterogeneous hardware or network capabilities. 
For asynchronous training, \citep{wu2026optimalasynchronousstochasticnonconvex} established the optimal rate under heavy-tailed noise for normalized SGD with momentum and a delay threshold without local updates. Our work focuses on a more challenging setting: we consider local algorithms with multiple updates setting and clipping-based optimizers that are momentum-free. Our analysis shows that clipping alone suffices to attain the optimal rate under local updates. Furthermore, we propose and theoretically analyze additional strategies that incorporate stale updates rather than discarding them, allowing for unbounded delays. We empirically show that a prior clipping-based method proposed to address heavy-tailed noise can benefit asynchronous training as well, via limiting the impact of updates from stragglers (Section~\ref{sec:exp}). 

\textbf{Asynchronous Distributed Training.} Asynchronous optimization with many variants and setups has been extensively studied in prior literature for decades~\citep{recht2011hogwild,dean2012large,de2015taming, cohen2021asynchronousstochasticoptimizationrobust, koloskova2022sharperconvergenceguaranteesasynchronous, mishchenko2023asynchronoussgdbeatsminibatch,tyurin2023optimaltimecomplexitiesparallel, tyurin2025tighttimecomplexitiesparallel, maranjyan2025ringmasterasgdasynchronoussgd}. To further hide communication in distributed environments, existing works have also studied asynchronous training with local updates, where worker nodes run multiple (instead of one) gradient steps before sending the updates~\citep{nguyen2022federated,xie2019asynchronous,liu2024asynchronous}. For the full generality of our method, throughout the paper, we consider frameworks with local optimization by default, which involve both \textit{local (inner)} and \textit{global (outer)} optimizers. Such a nested scheme has been used in other prior works on asynchronous training of (large language) models~\citep{liu2024asynchronous,kim2025halos}. \citet{kim2025halos} propose a hierarchical local SGD structure combined with asynchronous training. We differ from this work by proposing to incorporate staleness-aware downplaying and delay compensation and address heavy-tailed noise. Delay compensation has appeared in prior asynchronous distributed SGD works~\citep{zheng2017asynchronous,wang2022asynchronous,guan2017delay}, but it has not been adapted, analyzed, or evaluated in our setting. Our specific staleness-aware downplaying algorithms are different from existing ones that downweight stale updates~\citep{wang2024fadas} and demonstrate theoretically-improved tolerance to extreme delays (Section~\ref{sec:analysis}). For theoretical analysis, we note that the line of work~\citep{tyurin2023optimaltimecomplexitiesparallel, tyurin2025tighttimecomplexitiesparallel, maranjyan2025ringmasterasgdasynchronoussgd} develops a new theoretical framework of \textit{time complexity} to evaluate asynchronous methods. We will base our theoretical guarantees on the traditional \textit{number-of-iterations} bounds, though the time complexity model can be similarly applied in our setting.

%We propose two techniques to correct and compensate for staleness.

\textbf{Notations.}  $M$ is the asynchronous buffer size (server waiting for $M$ client updates to make one global update), and $N$ is the total number of clients. We use $K$ to denote the number of local steps, and use $\odot$ or $(\cdot)^{\odot 2}$ to denote element-wise multiplication.   %$x_{i,k}$ with $i\in [M], k\in [K]$ denotes the local model parameter for client $i$ at local iteration $k$. 

\vspace{-0.5em}
\section{Asynchronous Heavy-Tailed Optimization}
\label{sec:methods}
\vspace{-0.5em}
Throughout the paper, we consider the optimization objective $F(x) := \frac{1}{N} \sum_{i=1}^N F_i(x)$ for model parameters $x\in\mathbb{R}^d$ and client $i$'s loss $F_i(x)$. Here, $F_i(x):=\mathbb{E}_{\xi \sim \mathcal{D}_i}[F_i(x,\xi)]$ where $F_i(x,\xi) := F_i(x)+\langle x, \xi\rangle$ represents the stochastic local objective, and $\mathcal{D}_i$ is the noise distribution. This term comes from integrating the gradient noise  $\nabla F_i(x,\xi_i)=\nabla F_i(x)+\xi_i$, where $\xi_i$ is the stochastic gradient noise. %And this comes from the gradient noise model $\nabla F(x,\xi_i)=\nabla F(x)+\xi_i$ where $\xi_i\sim \mathcal{D}_i$ for $\mathcal{D}_i$ being the local noise distribution for the $i^{th}$ client. 
Importantly, we assume that the stochastic noise can be heavy-tailed, i.e., we have that $\mathbb{E}[\|\xi_i\|^\alpha ]\leq D^\alpha$ for some $D>0$ and $\alpha\in (1,2)$ for any $\xi_i\sim \mathcal{D}_i$.
In this section, we aim to develop algorithms that handle heavy-tailed noise under an asynchronous setting without bounded delays. %. On a high level, we propose an asynchronous framework that uses the clipping method to control the heavy-tailed noise and enables this framework to be aware of the delay of the updates. % introduced by the asynchronous training process. %, which consequently ensures that the algorithms obtain reasonable convergence guarantees under biases both from asynchrony and heavy-tailed noise. 

\textbf{Server- and Client-Centric Asynchronous Models.}
First, we establish two baseline schemes for asynchronous learning. % that also cover most of previous works. 
%We start by raising two ways to incorporate multiple workers' updates per global update on the server side when we deal with heavy-tailed noise under an asynchronous setting. 
%Following previous notation, let the prefixed number of client updates needed for each global update be denoted as $M$. 
We allow for each client to run local updates (as opposed to one iteration to compute gradients) and the server to aggregate the accumulated model updates sent by the clients, which results in  a \textit{nested} optimization framework. To address heavy-tailed noise, we employ a coordinate-wise clipping based optimizer on the client side when running local optimization~\citep{lee2025efficient}. Under an asynchronous setting, we need a way to incorporate updates sent by the clients at different times. To this end, we consider \textit{server-centric} and \textit{client-centric} frameworks. %Note that a similar structure has been present before in some federated learning algorithms, but never in {TailOPT}.

Firstly, in the server-centric variant, the server proactively coordinates the learning process by sending global models to idle clients and updating the model whenever it receives $M$ returned updates; while the clients passively wait for the server to send models. %during each global iteration $t$, the clients that have sent their updates to the server will wait until the server collected $M$ updates, updated the global model, and then pull the new global model parameters and conduct local updates again. 
At any global round $t$, the server updates the global model $x_{t-1}$ as soon as it has collected $M$ clients updates. We denote the local model from client $i$ as ${x}_i^{\tau_{t,i}}$, and these $M$ client updates may have different amounts of staleness. $\tau_{t,i}$ denotes that the client model sent by client $i$ and aggregated at global round $t$ is a result of local optimization starting from an old global model at time stamp ${\tau_{t,i}}$. % as the time stamp for the global round at which the update sent to the server at round $t$ by client $i$ started. %\tian{in Eq. (1) and proof, we are using different notations? need to be consistent. better not to have the overline because it doesn't add any additional information} \tian{how about simplifying $x_{\tau_{C_t, i}}$ to $x_{\tau_{t,i}}$, and $\tau_{C_t,i}$ to $\tau_{t,i}$. We were using $\tau_t$ originally to denote the client index when $M=1$ at time stamp $t$. When $M>1$, $\tau_{t,i}$ can be cleaner and easier to read} 
In other words, client $i$ runs local optimization from $x_{\tau_{t,i}}$, and the update ${x}_i^{\tau_{t,i}}-x_{\tau_{t,i}}$ contributes to global model $x_t$ at the $t$-th global round. 
% that client $i$ trains with starts to train with we denote the client up $\tau_{t,i}$ denotes the time stamp (i.e., the global round) that the update received at $t$ global round from client $i$ started from. The model parameter received at $t$ global round by the global model from client $i$ is denoted $\overline{x^{\tau_t}}$.

\begin{figure}[!b]
\vspace{-1em}
    \centering
    \begin{minipage}{0.235\textwidth}
        \centering
        \includegraphics[width=\linewidth]{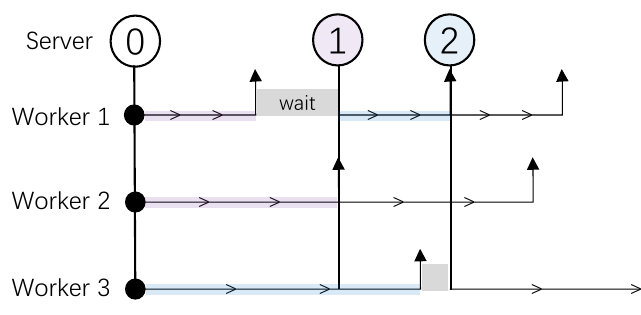}
        %\subcaption{Server-centric asynchronous, $M$$=$2, $K$$=$3, $N$$=$3}
        \label{fig:sc}
    \end{minipage}
    %\hfill
    \begin{minipage}{0.235\textwidth}
        \centering
        \includegraphics[width=\linewidth]{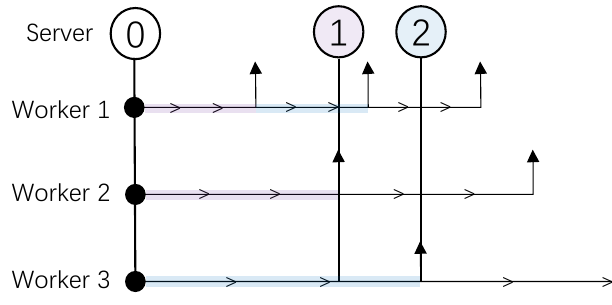}
        %\subcaption{Client-centric asynchronous, $M$$=$2, $K$$=$3, $N$$=$3}
        \label{fig:cc}
    \end{minipage}
    \vspace{-1em}
    \caption{Illustrations of server- and client-centric asynchronous models where the total number of workers $N=3$, number of local steps $K=3$, and asynchronous buffer size $M=2$.}
    \label{fig:illustration}
\end{figure}

On the contrary, in the client-centric case, every time a client sends its local model to the server, it will \textit{immediately pull} the latest global model and perform the next update. This will result in a higher utilization rate of the clients, yet potentially making updates more biased by asynchrony. The pseudo-algorithm is summarized in Algorithm~\ref{alg:vanilla_async} in the appendix.  %Illustrations are provided for these two versions of TailOPT in 
Figure~\ref{fig:illustration} illustrates the differences between the two asynchrony schemes. Some previous asynchronous algorithms are similar to the vanilla Algorithm~\ref{alg:vanilla_async}~\citep{nguyen2022federated,xie2019asynchronous,liu2021adaptive}, but we for the first time analyze the convergence under heavy-tailed noise (Section~\ref{convergence: vanilla}). %Further, we make additional modifications as discussed in the following. %this~\citep{nguyen2022federated,xie2019asynchronous,liu2021adaptive}.  %The work of \citep{xie2019asynchronous} is close to this vanilla client-centric asynchronous variant with $M=1$ without  theoretical analysis. 

It is well-known that although an asynchronous training procedure can reduce the training time, asynchrony inevitably introduces bias that comes from staleness, harming convergence: the gap between the historical global model that an update started from and the current global model. Therefore, besides simply incorporating asynchronous local updates from the clients as in Algorithm~\ref{alg:vanilla_async}, we introduce additional algorithmic modifications below.

\textbf{Staleness-Aware Downplaying and Delay Compensation.} In this part, for simplicity of presentation, we limit our discussion to the specific cases of Algorithm~\ref{alg:vanilla_async} where the outer optimizer is either simple averaging or $Clip(\cdot)$; whereas the inner optimizer is $Clip(\cdot)$. Here $Clip(\cdot)$ denotes clipping each coordinate magnitude to be below some threshold, which  brings about the benefits of handling heavy-tailed noise and memory-efficient preconditioning~\citep{lee2025efficient}. When the outer optimizer on the server side is simple average or does not apply any clipping operation, we name the base algorithm (Algorithm~\ref{alg:vanilla_async}) as $SGDClip$. When the outer optimizer uses $Clip(\cdot)$, we name the base algorithm as $Clip^2$.
% because they are the instances of {TailOPT} that are to be having some of the best performance under synchronous heavy-tailed setting (\citep{lee2025efficient}). 

Intuitively, one can consider downweighting the effects of delayed updates. %Along this line, we first propose a theoretically-principled way to re-scale the model updates based on the amount of staleness (named staleness-aware downplaying) so that the extremely stale updates can impact the current global model less. 
Specifically, we adopt a dynamic outer learning rate $\eta/p_{t,i}$ for some constant $\eta$ and $p_{t,i}:=t-\tau_{t,i}$ being the delay for time $t$ for client $i$, i.e., $\tau_{t,i}$ denoting the time stamp of the global model that client $i$ starts with to obtain local updates $x_{i}^{\tau_{t,i}}$. We present the client-centric version of such a \textit{staleness-aware downplaying (SD)} strategy in Algorithm~\ref{alg:da} with {blue} highlight; the server-centric variant is similar. This indicates that we re-scale the updates sent to the server based on the amount of staleness, so that the bias in stale updates is controlled. We prove that we can tolerate larger (unbounded) delays for the convergence to hold compared with baselines (Section~\ref{sec:sub:sd}), and observe empirically that this delay-aware technique makes the algorithm more robust to hyperparameter tuning (Section~\ref{sec:exp:effectiveness}). 

A potential downside of staleness-aware downplaying is that it may overlook useful information in the delayed updates. %the solution that it proposes to the bias of delay is to downplay the effect of the delayed update. This means the information in the delayed updates will be paid less attention to. Therefore, 
Hence, we propose a delay compensation (DC) technique aiming to approximate fresh model updates from the delayed ones (Algorithm~\ref{alg:da} with red highlight). In particular, consider the first-order Taylor expansion of $\nabla F(x_{t-1})$ at $x_{\tau_{t,i}}$. To approximate the Hessian in such Taylor expansion cheaply, we use the product between the gradients at $x_{\tau_{t,i}}$. As we consider multiple local updates, the clients sum up the Hessian approximators across all local iterations and send the statistics to the server, together with their model updates. Such a strategy is adapted from prior work \citep{zheng2020asynchronousstochasticgradientdescent}. %\tian{need to add more details about where Eq 1 comes}
%Each client maintains delay compensation vectors and send them together with model updates to the server when finishing local optimization. 
Suppose at global round $t$, the server receives $M$ updates $\{x_{i}^{\tau_{t,i}}\}_{i\in [M]}$ and Hessian approximators $\{A^{\tau_{t,i}}_{i}\}_{i \in [M]}$ (with initial global model $x_{\tau_{t,i}}$), the server-side aggregated model updates $\Delta_t$ would be corrected as
% \setmuskip{\thinmuskip}{0mu}
% \setmuskip{\medmuskip}{0.75mu}
% \setmuskip{\thickmuskip}{0.75mu}
% \thinmuskip=0.5mu
% \medmuskip=0.5mu
% \thickmuskip=0.5mu
$\Delta_{t}-\frac{1}{M}\sum\limits_{i\in[M]}A^{\tau_{t,i}}_{i}\odot (x_{t-1}-x_{\tau_{t,i}}) \text{ where }  A^{\tau_{t,i}}_{i}\leftarrow \sum\limits_{k=1}^K\lambda_t\eta_\ell^{\tau_{t,i}}(Clip({g}_{i,k})\odot 
    Clip({g}_{i,k}))$.
Here, $\eta_\ell^{\tau_{t,i}}$ is the local learning rate, $\lambda_t$ is some compensation parameter, $Clip({g}_{i,k})$ denotes the clipped local stochastic gradients on client $i$ and local iteration $k$ and $\odot$ denotes the element-wise product.
%\vspace{-0.25em}
% and $g_k^{\tau_{t,i}}$ is $\nabla F(x_k^{\tau_{t,i}})+\xi_k^{\tau_{C_{t,i }}}$. \tian{do we define $g_k^{\tau_{t,i}}$ anywhere?} 
% \tian{make the notations consistent everywhere. for example, in Algorithm 3, client sends $x_{i,z}$ whereas server collects $x_{\tau_t, i}$ [the subscriptions have different meanings]. $\tau_{C_{t,i }}$ is not defined anywhere. I feel we can delete it, and just say $\tilde{g}_k$ is the stochastic gradient; also notations in Eq. (1) doesn't match that in Alg 3 }
% The update in Eq. (\ref{eq:dcupdate}) takes inspiration from the update in \citep{zheng2020asynchronousstochasticgradientdescent} based on the first-order Taylor approximation. To understand it, we expand the content of $\Delta_t$ in the update above and realize that the update is equivalent to
% \begin{align} \label{eq:expDC}
%     \resizebox{0.5\textwidth}{!}{
%     $
%         -\frac{1}{M}\sum\limits_{i\in[M]}\left[\sum\limits_{k=1}^K \eta_\ell^{\tau_{t,i}} \left( \tilde{g}_{i,k}^{\tau_{t,i}}+\eta_\ell^{\tau_{t,i}}(\tilde{g}_{i,k}^{\tau_{t,i}}\odot \tilde{g}_{i,k}^{\tau_{t,i}})\odot (x_{t-1}-x_{\tau_{t,i}})\right)\right].$
%     }
% \end{align}
% Now, we notice that each inner term in Eq.~(\ref{eq:expDC}) is essentially Eq.~(10) in  \citep{zheng2020asynchronousstochasticgradientdescent} where $\tilde{g}_{i,k}^{\tau_{t,i}}$ are clipped local stochastic gradients instead of the original ones. As discussed in \citep{zheng2020asynchronousstochasticgradientdescent}, the dot product of the original stochastic gradients is a cheap approximator of the Hessian. 
We note that when $K=1$ and $M=1$, such updating rule reduces to the vanilla delay compensation studied in~\citep{zheng2020asynchronousstochasticgradientdescent}. % the inner terms are essentially approximations of first-order Taylor expansions at different $x_{\tau_{t,i}}$. They should offer a better approximation than the updates without compensation terms in Vanilla $Clip^2$, which corresponds to the inner terms being zero-order Taylor expansions. 
However, our update provably works for a distributed nested optimization structure, and with the local updates.  %This way, even when there are clients that are consistently slow, the updates of the slow clients will not be ignored. %\tian{edit the format of this algorithm box to reduce the white space and make it look nicer}

%\tian{put Alg 3 and 4 pseudocode in the same block and highlight the two techniques with different colors}

\setlength{\textfloatsep}{0.2pt}
\setlength{\abovecaptionskip}{1pt}
\setlength{\belowcaptionskip}{1pt}
\setlength{\interspacetitleruled}{1pt}
\setcounter{algocf}{1}
\begin{algorithm}[h!]
\LinesNumbered
\SetAlgoLined
\DontPrintSemicolon
\caption{(Proposed) $SGDClip$/$Clip^2$ with \colorbox{highlightblue}{staleness-aware downplaying (SD)} and $Clip^2$ with \colorbox{highlightpink}{delay compensation (DC)} }
\label{alg:da}
\KwIn{Initial model $x_0$, 
local learning rate schedule $\eta_\ell^t$, global learning rate schedule $\eta_t$, \highlightbox{compensation parameter $\lambda_t$}, local clipping threshold $u_t \ge 0$, global clipping threshold $\tilde{u}_t \ge 0$, minimal number of model updates $M$ $(1\leq M\leq N)$}
%\STATE $t\leftarrow 1$
\textbf{Server side: }
\For{$t = 1, \dots, T$}{
    Collects first $M$ client models $\{x_{i, K}\}_{i \in [M]}$ along with $M$ time stamps \;
    Denotes the $M$ time stamps as $\{\tau_{t,i}\}_{i\in[M]}$ and denotes $\{x_{i, K}\}_{i \in [M]}$ as $\{{x}_i^{\tau_{t,i}}\}_{i \in [M]}$   \;
    % \colorbox{highlightpink}{Collects the  Hessian approximators $\{A_{s,i}\}_{i \in [M]}$ and denotes them as $\{{A^{\tau_{t,i}}_i}\}_{i \in [M]}$} 
    % \vspace{-0.1in}
    %\tian{fix}
    \highlightboxblue{
     \begin{varwidth}{\linewidth}
     $p_{t,i} \leftarrow t - \tau_{t,i}, i \in [M]$, ~$\Delta_t \leftarrow \frac{1}{M}\Sigma_{i\in [M]}({x}_i^{\tau_{t,i}}-x_{\tau_{t,i}}) / p_{t, i}$ ~~(SD) \;
    {$x_{t} \leftarrow x_{t-1}+ \eta_t \Delta_t$  $(SGDClip$ w/ SD)} \text{,~or}
    {$x_{t} \leftarrow x_{t-1}
    + \eta_t Clip(\tilde{u}_t, \Delta_t)$ ($Clip^2$ w/ SD)} 
    \end{varwidth}}\;
    \vspace{-1em}
    $\Delta_t \leftarrow \frac{1}{M}\Sigma_{i\in [M]}({x}_i^{\tau_{t,i}}-x_{\tau_{t,i}})$\;
    \highlightbox{\label{line:sd_line}
        \begin{varwidth}{\linewidth}
        Collects the  Hessian approximators $\{A_{s,i}\}_{i \in [M]}$ and denotes them as $\{{A^{\tau_{t,i}}_i}\}_{i \in [M]}$ \;
        $\overline{g}_{t}\leftarrow \Delta_{t}-\frac{1}{M}\sum\limits_{i\in[M]}A^{\tau_{t,i}}_{i}\odot(x_{t-1}-x_{\tau_{t,i}})$, 
        $x_t \leftarrow x_{t-1}+\eta_t Clip(\tilde{u}_t,  \overline{g}_{t})$ ($Clip^2$ w/ DC)
    \vspace{-15pt}
    \end{varwidth}
    } \;
    \vspace{-1em}
    Sets the global model ready for pull as $x_t$ \;
}
    \textbf{Client side: } 
\For{each client $i$ \textbf{in parallel}}{
    \While{TRUE}{
    Pulls the current global model (assuming at global round $s$) from the server and set it as $x_{i,0}$,  
    \colorbox{highlightpink}{set $A_{s, i}=0$} \;
    \For{each local step $k \in [K]$}{
        Draws gradient $g_{i,k} \leftarrow \nabla F_i(x_{i,k}; \xi_{i,k})$ and 
        runs: $x_{i,k} \leftarrow  x_{i,k-1} - \eta_\ell^{s} \cdot Clip(u_s,g_{i,k})$ \;
\begingroup
        \thinmuskip=0.25mu
        \medmuskip=0.25mu
        \thickmuskip=0.25mu
        \colorbox{highlightpink}
        {\small $A_{s,i}$$\leftarrow$ $A_{s,i}+\lambda_s\eta_\ell^{s}(Clip(u_s,g_{i,k})\odot Clip(u_s,g_{i,k}))$}
\endgroup
    }
    Sends $x_{i,K}$, $s$, and \colorbox{highlightpink}{$ A_{s,i}$} to the server \;
    }}
\end{algorithm}

% \begin{algorithm}[h!]
% \caption{Client-Centric $Clip^2$ with delay compensation}
% \label{alg:dc}
% \KwIn{Initial model $x_0$, 
% learning rate schedule $\eta_\ell^t$, global learning rate schedule $\eta_t$, local clipping schedules $u_t \ge 0$, global clipping schedules $\tilde{u}_t \ge 0$, minimal number of model updates $M$ $(1\leq M\leq N)$}
% %\STATE $t\leftarrow 1$
% \vspace{1em}
% \textbf{Server side: }
% \For{$t = 1, \dots, T$}{
%     Collecting $M$ client models $\{x_{\tau_{t,i}}\}_{i \in [M]}$ and $\{A_{\tau_{t,i}}\}_{i \in [M]}$\;
%     $\Delta_t=\frac{1}{M}\Sigma_{i\in [M]}\frac{1}{p_{t,i}}(x_{\tau_{t,i}}-x_{t-1})$ \;
%     $\overline{g}_{t}=&\Delta_{t}-\frac{1}{M}\sum\limits_{i\in[M]}A_{\tau_{t,i}, i}\odot(x_{t-1}-x_{\tau_{t,i}})$ \;
%     $x_t=&x_{t-1}+\eta_t Clip(\tilde{u}_t, \tilde{d_t}, \overline{g}_{t})$ \;
%     Set the current global model ready for pull as $x_t$ \;
% }
% \vspace{1em}
% \textbf{Client side: } 
% \For{each client $i$ \textbf{in parallel}}{
%     \While{TRUE}{
%     Pulls the current global model (assuming at global time stamp $t$) from the server and set it as $x_{i,0}$ \;
%     $A_{t, i}=0$ \;
%     \For{each local step $k \in [z]$}{
%         Draw stochastic gradient $g_{k, i} \leftarrow \nabla F_i(x_{i,k}; \xi_{i,k})$ \;
%         Inner optimization: $x_{i,k} \leftarrow  x_{i,k-1} - \eta_\ell^t \cdot Clip(u_t,g_{i})$ \;
%         $A_{t,i} \leftarrow A_{t.i}+(\eta_\ell^{t})^2(Clip(u_t,g_{i})\odot Clip(u_t,g_{i}))$ \;
%     }
%     Send $x_{i,z}, A_{t,i}$ to the server \;
%     }}
% \end{algorithm}
\vspace{-1em}
\section{Convergence Analysis} 
\label{sec:analysis}
\vspace{-1em}
In this section, we present convergence and delay-tolerance results of the proposed frameworks in section~\ref{sec:methods}. Detailed proofs are in Appendix~\ref{app:proofs}. Throughout this section, we use $\tau$ to refer to the maximum delay during the training process, i.e., $\tau := \max_{t,i} p_{t,i}$ (recall $p_{t,i}:=t-\tau_{t,i}$). A table summarizing the convergence guarantees of our methods and comparisons with other prior works is presented in Appendix~\ref{app:sub:big_table}.

% Moreover, similarly to \citep{lee2025efficient}, we carry out the analysis over a sufficiently large, compact domain $\mathcal{X}$. This means by applying EVT, we obtain a universal bound $G$ on the deterministic gradients $\nabla F(\cdot)$ that we encounter during the analysis. This is a standard assumption necessary for the control of heavy-tailed noise in analysis, similarly done in relevant papers \citep{HighProbAdaGConv,lee2025efficient, inprobsupreme, heavytail2, CentralClip}. 

\subsection{Vanilla Asynchronous Convergence Under Heavy-Tailed Noise} \label{convergence: vanilla}
% \tian{provide some formal statements on the effects of $M$ somewhere later} 
% We first state the heavy-tailedness condition formally.

% \begin{definition}[Heavy-Tailed Gradients]
%   \tian{todo: consider having a formal definition at the beginning, so we don't have to mention $E[\|\xi\|^{\alpha}] \leq D^{\alpha}$ in every theorem}
% \end{definition}

First of all, we present the convergence results of vanilla server- and client-centric asynchronous framework under heavy-tailed noise (Algorithm~\ref{alg:vanilla_async}). \textit{We show that under certain scheduling of client- and server-side clipping thresholds and learning rates, asynchronous $SGDClip$ and $Clip^2$ match the lower bound for stochastic
    first-order methods under heavy-tailed noise with bounded $\alpha$-th 
    moment~\citep{zhang2020adaptive}, while demonstrating reasonable delay tolerance}.

% \textit{Specifically, when the stochastic noise satisfies that $\mathbb{E}[\|\xi\|^{\alpha}] \leq D^{\alpha}$ for $\alpha \in (1,2)$, under mild assumptions, asynchronous $SGDClip$ achieves}
% \begin{align} \label{eq:sgdbound}
%     \min_{t \in [T]}\mathbb{E}[\|\nabla F(x_{t-1})\|^2] \leq  O\left(T^{-\frac{\alpha-1}{\alpha}}\right)
% \end{align}
% \textit{with a delay tolerance of maximum delay $\tau\leq O(T^{\frac{2-\alpha}{\alpha}})$. And asynchronous $Clip^2$ achieves}
% \begin{align} \label{eq:clip2bound}
%     \min_{t \in [T]}\mathbb{E}[\|\nabla F(x_{t-1})\|^2] \leq  O\left(T^{-\frac{\alpha-1}{4\alpha-2}}\right)
% \end{align}
% \textit{when $\tau \leq O(T^{\frac{\alpha}{4\alpha-2}})$}. Due to space constraints, the formal statements and proofs are provided in Appendix~\ref{app:sub:sgdandclip2}. Now, we move on to the analysis of our proposed delay-aware framework (Algorithm~\ref{alg:da}).

Specifically, when the stochastic noise satisfies that $\mathbb{E}[\|\xi\|^{\alpha}] \leq D^{\alpha}$ for $\alpha \in (1,2)$, under mild assumptions, both asynchronous $SGDClip$ and $Clip^2$ achieves
{\small\begin{align} \label{eq:sgdbound}
    \min_{t \in [T]}\mathbb{E}[\|\nabla F(x_{t-1})\|] \leq  O\left(T^{\frac{1-\alpha}{3\alpha-2}}\right)
\end{align}}
with a delay tolerance of maximum delay $\tau\leq O(T^{\frac{\al}{2(3\alpha-2)}})$, which recovers the previous bound $O(T^{-\frac{1}{4}})$ when $\alpha\to 2$ \citep{zhang2020adaptive, zheng2020asynchronousstochasticgradientdescent, mishchenko2023asynchronoussgdbeatsminibatch, tyurin2023optimaltimecomplexitiesparallel, tyurin2025tighttimecomplexitiesparallel, guan2017delay, wang2024fadas}%\tian{cite async papers with this bound}
. Formal statements and proofs are in Appendix~\ref{app:sub:sgdandclip2}. Now, we move on to analyze our proposed delay-aware framework (Algorithm~\ref{alg:da})

\subsection{Staleness-Aware Downplaying} \label{sec:sub:sd}

In this subsection, we will show that staleness-aware downplaying provides a weaker delay restriction compared with vanilla $SGDClip$ or $Clip^2$. %: Instead of being completely restricted by the maximum delay $\tau$, staleness-aware downplaying makes the algorithm more tolerant to the maximum delay while still achieving convergence. % when there is a considerable proportion of fast workers. The proportional rescaling/downplaying makes the algorithm able to utilize the variability of the distribution across different workers to address the large bias from the most delayed updates. 
That is, \textit{without assuming bounded delays, staleness-aware downplaying can achieve convergence even when the largest delay is unbounded with the same optimal convergence bound $O(T^\frac{1-\al}{3\al-2})$.} All proofs for statements in this section are detailed in Appendix~\ref{app:sub:SA}. 
% \tian{make our delay assumptions more explicit either by text or equations. because lots of previous asynchronous analyses assume standard bounded delay assumptions but we have weaker conditions}

In this section, we introduce the following notations: $
  P_1:=\sum_{t=1}^{T} \frac{1}{M}\sum_{i=1}^{M} \frac{1}{p_{t,i}}$, $
  P_2:=\sum_{t=1}^{T} \Big(\frac{1}{M}\sum_{i=1}^{M} \frac{1}{p_{t,i}}\Big)^{2}$, $
  \Lambda:=\max_{s\in[T]} \sum_{(t,i):\, \tau_{t,i} < s < t}
\frac{1}{M p_{t,i}} \sum_{r\in W_{t,i}} \frac{1}{M}\sum_{m=1}^{M} \frac{1}{p_{r,m}}$,
where $W_{t,i} := \{\tau_{t,i}+1,\dots,t-1\}$.

\subsubsection{$SGDClip$ with Staleness-Aware Downplaying}
We first present results of $SGDClip$ (i.e., server-side optimizer being simple average and client-side optimizer being coordinate-wise clipping) with the proposed staleness-aware downplaying strategy.

\begin{theorem} \label{thm:sasgd}
    Assume $F(\cdot)$ to be $L$-smooth. Set $u$ to be $u=\Theta(T^{\zeta})$, and $\eta_\ell=\Theta(T^\nu)$. If the stochastic gradient noise $\xi$ is heavy-tailed, i.e., it satisfies that $\mathbb{E}[\|\xi\|^{\alpha}] \leq D^{\alpha}$ for $\alpha \in (1,2)$, and we set$
  \eta\eta_\ell=\min\Bigl\{\sqrt{\frac{\Delta}{LK^2D^\alpha u^{2-\alpha}P_2}},
  \frac{1}{374\,LK\sqrt{\max\{1,\Lambda\}}}\Bigr\}$,
 then $SGDClip$ with staleness-aware downplaying gives
     \begin{align*}
        \min_{t\in[T]}\E\left[\|\nabla F(x_{t-1})\|\right]
          \le & O\bigg(T^{\frac{2-\alpha}{4}\zeta}\,\frac{P_2^{1/4}}{P_1^{1/2}}
          +\frac{\max\{1,\Lambda\}^{1/4}}{P_1^{1/2}}\\&
          +T^{(1-\alpha)\zeta}+T^{\nu+\zeta}+\sqrt{\frac{\Lambda}{P_2}}\bigg).
    \end{align*}
    % In particular, if all $p_{t,j}'s$ take the same value $p$, and $p=\Theta(T^b)$ with $b\leq \frac{2-\alpha}{2\alpha}$, then setting $\zeta=\frac{1}{\alpha}$ gives us 
    % \begin{align}
    %     \min_{t \in [T]} \mathbb{E}[\|\nabla F(x_{t-1})\|^2]\leq O\left(T^{\frac{1-\alpha}{\alpha}}\right). \nonumber
    % \end{align}
\end{theorem}

\begin{remark}
    We notice that if all $p_{t,j}'s$ take the same value $p=\Theta(T^b)$ with $b\leq \frac{\al}{2(3\alpha-2)}$, then setting $\zeta=\frac1{3\alpha-2}$ and $\nu=-\frac{\alpha}{3\alpha-2}$ gives us 
    $\min_{t \in [T]} \mathbb{E}[\|\nabla F(x_{t-1})\|]\leq O\left(T^{\frac{1-\alpha}{3\alpha-2}}\right)$, which is the same delay tolerance and the convergence guarantee as those of the vanilla asynchronous \textit{SGDClip}. This is expected since the introduction of $p_{t,j}$'s `evens out' the bias of stale model updates. If the delays are uniform, then the bias does not require any evening out from the first place.
    However, the benefit of such a method is that it is less restrictive on delays when the delays are not uniform. 
It can be checked that $\Lambda\le N/M$. Therefore, 
the bound and the condition of Theorem~\ref{thm:sasgd}
depend on the delays only through $P_1$, $P_2$. In
particular, a few very large delays barely change $P_1$ and $P_2$.
Thus, it indicates that staleness-aware downplaying enables convergence even when the largest delay is unbounded.
 
 %This means having even one very large delay requires a huge $T$. 
    % \tian{a bit confusing. $u$ is the clipping threshold also? refer to the delay tolerance for SGDClip?}
\end{remark}

% \begin{corollary}
%     For $F$ being $L$-smooth and $G$-Lipschitz, if we set
%         $\eta\eta_\ell=\sqrt{\frac{2C_1}{LK^2u^2\sum\limits_{t=1}^T\frac{1}{p_t^2}}}$,
%     SGDClip with Staleness-Aware Downplaying satisfies $\min_{t \in [T]} \mathbb{E}[\|\nabla F(x_{t-1})\|^2]$ is upper bounded by 
%     % \tian{find a way to make the equation one line}

%     \begin{equation}
%     \resizebox{0.94\textwidth}{!}{
%         $\displaystyle
%         \frac{2\sqrt{C_1(\frac{L}{2}u^2\sum\limits_{t=1}^T\frac{1}{p_t^2})}}{\sum\limits_{t=1}^T\frac{1}{p_t}}+2^{\alpha-1}(G^\alpha+D^\alpha) u^{1-\alpha}+GL\frac{(K-1)(K-2)}{2K}\eta_\ell u+\frac{2G\sqrt{LC_1}\sum\limits_{t=1}^T\frac{1}{p_t}\sum\limits_{i=t-p_t}^{t-1}\frac{1}{p_i}}{\sum\limits_{t=1}^T\frac{1}{p_t}\sqrt{\sum\limits_{t=1}^T\frac{1}{p_t^2}}}.
%         $
%     }
%     \end{equation}

%     % \begin{align}
%     %      \frac{2\sqrt{C_1(\frac{L}{2}u^2\sum\limits_{t=1}^T\frac{1}{p_t^2})}}{\sum\limits_{t=1}^T\frac{1}{p_t}}+2^{\alpha-1}(G^\alpha+D^\alpha) u^{1-\alpha}+GL\frac{(K-1)(K-2)}{2K}\eta_\ell u+\frac{2G\sqrt{LC_1}\sum\limits_{t=1}^T\frac{1}{p_t}\sum\limits_{i=t-p_t}^{t-1}\frac{1}{p_i}}{\sum\limits_{t=1}^T\frac{1}{p_t}\sqrt{\sum\limits_{t=1}^T\frac{1}{p_t^2}}}.
%     % \end{align}
% \end{corollary}

There is prior work that studies a different way to perform staleness-aware downweighting for classic asynchronous optimization~\citep{wang2024fadasfederatedadaptiveasynchronous}. We can compare the bound we obtained in Theorem~\ref{thm:sasgd} with that in \citep{wang2024fadasfederatedadaptiveasynchronous}. Firstly, the convergence in \citep{wang2024fadasfederatedadaptiveasynchronous}  requires bounded variance, whereas our bound is valid even with heavy-tailed noise.  Moreover, the convergence upper bound in prior work is determined by the average and median of the delays. If these values are large enough, the convergence bound explodes. \textit{However, a benefit of our convergence guarantee for $SGDClip$ with staleness-aware downplaying is that either with a certain proportion of workers that have small enough delay, or with the downplay is further adapted to the number of clients $N$ %\tian{where is the equation about $\eta$ adapted to $N$?}
(changing the downplay in line~\hyperref[line:sd_line]{4} of Algorithm~\ref{alg:da} to $\Delta_t \leftarrow \frac{1}{M}\Sigma_{i\in [M]}({x}_i^{\tau_{t,i}}-x_{\tau_{t,i}}) / \max\{p_{t, i}, N\}$), we can converge even when the median and average delays across all clients are large.}

\begin{proposition}\label{prop:sasgd}
Suppose that among all $p_{t,j}$'s for $1\le t\le T$, there is a fraction $q\in(0,1]$
of delays that satisfy $p_{t,j}\le A$, and let $T_{\rm eff}:=qT/A$. Let
$u=\Theta(T_{\rm eff}^{\frac1{3\alpha-2}})$ and
$\eta_\ell=\Theta(T_{\rm eff}^{-\frac{\alpha}{3\alpha-2}})$, and choose $\eta\eta_\ell$
as in Theorem~\ref{thm:sasgd}. Then for sufficiently large $T_{\rm eff}$, SGDClip
with staleness-aware downplaying satisfies
\begin{equation}\label{eq:prop-sd}
  \min_{t\in[T]}\E \left[\|\nabla F(x_{t-1})\|\right]
  \le O\bigg(T_{\rm eff}^{\frac{1-\al}{3\alpha-2}}+\sqrt{\frac{NA}{T_{\rm eff}}}\bigg).
\end{equation}
In particular, for $A=O(1)$, and $N\le O(T^{\frac{\alpha}{3\alpha-2}})$,
we have $\min_t\E \left[\|\nabla F(x_{t-1})\|\right] \le  O(T^{\frac{1-\alpha}{3\alpha-2}})$, regardless of
the remaining delays. If $A \le  O(T^c)$ with
$c$$<$$\frac12$, then
$\min_t\E \left[\|\nabla F(x_{t-1})\|\right]$$\le$$ O\big(T^{-(1-c)\frac{\alpha-1}{3\alpha-2}}+\sqrt N\,T^{c-\frac12}\big)$.

\end{proposition}

% \begin{proposition}\label{prop:sasgd}
%     Suppose that among all $p_{t,j}$'s for $1\le t\le T$, there is a fraction
%     $q\in(0,1]$ of delays that satisfy $p_{t,j}\le A$, where $A\le O(T^c)$. Let
%     $u=\Theta(T^\zeta)$ and $\eta_\ell=\Theta(T^\nu)$ with $\zeta>0$ and $\nu<-\zeta$,
%     and choose $\eta\eta_\ell$ as in Eq. (\ref{eq:sdmain-stepsize}). Then when
%     $\label{eq:prop-sd-cond}
%     c<\min\Big\{\frac12,\ 1-(2-\alpha)\zeta\Big\}$,
%     we have that SGDClip with staleness-aware downplaying converges. In
%     particular, this holds if the minimum delay $a\le O(T^c)$ takes up a fixed
%     fraction $q$ of all delays.
% \end{proposition}

\textit{Proposition~\ref{prop:sasgd} here shows that our method can converge at an optimal rate even with an unbounded largest delay, as long as a fixed fraction of delays are small enough and the total number of clients $N$ is controlled}. We next show that such conditions can be removed if we also adapt the global learning rate with respect to $N$.

\begin{theorem}\label{thm:sasgdwithN}
    With the same assumption as Theorem~\ref{thm:sasgd}, if we change the downplay in Line~\hyperref[line:sd_line]{4} of Algorithm~\ref{alg:da} to $\Delta_t \leftarrow \frac{1}{M}\Sigma_{i\in [M]}({x}_i^{\tau_{t,i}}-x_{\tau_{t,i}}) / max\{p_{t, i}, N\}$, with proper scheduling, SGDClip with staleness-aware downplaying satisfies:
    $
        \min_{t\in[T]}\E[\|\nabla F(x_t)\|]\leq O\Big(T^{-\frac{\alpha-1}{3\alpha-2}}+\frac{N}{\sqrt T}\Big).
    $
    
\end{theorem}
Notice that this recovers the $O(T^\frac{1-\al}{3\al-2})$ bound of vanilla SGDClip when $N\le O(T^{\frac{\alpha}{2(3\alpha-2)}})$, but without the restriction on the largest delay.

\subsubsection{$Clip^2$ with Staleness-Aware Downplaying}
Here we show that staleness-aware downplaying $Clip^2$ (i.e., both server- and client-side adopting coordinate-wise clipping) tolerates large (even unbounded) delays while still achieving convergence.

% \begin{theorem}\label{thm:saclip2}
%     With the same assumption as in Theorem~\ref{thm:sasgd}, $\tilde{u}$ being the server-side upper-clipping threshold, under the condition $\sqrt{\frac{C_1}{L\sum\limits_{t=1}^T\frac{1}{p_t^2}}}\leq \frac{2u\eta_\ell \sum\limits_{t=1}^T\frac{1}{p_t}}{(K+1)\sum\limits_{t=1}^T\frac{1}{p_t}\sum\limits_{i=t-p_t}^{t-1}\frac{1}{p_i}}$, $Clip^2$ with Staleness-aware Downplaying achieves the convergence guarantee:
%     \begin{align}
%         \min_{t \in [T]} \mathbb{E}[\|\nabla F(x_{t-1})\|^2]\leq & \frac{2\sqrt{\frac{C_1L\tilde{u}^2\sum\limits_{t=1}^T\frac{1}{p_t^2}}{2K^2\eta_\ell^2}}}{\sum\limits_{t=1}^T\frac{1}{p_t}}+\frac{G\tilde{D}\tilde{u}^{1-\alpha}\eta_\ell^{\alpha-1}}{K}+Gu^{1-\alpha}2^{\alpha-1}(G^\alpha+D^\alpha)+2GL\frac{K+1}{2}u\eta_\ell.
%     \end{align}
%     % when we have that 
%     % \begin{align}
%     %     \sqrt{\frac{C_1}{L\sum\limits_{t=1}^T\frac{1}{p_t^2}}}\leq \frac{2u\eta_\ell \sum\limits_{t=1}^T\frac{1}{p_t}}{(K+1)\sum\limits_{t=1}^T\frac{1}{p_t}\sum\limits_{i=t-p_t}^{t-1}\frac{1}{p_i}}.
%     % \end{align}
% \end{theorem}

\begin{theorem}\label{thm:saclip2}
    With the same assumption as in Theorem~\ref{thm:sasgd}, let the server-side
    clipping threshold satisfy $\tilde u \geq \kappa K \eta_\ell u$ for some
    $\kappa \in (0,1]$.
    % $\sqrt{\frac{C_1}{L\sum\limits_{t=1}^T(\sum\limits_{j=1}^M\frac{1}{p_{t,j}})^2}}\leq \frac{2u\eta_\ell \sum\limits_{t=1}^T(\sum\limits_{j=1}^M\frac{1}{p_{t,j}})}{(K+1)\sum\limits_{t=1}^T\sum\limits_{j=1}^M\frac{1}{p_{t,j}}\sum\limits_{i=t-p_{t,j}}^{t-1}\frac{1}{p_{i,j}}}$,
    % {\small{$\sqrt{C_1/(LKQ_\alpha\sum\limits_{t=1}^T(\sum\limits_{j=1}^M \frac{1}{p_{t,j}})^2}) \leq (K(K+1)u^\frac{2-\alpha}{2}\eta_\ell \sum\limits_{t=1}^T(\sum\limits_{j=1}^M \frac{1}{p_{t,j}}))/ (2\sum\limits_{t=1}^T\sum\limits_{j=1}^M\frac{1}{p_{t,j}}\sum\limits_{i=t-p_{t,j}}^{t-1}\frac{1}{p_{i,j}})$}},
    If
    $\eta_\ell K\le\frac{\kappa}{46L}$ and $u\ge D(\frac{1024}{\kappa^2})^{1/\alpha}$, and $
      \eta\eta_\ell=\min\Bigl\{\sqrt{\frac{\Delta}{LK^2D^\alpha u^{2-\alpha}P_2}},
      \frac{\kappa}{536\,LK\sqrt{\max\{1,\Lambda\}}}\Bigr\}$,
    then when we use $\eta_\ell=\Theta(T^\nu)$, $u=\Theta(T^\zeta)$, $Clip^2$ with staleness-aware downplaying achieves the convergence guarantee:
    % \begin{align}
    %     \min_{t \in [T]} \mathbb{E}[\|\nabla F(x_{t-1})\|^2] \nonumber \leq  O\Bigl(T^{\tilde{\zeta}-\nu}\frac{\sqrt{\sum\limits_{t=1}^T(\sum\limits_{j=1}^M\frac{1}{p_{t,j}})^2}}{\sum\limits_{t=1}^T(\sum\limits_{j=1}^M\frac{1}{p_{t,j}})} 
    %     \\ +T^{(1-\alpha)(\tilde{\zeta}-\nu)}+T^{(1-\alpha)\zeta}+T^{\zeta+\nu}\Bigr). \nonumber
    % \end{align}
    \begin{align*}
        \min_{t\in[T]} \mathbb{E}\big[\|\nabla F(x_{t-1})\|\big]
        \leq & O\!(
        T^{\frac{2-\alpha}{4}\zeta}\,\frac{P_2^{1/4}}{P_1^{1/2}}
        + \frac{\max\{1,\Lambda\}^{1/4}}{\sqrt{P_1}}\\
        &+ T^{(1-\alpha)\zeta}
        + T^{\nu+\zeta}
        + \sqrt{\frac{\Lambda}{P_2}}
        ). \nonumber
    \end{align*}
    % In particular, if all $p_{t,j
    % }'s$ take the same value $p$ and  $p=O(T^b)$ with $b\leq \frac{1}{4}+\frac{1}{4\alpha}$, we have 
    % \begin{align}
    %     \min_{t \in [T]} \mathbb{E}[\|\nabla F(x_{t-1})\|^2]\leq O\left(T^{- \frac{\alpha-1}{4\alpha}}\right). \nonumber
    % \end{align}
    % In particular, if all $p_{t,j
    % }'s$ take the same value $p$ and  $p=O(T^b)$ with $b\leq \frac{2-\alpha}{\alpha}$, we have 
    % \begin{align}
    %     \min_{t \in [T]} \mathbb{E}[\|\nabla F(x_{t-1})\|^2]\leq O\left(T^{\frac{1-\alpha}{\alpha}}\right). \nonumber
    % \end{align}
\end{theorem}
It's worth noting that since client-side clipping ensures that each client update has norm at most
$K\eta_\ell u$, the server-side clipping is active only when $\tilde u < K\eta_\ell u$, which is why $\tilde u$ enter the bound only through $\kappa$ instead of explicitly.

We can similarly see that, for $Clip^2$, \textit{the algorithm can converge at an optimal rate even when the average and median delays are large, as long as we have a certain proportion of $p_{i,t}'s$ being small}, formally stated below.
\begin{proposition}\label{prop:saclip2}
    % If among all $p_{t,j}'s$ for $1\leq t\leq T$, we have that there is  $q \in (0,1)$ fraction of delays that satisfies $p_{t,j}\leq A$ where $A\leq  O(T^c)$, then when
    %     % $c< \min\left\{\frac{1}{2}+\nu-\tilde{\zeta}, \frac{1}{4}\right\}$,
    %      $c< \min\left\{\frac{1}{2}-\frac{2-\alpha}{2}\zeta, \frac{2-\alpha}{4}\zeta+\frac{1}{4}\right\}$,
    % % we have that $Clip^2$ converges. Or, if the minimum delay $a_{\min}$ takes up a fixed fraction $p\in(0,1)$ of all delays and satisfies $a_{\min}\leq O(T^{\frac{1}{2}})$, we also have that $Clip^2$ converges.
    % we have that $Clip^2$ converges. Or, if the minimum delay $a_{\min}$ takes up a fixed fraction $p\in(0,1)$ of all delays and satisfies $a_{\min}\leq O(T^{\frac{1}{2}+\frac{2-\alpha}{2}\zeta})$, we also have that $Clip^2$ converges.
    
    % Suppose that among all $p_{t,j}$'s for $1\le t\le T$, there is a fraction
    % $q\in(0,1]$ of delays that satisfy $p_{t,j}\le A$, where $A\le O(T^c)$. Let
    % $u=\Theta(T^\zeta)$ and $\eta_\ell=\Theta(T^\nu)$ with $\zeta>0$ and $\nu<-\zeta$,
    % and choose $\eta\eta_\ell$ as in Eq.~\ref{eq:sd2main-stepsize}. Then when
    % $c<\min\{\frac12,\ 1-(2-\alpha)\zeta\}$, Clip$^2$ with staleness-aware
    % downplaying converges. In particular, this holds if the minimum delay $a\le O(T^c)$
    % takes up a fixed fraction $q$ of all delays.

    With the same setting and scheduling in Proposition~\ref{prop:sasgd}, if we make
    $\tilde u=\kappa K\eta_\ell u$ for a constant $\kappa\in(0,1]$, and choose
    $\eta\eta_\ell$ as in Theorem~\ref{thm:saclip2}. Then for sufficiently large
    $T_{\rm eff}$, Clip$^2$ with staleness-aware downplaying also satisfies the same bound as in Proposition~\ref{prop:sasgd}:
    $
      \min_{t\in[T]}\E\|\nabla F(x_{t-1})\|
      \le O\bigg(T_{\rm eff}^{\frac{1-\alpha}{3\alpha-2}}+\sqrt{\frac{NA}{T_{\rm eff}}}\bigg).
    $
\end{proposition}

Similarly, as shown below, we can also drop the assumptions on delays if we adapt the downplay, which results in  
the optimal bound $O(T^\frac{1-\al}{3\al-2})$.
\begin{theorem}\label{thm:clip2withN}
    With the same assumption as Theorem~\ref{thm:sasgd}, if we change the downplay in line~\hyperref[line:sd_line]{4} of Algorithm~\ref{alg:da} to $\Delta_t \leftarrow \frac{1}{M}\Sigma_{i\in [M]}({x}_i^{\tau_{t,i}}-x_{\tau_{t,i}}) / max\{p_{t, i}, N\}$, with proper scheduling, $Clip^2$ with staleness-aware downplaying satisfies
    $
        \min_{t\in[T]}\E[\|\nabla F(x_t)\|]\leq O\Big(T^{-\frac{\alpha-1}{3\alpha-2}}+\frac{N}{\sqrt T}\Big).$
\end{theorem}

\subsection{Delay Compensation} \label{sec:sub:dc}

With the benefit of staleness-aware downplaying clarified theoretically, we move to analyze the convergence of delay compensation in Algorithm~\ref{alg:da}. Note that the crucial difference in our setting is that the noise is heavy-tailed, and we perform local updates. We  show in the following that \textit{$Clip^2$ with delay compensation achieves the same convergence rate as the synchronous version. Moreover, we demonstrate that  heavy-tailed noise influences delay compensation's ability to handle delays in that the algorithm obtains a better delay tolerance when the noise is less heavy-tailed.} %The delay tolerance is improved by a constant compared with vanilla asynchronous $Clip^2$. %In this section, we use $w$ instead of $x$ as the model parameters.

% \begin{theorem} \label{thm:dc}
%     Assume $F$ is $\mu$-strongly convex in a ball centered at each local optimum $x_{loc}$ with radius $r$; $F$'s second and third order gradients are bounded; and satisfies $\| \mathbb{E}[(Clip(u_s, g_{i,k}))^{\odot 2}] - \mathbb{E}[Diag(H(x_{i,k}))] \| \leq O(1) \| x_{i,k} - x_{loc} \| + O(1)$ for any $i,k$. Running $Clip^2$ with delay compensation using the assignment 
%     $ \omega=-\frac{1}{2}, \nu=-\frac{\alpha}{4\alpha-2}, \tilde{\zeta}=0, \zeta=\frac{1}{4\alpha-2} 
%     $ for $\eta=\Theta(T^\omega), \eta_\ell=\Theta(T^\nu),  u=\Theta(T^\zeta), \tilde{u}=\Theta(T^{\tilde{\zeta}})$
%      gives us that when $\tau\leq O(T^{\frac{\alpha}{4\alpha-2}})$, we have
%     \begin{align}
%         \min_{t\in [T]} \|\nabla F(w_{t-1})\|^2\leq O\left(T^{\frac{1-\alpha}{4\alpha-2}}\right). \nonumber
%     \end{align}
% % \tian{can we make the theorems easier to read, maybe removing this and directly give $T^{\frac{1-\alpha}{4\alpha-2}}$; I don't think we need Lemma 1 in the main paper?}
% \end{theorem}

\begin{theorem} \label{thm:dc}
    % Assume $F$ is $\mu$-strongly convex in a ball centered at each local optimum $x_{loc}$ with radius $r$; $F$'s second and third order gradients are bounded; and satisfies $\| \mathbb{E}[(Clip(u_s, g_{i,k}))^{\odot 2}] - \mathbb{E}[Diag(H(x_{i,k}))] \| \leq O(1) \| x_{i,k} - x_{loc} \| + O(u^{2-\al})$ for any $i,k$. Running $Clip^2$ with delay compensation using the assignment 
    % $ \omega=\frac{\al-1}{\al}, \nu=-1, \tilde{\zeta}=\frac{2-3\al}{2\al}, \zeta=\frac{1}{\al} 
    % $ for $\eta=\Theta(T^\omega), \eta_\ell=\Theta(T^\nu),  u=\Theta(T^\zeta), \tilde{u}=\Theta(T^{\tilde{\zeta}})$
    %  gives us that when $\tau\leq O(T^{\frac{1}{2\alpha}})$, we have
    % \begin{align}
    %     \min_{t\in [T]} \|\nabla F(w_{t-1})\|^2\leq O\left(T^{\frac{1-\alpha}{\alpha}}\right). \nonumber
    % \end{align}
    Assume $F$ is $L$-smooth, and its Hessian is $L'$-Lipschitz. Let $\kappa\in (0,1] $. Running $Clip^2$ with delay compensation using the assignment 
    $ \nu=-\frac{\al}{3\al-2}, \zeta=\frac{1}{3\al-2} 
    $ for $\eta_\ell=\Theta(T^\nu),  u=\Theta(T^\zeta)$, $\tilde u=\kappa K\eta_\ell u$, and delay compensation parameter $\lambda_T$ such that $\lambda_Tu^2\leq O(1)$ 
     gives us that when $\tau\leq O(T^{\frac{\al}{2(3\al-2)}})$, we have
    \begin{align}
        \min_{t\in [T]} \E[\|\nabla F(w_{t-1})\|]\leq O\left(T^{\frac{1-\alpha}{3\alpha-2}}\right). \nonumber
    \end{align}
\vspace{-0.4em}
\end{theorem}

\begin{remark}
    Notice that we achieve the same optimal guarantee $O(T^\frac{1-\al}{3\al-2})$ as our previous methods, which achieves $O(T^{-\frac{1}{4}})$ as $\al\to2$.

    We note that the benefit of delay compensation
    lies in the constant of the delay-dependent term. In Vanilla $Clip^2$
    (Theorem~\ref{thm:clip2}), this term is
    $O(L^2\tau^2\gamma^2K^2V)$. With delay compensation, it can be seen in Theorem~\ref{thm:dc_pre} in the Appendix~\ref{app:sub:dc} that the term is
    $O(S_{\rm DC}\tau^2\gamma^2K^2V)$, where $S_{\rm DC}\le3\varepsilon_H^2+24L'^2\tau^2\gamma^2K^2u^2+LK\gamma(16L^2+\lambda^2u^4)$, and 
    $
    \varepsilon_H:=\operatorname*{ess\,sup}_{t,k}\big\|\nabla^2F(x^{\tau_t}_{k-1})-C_{t,k}\big\|
    $
    as the curvature estimation accuracy.
    
    Therefore, as long as the delays are moderate in comparison to the curvature estimation accuracy, i.e.,
    $L'\tau\gamma Ku\le\varepsilon_H$, the delay-dependent term is reduced by a
    factor of order $(\varepsilon_H/L)^2$ (up to the higher-order terms). This means that the delay tolerance for a given
    accuracy increases by a factor of order $L/\varepsilon_H$, which is expected since a smaller $\varepsilon_H$ means better curvature estimation, and thus corresponds to more robustness against delays. Moreover, for larger delays,
    the minimum in $\varepsilon_{\rm eff}$ in Lemma~\ref{lem:dc-residual} enables a proof similar to vanilla $Clip^2$. This guarantees that delay compensation is never worse than the vanilla $Clip^2 $ up to a constant. 
    
    Finally, it is worth pointing out that $\lambda_T$ shrinks with $T$ as we need $\lambda_T u^2\leq O(1)$. This means a larger clipping threshold $u$ leads to weaker compensation. This is reasonable since the approximation of the Hessian relies
    on the second moment of the gradients, which, under heavy-tailed
    noise, is dominated by the clipped noise variance of order $u^{2-\alpha}$. Therefore, $\lambda$ shrinks in response to the large noises of the compensation once $u$ grows.  The benefit of delay compensation is
    therefore a non-asymptotic one that reflects in a constant factor.

     % But notice that the delay tolerance is $\tau\leq O(T^\frac{1}{2\al})$. This restriction on delay is looser than the delay tolerance of vanilla asynchronous methods when $\al\geq 1.5$. \textit{This means that delay compensation can effectively mitigate the influence of asynchrony better when the heavy-tailed noise is less severe (a larger $\al$)}. The larger the heavy-tailed noise, the less accurate the $A_i^{\tau_{t,i}}$ are as Hessian approximates, which make the compensation less effective.
\end{remark}
The complete convergence statements and proofs for all of our theoretical results are in Appendix~\ref{app:sub:dc}. 
%being as the following. With $f$ being the cross-entropy loss $f(x,y,w)=-\sum\limits_{c=1}^C(\chi(y=c)log(\sigma_c(x; w))$ ($\sigma_c$  is the softmax operator), we have that $F$ is defined to be the expected error, i.e., $F(w)=\mathbb{E}_{\mathcal{D}}[f(x,y,w)]$.
% \tian{what is $f$ and what is $c$}
%A more detailed description of why we adopt this setting is in Appendix~\ref{app:sub:dc}. 
On a high level, the local step number $K$ presents a tradeoff: as $K$ increases, the primary convergence term, which is inversely proportional to $K$, decreases because of more local updates per step. However, a larger $K$ leads to a larger drift error between the local model and the global model, which increases several terms in the rates (including the one associated with $\tau$) and reduces the system's tolerance to delays. Similarly, the clipping thresholds also present a tradeoff on the convergence rates: certain terms increase as they increase as a result of the bias from heavy-tailed noise, yet other terms increase as they decrease, reflecting the loss of information from smaller clipping thresholds. 

\vspace{-0.4em}
\section{Experiments}
\label{sec:exp}
\vspace{-0.4em}

\begin{table*}[!h]
\vspace{-0.8em}
\caption{\small Comparison of synchronous training with vanilla server- and client-centric asynchronous training (Algorithm~\ref{alg:vanilla_async}) on the GLUE benchmark. We see that under both mild and large delays, neither server-centric or client-centric variant degrades accuracies compared with the synchronous case, while greatly reducing the runtime (see Table~\ref{tab:glue-vanila} in the appendix). The optimizer is $Clip^2$.}
\label{tab:glue_sync_async}
\centering
\resizebox{0.95\textwidth}{!}{
\begin{tabular}{ccccccccccc}
\toprule
& Async Mode & MNLI & QNLI & QQP & RTE & SST-2 & MRPC & CoLA & STS-B & Avg\\
\midrule
\multirow{3}{*}{Mild Straggler} 
  & Sync &  83.67 & 87.33 & 79.86& 64.30 & 92.88 & 86.03 & 82.08 & 86.54 & 82.84 \\ 
  %\cline{2-11}
  & \multirow{1}{*}{Server} & 83.54 & 84.72 & \textbf{86.33} & 60.29 & 92.66 & \textbf{86.03} & 81.11 & \textbf{88.14} & \textbf{82.85} \\
  & \multirow{1}{*}{Client} &  81.87 & 86.86 & \textbf{85.37} & \textbf{70.04} & 91.63 & \textbf{87.01} & 81.02 & \textbf{87.14} & \textbf{83.87} \\
\midrule
\multirow{3}{*}{Large Straggler} 
& Sync  & 83.30 & 87.22 & 83.03 & 64.30 & 92.78 & 85.05 & 79.77 & 86.43 & 82.74 \\ 
  & \multirow{1}{*}{Server} & 82.58 & 85.78 & \textbf{86.09} & 59.95 & 91.97 & 83.58 & \textbf{80.44} & 85.63 & 82.00 \\
  & \multirow{1}{*}{Client} &  82.19 & 86.89 & \textbf{85.39} & \textbf{67.87} & 92.29 & \textbf{85.05} & \textbf{80.15} & \textbf{86.92} & \textbf{83.34} \\
  \bottomrule
\end{tabular}}
\vspace{-.5em}
\end{table*}

%\tian{move some GLUE benchmark results to experiments}
In this section, we first evaluate the effects of vanilla server- and client-centric asynchronous training compared with synchronous variants (Section~\ref{sec:exp:vanilla}). We then demonstrate the benefits of our proposed staleness-aware downplaying and delay compensation techniques in Section~\ref{sec:exp:effectiveness}, followed by additional studies on hyperparameters (Section~\ref{sec:exp:additional}). We assess the performance on a diverse set of empirical tasks, including 1) image classification on CIFAR-10~\citep{Krizhevsky2009LearningML}  using a ViT model~\citep{sharir2021image},  2) natural language understanding on the GLUE benchmark~\citep{wang2018glue} by fine-tuning a pre-trained RoBERTa-base model~\citep{liu2019robertarobustlyoptimizedbert}, and  3) machine translation on the WMT TED talk dataset \citep{cettolo-etal-2012-wit3} using a T5 model \citep{t5} to examine sequence-to-sequence generation performance. 
We simulate two scenarios with \textit{mild} and \textit{large} delays for each dataset. The specific experimental setup, runtime simulation, and hyperparameter tuning are laid out in Appendix~\ref{app:exp_details}. 
%; and 3) machine translation on the WMT~\citep{wmt} with a generative model T5~\citep{t5}, serving as a large-scale, computationally intensive sequence-to-sequence task.

\vspace{-0.4em}
\subsection{Tradeoffs of Asynchronous Training under Heavy-Tailed Noise} \label{sec:exp:vanilla}
\vspace{-0.4em}

\begin{figure}[b!]
    \centering
    \includegraphics[width=\linewidth]{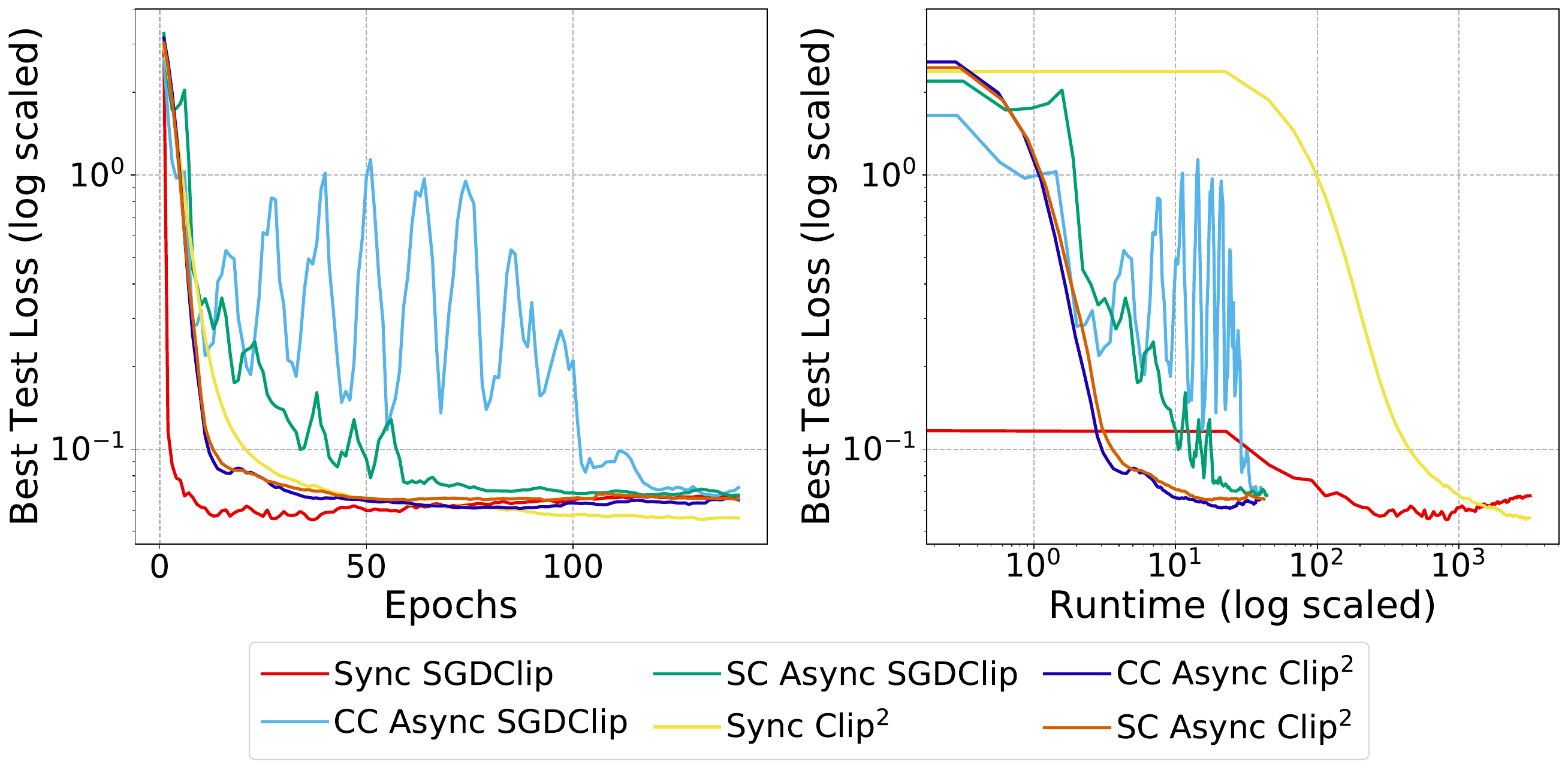}
  \caption{\small %Test loss versus \#epochs (left) and runtime (right) of synchronous SGDClip/$Clip^2$ and asynchronous SGDClip/$Clip^2$ with large stragglers. 
  In large straggler settings, vanilla asynchronous methods arrive at similar losses as synchronous versions with significantly less runtime, across different optimizers. Moreover, %comparing `CC Async SGDClip' with `CC Async Clip$^2$' (or `SC Async SGDClip' with `SC Async Clip$^2$'), 
  we see that 
    $Clip^2$ has inherent benefits to control the bias caused by asynchrony compared with SGDClip. We observe similar trends for the mild straggler setting (See Figure~\ref{fig:normal_async} in Appendix). %\dixi{minor: the color ofsync clip2 is a bit faint.} 
    }
    \label{fig:clip2_advantage}
\end{figure}

In this section, we examine the performance of the proposed vanilla asynchronous training under heavy-tailed noise. Figure~\ref{fig:clip2_advantage} shows the convergence results on CIFAR-10 and Table~\ref{tab:glue_sync_async} demonstrates the final model performance on GLUE. \textit{We see that under heavy-tailed noise, asynchronous training could achieve competitive performance (similar accuracies while requiring much less runtime) compared with the synchronous counterpart for both client-centric (CC) and server-centric (SC) variants.}  We will also see that all of our proposed methods achieve similar small runtime, which indicates that our methods achieve improved time complexity  (defined in \citet{tyurin2023optimaltimecomplexitiesparallel, tyurin2025tighttimecomplexitiesparallel,maranjyan2025ringmasterasgdasynchronoussgd}) compared to the synchronous counterpart. 
The occurrence of large stragglers, however, may degrade performance compared with mild stragglers, given the bias introduced by larger delay. The final model performance and average runtime results are shown in Tables~\ref{tab:1} and \ref{tab:2} in Appendix~\ref{app:exp_results}.

% \begin{figure}[b!]
% \vspace{\baselineskip}
%     \centering
%     \includegraphics[width=0.6\linewidth]{images/extreme_async_new.pdf}
%     \caption{\small Test loss versus \#epochs (left) and runtime (right) of synchronous SGDClip/$Clip^2$ and asynchronous SGDClip/$Clip^2$ with large stragglers. We observe that asynchronous methods provide similar loss as synchronous versions with significantly less runtime, across different optimizers including $Clip^2$ which is designed to handle heavy-tailed noise. Moreover, comparing `CC Async SGDClip' with `CC Async Clip$^2$' (or `SC Async SGDClip' with `SC Async Clip$^2$'), we see that 
%     $Clip^2$ has inherent benefits to control the bias caused by asynchrony. We observe similar trends for the mild straggler setting (Figure~\ref{fig:normal_async} in the appendix). %\dixi{minor: the color ofsync clip2 is a bit faint.} 
%     }
%     \label{fig:clip2_advantage}
%     \vspace{\baselineskip}
% \end{figure}

\textit{Moreover, $Clip^2$, originally developed to handle heavy-tailed noise, has an inherent benefit of handling bias introduced by asynchrony (Figures~\ref{fig:clip2_advantage} and~\ref{fig:extreme_async})}.  For instance, in Figure~\ref{fig:clip2_advantage},  asynchronous SGDClip introduces oscillations in the loss, which is caused by the bias from delays instead of the heavy-tailed noise, since there were no such oscillations under the synchronous setting. However, $Clip^2$ with a proper server clipping threshold flattens those oscillations and converges. %This shows that $Clip^2$ is indeed controlling the bias introduced by asynchrony.

\vspace{-0.4em}
\subsection{Effectiveness of Staleness-Aware Downplaying and Delay Compensation} \label{sec:exp:effectiveness}
\vspace{-0.4em}
In this part, we demonstrate the effectiveness of the two modifications we proposed (staleness-aware downplaying and delay compensation). Firstly, the average final model performance and runtime comparisons are presented in Table~\ref{tab:1}-Table~\ref{tab:4} in Appendix~\ref{app:exp_results}. \textit{Across all three datasets, we see that DC and SD can result in better best accuracies than vanilla asynchronous training}. A subset of GLUE benchmark results are presented in Table~\ref{tab:glue-dc-clip2}. 
We see that SD and DC give modest gains in the client-centric setting and are roughly neutral in the server-centric setting.
The simulated runtime for all asynchronous training runs  is similar for each dataset, and uniformly smaller than the synchronous counterpart (Tables~\ref{tab:glue-vanila}, \ref{tab:glue-sa}, and \ref{tab:glue-dc}).

\begin{table*}[!h]
\caption{\small Comparison of results on GLUE for server- and client-centric training scenarios with and without staleness-aware downplaying (SD) and delay compensation (DC). The base optimizer is $Clip^2$. We see that SD or DC (Algorithm~\ref{alg:da}) improves over %\tr{reword, "improves over" is not grammatically correct} 
the vanilla asynchronous algorithm (Algorithm~\ref{alg:vanilla_async}) in the client-centric setting and are roughly neutral in the server-centric setting.}
\resizebox{\textwidth}{!}{
\begin{tabular}{cccccccccccc}
\toprule
& Async Mode & Methods & MNLI & QNLI & QQP & RTE & SST-2 & MRPC & CoLA & STS-B & Avg\\
\toprule
\multirow{6}{*}{Mild Straggler} 
    %& Sync & -- & 83.67 & 87.33 & 79.86& 64.30 & 92.88 & 86.03 & 82.08 & 86.54 & 82.84 \\ \cline{2-12}
  & \multirow{3}{*}{Server} & vanilla & \textbf{83.54} & 84.72 & 86.33 & 60.29 & 92.66 & \textbf{86.03} & \textbf{81.11} & \textbf{88.14} & \textbf{82.85} \\
  && SD     &82.18 & \textbf{87.24} & 84.48 & \textbf{64.32} & \textbf{93.11} & 82.11 & 80.47& 87.33 &  82.66\\
  && DC     & 83.30 & 86.66 & \textbf{86.52} & 56.32 & 93.23 & 82.35 & 79.87 & 86.36 & 81.83 \\ \cline{2-12}
  & \multirow{3}{*}{Client} & vanilla & 81.87 & 86.86 & 85.37 & 70.04 & 91.63 & \textbf{87.01} & 81.02 & 87.14 & 83.87 \\
  && SD     & \textbf{83.42} & \textbf{86.91} & 86.10 & \textbf{70.76} & \textbf{92.20} & 85.78 & 80.35 & 86.19 & 84.00\\
  && DC     & 81.56 & 86.53 & \textbf{86.57} & 70.40 & 92.09 & 86.76 & \textbf{81.88} & \textbf{88.37} & \textbf{84.27} \\
\midrule
\multirow{6}{*}{Large Straggler} 
%& Sync & -- & 83.30 & 87.22 & 83.03 & 64.30 & 92.78 & 85.05 & 79.77 & 86.43 & 81.45 \\ \cline{2-12}
  & \multirow{3}{*}{Server} & vanilla & 82.58 & 85.78 & 86.09 & \textbf{59.95} & 91.97 & 83.58 & 80.44 & \textbf{85.63} & \textbf{82.00} \\
  && SD     & 82.50 & \textbf{87.59} & 84.47 & 54.51 & \textbf{92.78} & \textbf{84.31} & \textbf{81.40} & 85.60 & 81.65\\
  && DC     & \textbf{83.50} & 86.47 & \textbf{86.92} & 58.14 & \textbf{92.78} & 77.45 & 80.57 & 85.60 & 81.43 \\ \cline{2-12}
  & \multirow{3}{*}{Client} & vanilla & 82.19 & \textbf{86.89} & 85.39 & 67.87 & \textbf{92.29} & 85.05 & 80.15 & 86.92 & 83.34 \\
  && SD& \textbf{83.35} & 86.55 & \textbf{86.96} & \textbf{72.56} & 91.51 & \textbf{86.52} & 80.73 & 86.60 & \textbf{84.35} \\ 
  && DC& 82.29 & 86.36 & 84.44 & 66.43 & 91.63 & 86.03 & \textbf{81.78} & \textbf{87.91} & 83.36 \\
  \bottomrule
\end{tabular}}
\label{tab:glue-dc-clip2}
\vspace{-0.5em}
\end{table*}

In addition, we observe that for a large part of the hyperparameter choices in our sweep range that give non-converging loss curves for the vanilla asynchronous case, staleness-aware downplaying demonstrates converging behaviors instead. We present three specific hyperparameter choices that showcase this phenomenon in Figure~\ref{fig:extreme_sd} on the CIFAR-10 dataset. \textit{This shows that  staleness-aware downplaying effectively makes the choice of hyperparameter more robust under an asynchronous setting with heavy-tailed noise.}  However, such a benefit is not evident for $Clip^2$ as all hyperparameter choices for vanilla asynchronous $Clip^2$ already give converging results.
\vspace{-0.4em}
\begin{figure}[h!]
\vspace{-0.5em}
    \centering
    \includegraphics[width=\linewidth]{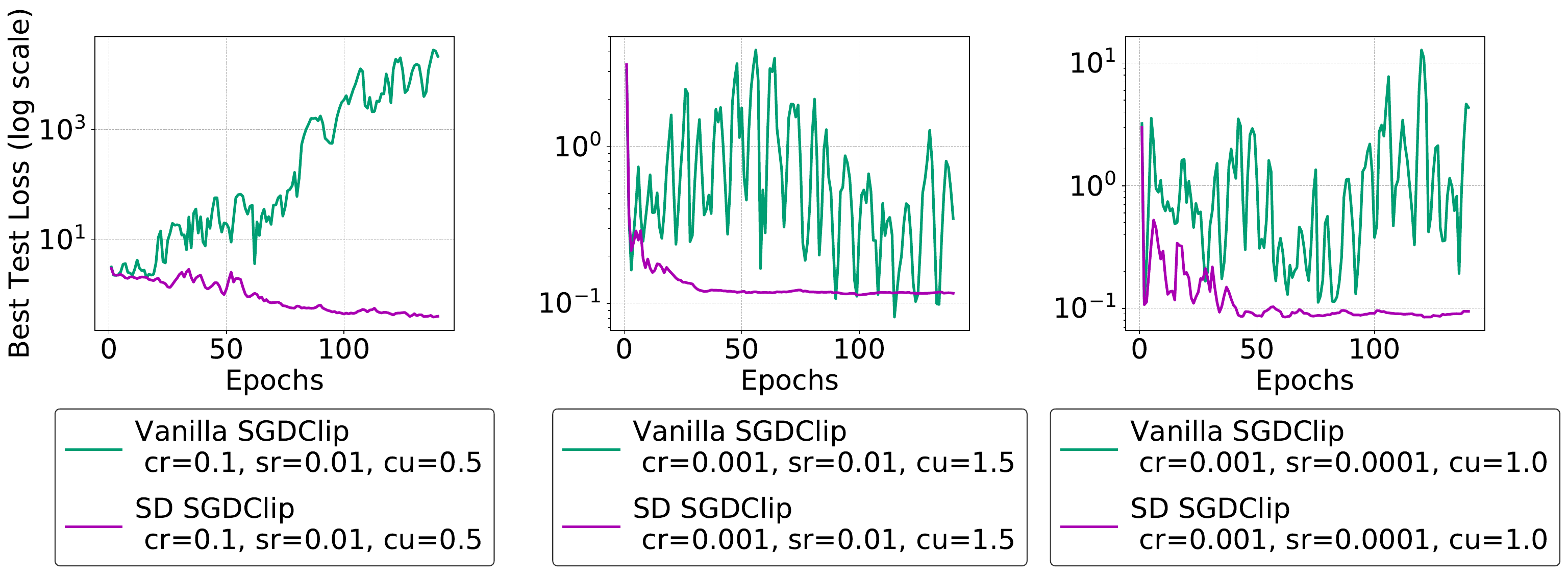}
    %\vspace{-0.2em}
    \caption{\small Test loss versus \#epochs of vanilla asynchronous method and staleness-aware downplaying (SD) under the SGDClip optimizer for three hyperparameter choices. SD largely improves the robustness of hyperparameter choosing. %\tian{can we breakdown the legend into three separate ones; also to make the text bigger, we can break down each legend into two lines, e.g., `vanilla SGDClip' and 'cr=0.1, sr=0.01, cu=0.5' two lines}%
    }
    \label{fig:extreme_sd}
    \vspace{-0.5em}
\end{figure}

%We perform experiments using RoBERTa, a pre-trained BERT model, on the GLUE benchmark. The optimal hyperparameter configurations are summarized in \cref{tab:optimal_hyp_glue1} and \cref{tab:optimal_hyp_glue2} in \cref{app:exp_details}. 
% The aggregated results, reported as averaged accuracy across GLUE tasks, are presented in \cref{tab:1}, \cref{tab:2}, \cref{tab:3}, and \cref{tab:4}. In addition, detailed results for each individual task—under different optimizers and settings, and with or without staleness-aware downplaying and delay compensation—are provided in \cref{tab:glue-vanila}, \cref{tab:glue-sa}, and \cref{tab:glue-dc}. 
\vspace{-1em}
\textbf{Comparison with Other Asynchronous Baselines.} Our results %above \tr{just say "our results," it's not clear what "above" refers to -- could be figures, could be the preceding paragraph, etc.} 
show the effectiveness of our methods compared to vanilla asynchronous and synchronous baselines \citep{lee2025efficient}. In addition, we consider three asynchronous baselines: FADAS \citep{wang2024fadas}, which uses client local SGD and a different server-side delay-aware aggregation with adaptive updates, DN+DyLU \citep{liu2024asynchronous}, which performs delayed Nesterov for server aggregation and dynamic local updates (DyLU) with SGD, and RANSGDm \citep{wu2026optimalasynchronousstochasticnonconvex}, which is an asynchronous normalized SGD method with momentum and a delay threshold specifically designed for heavy-tailed noise, thus most relevent to our setting. \textit{We can see that our proposed methods achieve higher accuracies compared to all the baselines in Tables~\ref{tab:cifar_baseline}, \ref{tab:ted_baseline}, and \ref{tab:glue-baseline} in the appendix}.  

%\tian{add a table on stalenss-aware downplaying}

%\tian{double check if stalenss-aware downplaying can help with stability across different hyperparameters for SGDClip on the GLUE benchmark}
\vspace{-0.4em}
\subsection{Additional Evaluation} \label{sec:exp:additional}
\vspace{-0.4em}
%We additionally study the effects of several important parameters.
\textbf{Effects of $M$.} Recall that $M$ denotes the asynchronous buffer. When $M=1$, %\tr{capital M?} 
the algorithm reduces to the extreme asynchronous version where the server-centric is equivalent to the client-centric variant. We examine the losses and runtimes for $M\in\{1, 10, 20, 30\}$. The full results are in Appendix~\ref{app:sub:m_effect}. In general, we notice that changing from $M=1$ to other values always improves the accuracy, whereas increasing $M$ from $10$ might not give better accuracy. On the other hand, the runtime always increases as $M$ increases. But the runtime of asynchronous methods are always significantly lower than the synchronous method regardless of $M$. %\tr{beginning a sentence with "But" and why is asynchronous capitalized?}

% \begin{table}[h!]
%     \centering
%     \begin{tabular}{c|c|c|c}
%        method  & $M=1$ & $M=10$ & $M=20$ \\
%          & & & \\
%     \end{tabular}
%     \caption{Effects of $M$ on the CIFAR-10 dataset. \tian{present subset of result here}}
%     \label{tab:placeholder}
% \end{table}

\textbf{Effects of Server-Centric and Client-Centric Variants.} We study the effects of the two asynchronous mode---server- and client-centric settings in terms of loss and runtime. From the results in Appendix~\ref{app:sub:exp_results}, we notice that in general, server-centric methods provide similar accuracies in comparison to client-centric methods, but require slightly longer runtime, as expected, given that client-centric methods do not force the clients to wait, making it less time-consuming.

\vspace{-0.4em}
\section{Conclusion \& Discussions} 
\label{sec:conclusion}
\vspace{-0.4em}
In this work, we explored asynchronous distributed optimization under heavy-tailed noise. We proposed two delay-aware strategies, staleness-aware downplaying and delay compensation, to improve final model performance and robustness to hyperparameters. We analyzed the convergence behaviour of our framework under the general setup of heavy-tailed noise, local optimization, and clipped-based optimizers. We are able to achieve the optimal convergence rate under heavy-tailed noise and improved delay tolerances compared with other asynchronous approaches. We empirically validated the effectiveness of our approach on image, text, and machine translation datasets. 
One limitation of our work is that our approaches are not evaluated in real-world, larger-scale distributed environments with at least thousands of heterogeneous computing nodes, which we leave for future work. % We believe the empirical evaluation of our proposed methods on such large scale tasks is an important future direction.
% In this work, all the theoretical results are developed to establish the first number-of-iterations convergence guarantee under the unique setting of asynchronous training under heavy-tailed noise. One limitation of our work is that we do not theoretically model the end-to-end time complexity benefits of asynchronous learning. It is interesting to theoretically analyze heavy-tailed asynchronous optimization incorporating varying and more fine-grained computation models~\citep{tyurin2023optimaltimecomplexitiesparallel, tyurin2025tighttimecomplexitiesparallel, maranjyan2025ringmasterasgdasynchronoussgd}.  %develops a new theoretical framework of \textit{time complexity} to evaluate asynchronous methods that focuses on a more nuanced guarantee that explicitly depends on all delays. The theoretical analysis of heavy-tailed asynchronous optimization under this framework is an interesting future direction. 
Moreover, it remains open to study the interplays between adaptivity and staleness, in particular, how and when adaptive methods provably mitigate issues of asynchronous training. %the theoretical guarantee of delay compensation does not recover the optimal $O(T^{-\frac{1}{2}})$ bound as $\alpha\to 2$. This is due to the difficulties in analyzing the update $\overline{g}_t$. As shown in Algorithm~\ref{alg:da}, the update $\overline{g}_t$ involves $x_{t-1}-x_{\tau_{t,i}}$

\bibliographystyle{apalike}
\bibliography{cite} 

% References follow the acknowledgements. Use an unnumbered third level
% heading for the references section.  Please use the same font
% size for references as for the body of the paper---remember that
% references do not count against your page length total.

% \begin{thebibliography}{}
% \setlength{\itemindent}{-\leftmargin}
% \makeatletter\renewcommand{\@biblabel}[1]{}\makeatother
% \bibitem{} J.~Alspector, B.~Gupta, and R.~B.~Allen (1989).
%     \newblock Performance of a stochastic learning microchip.
%     \newblock In D. S. Touretzky (ed.),
%     \textit{Advances in Neural Information Processing Systems 1}, 748--760.
%     San Mateo, Calif.: Morgan Kaufmann.

% \bibitem{} F.~Rosenblatt (1962).
%     \newblock \textit{Principles of Neurodynamics.}
%     \newblock Washington, D.C.: Spartan Books.

% \bibitem{} G.~Tesauro (1989).
%     \newblock Neurogammon wins computer Olympiad.
%     \newblock \textit{Neural Computation} \textbf{1}(3):321--323.
% \end{thebibliography}

%%%%%%%%%%%%%%%%%%%%%%%%%%%%%%%%%%%%%%%%%%%%%%%%%%%%%%%%%%%%

\clearpage
\appendix
\thispagestyle{empty}

% Supplementary material: To improve readability, you must use a single-column format for the supplementary material.
\onecolumn
\aistatstitle{Supplementary Materials}

\section{Convergence Result Comparison} \label{app:sub:big_table}

\begin{table*}[!h]
\centering
\begin{threeparttable}

\caption{Comparison of convergence rates and delay dependence for various optimization settings. Our rates are stated for $\E[\|\nabla F\|]$ whereas previous methods' rates are stated for $\E[\|\nabla F\|^2]$. Our rates match the optimal rate under heavy-tailed noise and are faster than the synchronous heavy-tailed baseline \citep{lee2025efficient}. Moreover, staleness-aware downplaying possesses the benefit of converging even with unbounded delay, which is different from prior works \citep{wang2024fadas,liu2024asynchronous}, and its $N$-adapted variant (adapting the server-side learning rate to $\max\{p_{t,i}, N\}$ as in Theorems~\ref{thm:sasgdwithN}, \ref{thm:clip2withN}) requires no assumption on the delays. 
Delay compensation preserves the rate and delay tolerance of the vanilla methods and reduces the delay-dependent error by a constant under accurate curvature estimates.}
\label{tab:convergence_comparison}

\begingroup
\setlength{\tabcolsep}{5pt}
\renewcommand{\arraystretch}{1.15}

\begin{tabularx}{\textwidth}{
    @{}
    >{\raggedright\arraybackslash}p{0.17\textwidth}
    >{\raggedright\arraybackslash}p{0.27\textwidth}
    >{\raggedright\arraybackslash}p{0.18\textwidth}
    >{\raggedright\arraybackslash}X
    @{}
}
\toprule
\textbf{Assumptions}
& \textbf{Method}
& \textbf{Convergence rate}
& \textbf{Delay dependence / tolerance} \\
\midrule

Bounded variance
& Synchronous local updates
& $T^{-1/2}$
& --- \\

\addlinespace[3pt]
Bounded variance
& Async \citep{wang2024fadas}
& $T^{-1/2}+T^{-1}\tau\tau_{\mathrm{avg}}$
& $O(\tau\tau_{\mathrm{avg}})$ \\

\addlinespace[3pt]
Bounded variance
& Async \citep{nguyen2022federated}
& $T^{-1/2}$
& $\tau \le O(T^{1/4})$ or $O(\tau^2)$ dependence \\

\midrule

Heavy-tailed noise\tnote{a}
& RMSPropClip (synchronous)
  \newline \citep{lee2025efficient}
& $T^{\frac{1-\alpha}{4\alpha-2}}$
& --- \\

\midrule
\multicolumn{4}{@{}l}{\textbf{This paper}} \\
\addlinespace[3pt]

Heavy-tailed noise
& Vanilla async
  \newline SGDClip / Clip$^2$
  \newline Theorems~\ref{thm:sgdclip}, \ref{thm:clip2}
& $T^{-\frac{\alpha-1}{3\alpha-2}}$
& $\tau \le O(T^{\frac{\alpha}{2(3\alpha-2)}})$ \\

\addlinespace[4pt]
Heavy-tailed noise
& Staleness-aware
  \newline SGDClip / Clip$^2$
  \newline Theorems~\ref{thm:sasgd}, \ref{thm:saclip2}
& $T^{-\frac{\alpha-1}{3\alpha-2}}$
& Depends on delays through
  $(\sum_t p_t^{-1})(\sum_t p_t^{-2})$;
  unrestricted by the largest delay
  (Propositions~\ref{prop:sasgd}, \ref{prop:saclip2}) \\

\addlinespace[4pt]
Heavy-tailed noise
& $N$-adapted staleness-aware
  \newline SGDClip / Clip$^2$
  \newline Theorems~\ref{thm:sasgdwithN}, \ref{thm:clip2withN}
& $T^{-\frac{\alpha-1}{3\alpha-2}}+NT^{-1/2}$
& No assumptions on delays (Theorems~\ref{thm:sasgdwithN}, \ref{thm:clip2withN})\\

\addlinespace[4pt]
Heavy-tailed noise;
\newline Lipschitz Hessian
& Delay-compensated Clip$^2$
  \newline Theorem~\ref{thm:dc}
& $T^{-\frac{\alpha-1}{3\alpha-2}}$
& $\tau \le O(T^{\frac{\alpha}{2(3\alpha-2)}})$;
  delay-dependent error reduced by a factor
  of order $(\varepsilon_H/L)^2$ \\

\bottomrule
\end{tabularx}
\endgroup

\begin{tablenotes}[flushleft]
\setlength{\itemsep}{2pt}

\item[] Rates are stated in big-$O$ notation,
with $O(\cdot)$ omitted in the convergence-rate column.
Our rates are for $\E\|\nabla F\|$.

\item[a] 
Without local updates,
\citet{wu2026optimalasynchronousstochasticnonconvex}
obtain the rate $O(T^{-\frac{\alpha-1}{3\alpha-2}})$
using normalized SGD with momentum, discarding updates
with delay $\ge R=\Theta(T^{\frac{\alpha}{3\alpha-2}})$. Our main table focuses on comparisons with approaches with local updates. See Section~\ref{sec:related} for discussions on broader literature.

\end{tablenotes}

\par\smallskip
\noindent\rule{\linewidth}{\heavyrulewidth}

\end{threeparttable}
\end{table*}

\section{Pseudocode of Vanilla Asynchronous Training}
We present the pseudocode of both client- and server-centric vanilla asynchronous training in Algorithm~\ref{alg:vanilla_async}.
$TailClip$ denotes the clipping-based local optimizer with clipping threshold $u_s$~\citep{lee2025efficient}, and $Outer\_Optimizer$ denotes any optimizer that the server uses to incorporate the model updates $\Delta_t$. For instance, it can perform clipping again on top of $\Delta_t$ before adding it to $x_{t-1}$.

\setlength{\textfloatsep}{4pt}
\setcounter{algocf}{0}
        \begin{algorithm}[!h]
        \DontPrintSemicolon
            \caption{Vanilla server- and client-centric asynchronous learning}
            \label{alg:vanilla_async}
            \KwIn{Initial model $x_0$, 
            local learning rate schedule $\eta_\ell^t$, clipping threshold $u_t \ge 0$, minimal number of model updates $M$ $(1\leq M\leq N)$}
            %\STATE $t\leftarrow 1$
            %\vspace{1em}
            \textbf{Server side: }
            \For{$t = 1, \dots, T$}{
                Collects first $M$ client models $\{x_{i, K}\}_{i \in [M]}$ along with $M$ time stamps\;
                Denotes the $M$ time stamps as $\{\tau_{t,i}\}_{i\in[M]}$ and denotes $\{x_{i, K}\}_{i \in [M]}$ as $\{{x}_i^{\tau_{t,i}}\}_{i \in [M]}$   %Denote those clients as $\mathcal{M}$ 
                % \tian{what are the differences between the notations $\overline{x}^{\tau_{t,i}}$ and $x_{\tau_{t,i}}$} \;
                %$\Delta_t \leftarrow \frac{1}{M}\Sigma_{i\in [M]}(\overline{x}^{\tau_{t,i}}-x_{t-1})$, 
                $\Delta_t \leftarrow \frac{1}{M}\Sigma_{i\in [M]}({x}_i^{\tau_{t,i}} - x_{\tau_{t,i}})$, ~
                $x_{t} \leftarrow Outer\_Optimizer \ (x_{t-1},\Delta_t)$ \;
                Sets the current global model as $x_t$  \\ 
                Sends $x_t$ to idle clients   \quad\quad \texttt{/*only for server-centric*/} 
            }
            %\vspace{1em}
            % \textbf{Client side: } 
            % \For{each client $i$ \textbf{in parallel}}{
            %     \While{TRUE} 
            %     {
            %    %  \While{Not Receiving global model from the server}
            %    % {Keep Listening}
            %     Waits to receive a global model  from the server at around $t$   ~~/*for server-centric*/ \;
            %     Pulls the current global model (assuming at global time stamp $t$) from the server and set it as $x_{i,0}$ ~~~/*for client-centric*/ \;
            %     \For{each local step $k \in [K]$}{
            %         Draw stochastic gradient $g_{i} \leftarrow \nabla F(x_{i,k}; \xi_{i,k})$ \;
            %         Run inner optimization: $x_{i,k} \leftarrow  x_{i,k-1} - \eta_\ell^t \cdot TailClip(u_t,g_{i})$ \;
            %     }
            %     Send $x_{i,z}$ to the server \;
            %     }}

                \textbf{Client side: } 
            \For{each client $i$ \textbf{in parallel}}{
                \While{TRUE} 
                {
               %  \While{Not Receiving global model from the server}
               % {Keep Listening}
                Waits to receive the latest global model (at global round $s$)  from the server  \quad \texttt{/*only for server-centric*/} \;
                Pulls the latest global model from the server   \quad\quad \texttt{/*only for client-centric*/} \;
                Sets the model to $x_{i,0}$ \;
                \For{each local step $k \in [K]$}{
                    Draws gradient $g_{i,k} \leftarrow \nabla F(x_{i,k}; \xi_{i,k})$ \;
                    Runs inner optimization: $x_{i,k} \leftarrow  x_{i,k-1} - \eta_\ell^t \cdot TailClip(u_{s},g_{i,k})$ \;
                }
                Sends $x_{i,K}$ and $s$ to the server \;
                }} 
        \end{algorithm}
    % \begin{minipage}[t]{0.48\textwidth}
    %     \begin{algorithm}[H]
    %         \caption{Client-Centric Async TailOPT}
    %         \label{alg:cctailopt}
    %         \KwIn{%Initial model $x_0$, 
    %         % local learning rate schedule $\eta_\ell^t$, clipping schedules $u_t \ge 0$, minimal number of model updates $M$ $(1\leq M\leq N)$
    %         }
    %         %\STATE $t\leftarrow 1$
    %         \vspace{1em}
    %         \textbf{Server side: }
    %         \For{$t = 1, \dots, T$}{
    %             Collecting $M$ client models and denote them $\{\overline{x}^{\tau_{t,i}}\}_{i \in [M]}$ \;
    %             $\Delta_t=\frac{1}{M}\Sigma_{i\in [M]}(\overline{x}^{\tau_{t,i}}-x_{t-1})$ \;
    %             $x_{t} = Outer\_Optimizer \ (x_{t-1},\Delta_t)$ \;
    %             Set the current global model ready for pull as $x_t$ \;
    %         }
    %         \vspace{1em}
    %         \textbf{Client side: } 
    %         \For{each client $i$ \textbf{in parallel}}{
    %             \While{TRUE}{
    %             Pulls the current global model (assuming at global time stamp $t$) and the global time stamp $t$ from the server and set it as $x_{i,0}$ \;
    %             \For{each local step $k \in [K]$}{
    %                 Draw stochastic gradient $g_{i} \leftarrow \nabla F(x_{i,k}; \xi_{i,k})$ \;
    %                 Inner optimization: $x_{i,k} \leftarrow  x_{i,k-1} - \eta_\ell^t \cdot TailClip(u_t,g_{i})$ \;
    %             }
    %             Send $x_{i,z}$ to the server \;
    %             }}
    %     \end{algorithm}
    % \end{minipage}

\section{Complete Proofs}
\label{app:proofs}
% Moreover, similarly to \citep{lee2025efficient}, throughout our analysis, we carry out the analysis over a sufficiently large, compact domain $\mathcal{X}$. This means by applying EVT, we obtain a universal bound $G$ on the deterministic gradients $\nabla F(\cdot)$ that we encounter during the analysis. This is a standard assumption necessary for the control of heavy-tailed noise in analysis, similarly done in relevant papers \citep{HighProbAdaGConv, inprobsupreme, CentralClip, lee2025efficient, heavytail2}.

Before we proceed with the proofs, we define the progress measure
\begin{equation}\label{eq:Phi}
  \Phi(s):=\min\{s,u/2\}\cdot s,\qquad s\ge0,
\end{equation}
and
\[
  B:=D^\alpha u^{1-\alpha},\qquad
  V:=D^\alpha u^{2-\alpha},\qquad
  c_\alpha:=5\cdot4^{\alpha-1},\qquad
  C_1:=8c_\alpha^2+\tfrac{4^\alpha}{64}.
\]
Moreover, we denote
$\bar g_t = \frac{1}{M}\sum_{m=1}^M g_{t,m}$,
$\psi_{t,m,k}:=Clip\big(u,\nabla F(x^{\tau_{t,m}}_{k-1},\xi^{\tau_{t,m}}_k)\big)$,
$a_t:=\nabla F(x_{t-1})$,
$b_{t,m,k}:=\nabla F(x^{\tau_{t,m}}_{k-1})$,
$\mu_{t,m,k}:=\E[\psi_{t,m,k}]$,
$\Phi_t:=\E\big[\Phi(\|\nabla F(x_{t-1})\|)\big]$,
$\delta_{t,m}:=\|x_{t-1}-x_{\tau_{t,m}}\|$, and
$W_t:=\{t-\tau+1,\dots,t-1\}$.

We also let $G_{t,k}$ denote the
$\sigma$-algebra generated by $x_0$, 
all stochastic noise in the local updates aggregated at rounds
$1,\dots,t-1$, and
$\xi^{\tau_t}_1,\dots,\xi^{\tau_t}_{k-1}$.

\begin{lemma} \label{lem:Phi}
    
    \item[(a)] For all $s,\delta\ge 0$, $\Phi_u(s+\delta)\le 2\Phi_u(s)+2\delta^2$.
        Consequently, for all $a,b\in\mathbb{R}^d$,
        \begin{equation}\label{eq:Phi-shift}
          \Phi_u(\|a\|)\le 2\Phi_u(\|b\|)+2\|a-b\|^2 .
        \end{equation}
  \item[(b)] For all $0<\theta\le u/2$ and $s\ge 0$,
        $s\le \theta+\Phi_u(s)/\theta$.
\end{lemma}
\begin{proof}
(a) We consider three cases.
Firstly, if $s+\delta\le u/2$,  then we have that
$\Phi(s+\delta)=(s+\delta)^2\le2s^2+2\delta^2=2\Phi(s)+2\delta^2$.
If we have $s<u/2<s+\delta$, then realize that
$\Phi(s+\delta)=\tfrac u2(s+\delta)<(s+\delta)^2\le2s^2+2\delta^2=2\Phi(s)+2\delta^2$.
Finally, if we have $s\ge u/2$, then $\Phi(s+\delta)=\tfrac u2s+\tfrac u2\delta$ and
$\Phi(s)=\tfrac u2s\ge\tfrac{u^2}{4}$. Now, notice that
$\tfrac u2\delta\le\tfrac{u^2}{16}+\delta^2\le\tfrac u2s+\delta^2$, and therefore
$\Phi(s+\delta)\le us+\delta^2\le2\Phi(s)+2\delta^2$.
This concludes the proof.
 
(b) If $s\le\theta$, the claim is trivial. If $s>\theta$, then
$\min\{s,u/2\}\ge\theta$ because $\theta\le u/2$, so we have $\Phi(s)\ge\theta s$, which is $s\le\Phi(s)/\theta$. This concludes the proof.
\end{proof}

\begin{lemma}\label{lem:clip}
Let $G$ be a $\sigma$-algebra, say $b\in\R^d$ be $G$-measurable, and 
$\xi\in\R^d$ satisfy $\E[\xi\mid G]=0$ and
$\E[\|\xi\|^\alpha\mid G]\le D^\alpha$ for some $\alpha\in(1,2)$ and $D\ge0$.
For $u>0$, we let $\psi:=Clip(u,b+\xi)$ and $\mu:=\E[\psi\mid G]$. Then we have
\begin{align}
  \langle b,\mu\rangle&\ge\tfrac18\Phi(\|b\|)-8c_\alpha^2B^2,\label{eq:clip-i}\\
  \|\mu\|^2&\le8\Phi(\|b\|)+2\cdot4^\alpha B^2,\label{eq:clip-ii}\\
\E\big[\|\psi\|^2\mid G\big]&\le4\Phi(\|b\|)+2V.\label{eq:clip-iii}
\end{align}
Moreover, if $\|b\|\le u/2$, then $\|b-\mu\|\le2^\alpha B$.
\end{lemma}

\begin{proof}
We first write $g:=b+\xi$, so that $\E[g\mid G]=b$, and let $m:=\min\{\|b\|,u/2\}$.
All expectations below are conditional on $G$.
 
We first consider the case where gradients are small, i.e., $\|b\|\le u/2$.
Notice that given $g-\psi=(1-u/\|g\|)\,g\,\mathbf 1\{\|g\|>u\}$, we have
$\|g-\psi\|\le\|g\|\,\mathbf 1\{\|g\|>u\}$. When $\|g\|>u$, we have
$\|\xi\|\ge\|g\|-\|b\|>u/2\ge\|b\|$, and hence $\|g\|\le\|b\|+\|\xi\|\le2\|\xi\|$.
Therefore
\[
  \|g-\psi\|\le2\|\xi\|\,\mathbf 1\{\|\xi\|>u/2\}\le2\|\xi\|^\alpha(u/2)^{1-\alpha},
\]
where we use $\|\xi\|^{1-\alpha}\le(u/2)^{1-\alpha}$ when $\|\xi\|>u/2$. Now, given that $b-\mu=\E[g-\psi\mid G]$, Jensen's inequality gives
\begin{equation}\label{eq:clip-bias}
  \|b-\mu\|\le2(u/2)^{1-\alpha}D^\alpha=2^\alpha B,
\end{equation}
which proves the last claim. Consequently, recalling that under this small gradients case, we have $m=\|b\|$, we obtain
\[
  \langle b,\mu\rangle=\|b\|^2-\langle b,b-\mu\rangle\ge\|b\|^2-2^\alpha B\|b\|
  \ge m\big(\tfrac14\|b\|-c_\alpha B\big).
\]
 
We now move to consider the large gradients case: $\|b\|>u/2$.
We denote $\psi=\lambda g$ with $\lambda:=\min\{1,u/\|g\|\}\in(0,1]$. If $\|\xi\|\le\|b\|/2$, then we have
$\|g\|\le\tfrac32\|b\|$, hence $\lambda\|b\|\ge\min\{\|b\|,\tfrac23u\}\ge u/2$,
and $\langle b,\xi\rangle\ge-\|b\|^2/2$. Therefore,
\[
  \langle b,\psi\rangle=\lambda\big(\|b\|^2+\langle b,\xi\rangle\big)
  \ge\tfrac12\lambda\|b\|^2\ge\tfrac u4\|b\|.
\]
When $\|\xi\|>\|b\|/2$, we use $\langle b,\psi\rangle\ge-\|b\|\|\psi\|\ge-u\|b\|$.
Combining these two scenarios, we have that
$\langle b,\psi\rangle\ge\tfrac u4\|b\|-\tfrac54u\|b\|\,\mathbf 1_{\{\|\xi\|>\|b\|/2\}}$.
Now, by the conditional Markov inequality, we have
$\Pr(\|\xi\|>\|b\|/2\mid G)\le2^\alpha D^\alpha\|b\|^{-\alpha}$, so
\[
  \langle b,\mu\rangle\ge\tfrac u4\|b\|-\tfrac54\,2^\alpha D^\alpha u\|b\|^{1-\alpha}
  \ge\tfrac u4\|b\|-\tfrac54\,2^\alpha D^\alpha u\,(u/2)^{1-\alpha}.
\]
Since $u=2m$ in this case, we have that
$\tfrac54\,2^\alpha D^\alpha u\,(u/2)^{1-\alpha}=\tfrac54\,4^\alpha Bm=c_\alpha Bm$
and $\tfrac u4\|b\|=\tfrac12m\|b\|$. Hence
$\langle b,\mu\rangle\ge m\big(\tfrac14\|b\|-c_\alpha B\big)$.
 
\emph{Proof of Eq.~\ref{eq:clip-i}:}
Both cases above give $\langle b,\mu\rangle\ge\tfrac14m\|b\|-c_\alpha Bm$. If
$\|b\|\ge8c_\alpha B$, then $c_\alpha Bm\le\tfrac18m\|b\|$. Otherwise,
$c_\alpha Bm\le c_\alpha B\|b\|<8c_\alpha^2B^2$. Either way, we end up with
$c_\alpha Bm\le\tfrac18m\|b\|+8c_\alpha^2B^2$, and therefore
$\langle b,\mu\rangle\ge\tfrac18m\|b\|-8c_\alpha^2B^2
=\tfrac18\Phi(\|b\|)-8c_\alpha^2B^2$.
 
\emph{Proof of Eq.~\ref{eq:clip-ii}:}
In the small gradients case, Eq.~\ref{eq:clip-bias} gives $\|\mu\|\le\|b\|+2^\alpha B$, so
$\|\mu\|^2\le2\|b\|^2+2\cdot4^\alpha B^2=2\Phi(\|b\|)+2\cdot4^\alpha B^2$.
In the large gradients case, we have $\|\mu\|\le\E[\|\psi\|\mid G]\le u$ and
$\Phi(\|b\|)=\tfrac u2\|b\|>\tfrac{u^2}{4}$, so $\|\mu\|^2\le u^2<4\Phi(\|b\|)$.
 
\emph{Proof of Eq.~\ref{eq:clip-iii}:}
In the large gradients case, we simply note again that $\|\psi\|^2\le u^2<4 \Phi(\|b\|)$. In the small gradients case,
we uses $\|\psi\|=\min\{\|g\|,u\}\le\|b\|+\min\{\|\xi\|,u\}$, which in turn gives us
$\|\psi\|^2\le2\|b\|^2+2\min\{\|\xi\|,u\}^2$. We now notice that $\alpha\le2$ and $x\ge0$, we
have $\min\{x,u\}^2\le x^\alpha u^{2-\alpha}$. Therefore, we have
$\E[\min\{\|\xi\|,u\}^2\mid G]\le D^\alpha u^{2-\alpha}=V$, and
$\E[\|\psi\|^2\mid G]\le2\Phi(\|b\|)+2V$. This concludes the proof.
\end{proof}

\subsection{Vanilla Asynchronous $SGDClip$ and $Clip^2$} \label{app:sub:sgdandclip2}
\subsubsection{Vanilla Asynchronous $SGDClip$}

\begin{theorem} \label{thm:sgdclip}
Assume smoothness of $F(\cdot)$. Assume the stochastic noise satisfies that $\mathbb{E}[\|\xi_k^{\tau_t}\|^{\alpha} \mid G_{t,k}] \leq D^{\alpha}$ for $\alpha \in (1,2)$. Suppose we use $Clip$ as
the inner optimizer and SGD as the outer optimizer with any buffer size
$1\le M\le N$. Let the learning rates satisfy $\eta\,\eta_\ell K\le\frac{1}{256\,L}$. Then if 
\begin{equation}\label{eq:delay-cond}
  \tau\le\frac{1}{374\,L\,\eta\,\eta_\ell K},
\end{equation}
we have that
\begin{equation}\label{eq:main-bound}
\begin{aligned}
  \frac1T\sum_{t=1}^T\E\big[\Phi(\|\nabla F(x_{t-1})\|)\big]
  \le{}&\frac{128\,\Delta}{T\eta\eta_\ell K}
  +128\,C_1D^{2\alpha}u^{2-2\alpha}
  +17408\,L^2K^2\eta_\ell^2u^2\\
  &+128\,LK\eta\eta_\ell D^\alpha u^{2-\alpha}\big(1+272\,LK\eta\eta_\ell\,\tau^2\big).
\end{aligned}
\end{equation}
\end{theorem}

% \begin{theorem} \label{thm:sgdclip}
% Let $\eta_t = \Theta(t^\omega), \eta_t^\ell = \Theta(t^\nu), u_t = \Theta(t^\zeta)$. $\zeta > 0$, $ \omega + \nu \geq -1, \omega \leq 0$. Assume smoothness of $F(\cdot)$. Assume the stochastic noise satisfies that $\mathbb{E}[\|\xi\|^{\alpha}] \leq D^{\alpha}$ for $\alpha \in (1,2)$. Suppose we use Clip as the inner optimizer and SGD as the outer optimizer; we have that
% % \begin{align}
% % \min_{t \in [T]} \mathbb{E}[\|\nabla F(x_{t-1})\|^2] \lesssim O\left(
% % T^{(2-\alpha)\zeta + \nu + \omega} + T^{(1-\alpha)\zeta} + T^{\nu + \zeta}  + \tau \cdot T^{\omega + \nu + \zeta} +T^{-1 - \omega - \nu}
% % \right).
% % \end{align}    
% \begin{align}
%     \min_{t \in [T]} \mathbb{E}[\|\nabla F(x_{t-1})\|^2] \lesssim O\left(
% T^{(2-\alpha)\zeta + \nu + \omega} + T^{(1-\alpha)\zeta} + T^{\nu} +\tau \cdot T^{\omega + \nu} + T^{-1 - \omega - \nu}
% \right).
% \end{align}   
% \end{theorem}

\begin{proof}
    
We have that
$x_t-x_{t-1}=-\eta\bar g_t$. Due to $L$-smoothness of $F(\cdot)$, we have
\begin{equation}\label{eq:descent}
  \E F(x_t)\le\E F(x_{t-1})-\eta\,\E\langle\nabla F(x_{t-1}),\bar g_t\rangle
  +\frac{L\eta^2}{2}\E\|\bar g_t\|^2 .
\end{equation}
Different from the conventional analysis, we don't compare the clipped stochastic gradients with the unclipped ones, since the resulting
clipping bias can't be controlled without a bound on $\|\nabla F(\cdot)\|$.
Instead, we decompose the inner product as
\begin{equation}\label{eq:decomp}
  \langle\nabla F(x_{t-1}),\bar g_t\rangle
  =\frac{\eta_\ell}{M}\sum_{m=1}^M\sum_{k=1}^K\Big(
  \underbrace{\langle b_{t,m,k},\psi_{t,m,k}\rangle}_{A_{1,m,k}}
  +\underbrace{\langle a_t-b_{t,m,k},\psi_{t,m,k}\rangle}_{A_{2,m,k}}\Big).
\end{equation}

Our goal now becomes to bound $\E[A_{1,m,k}]$, $\E[A_{2,m,k}]$, and $\E\|\bar g_t\|^2$
in terms of the distance between the current global model and the
local iterates: 
\[
  \mathrm{dist}_{t,m,k}:=\|x_{t-1}-x^{\tau_{t,m}}_{k-1}\|,
\]
which we will bound afterwards. 
 
Firstly, to bound $\E[A_{1,m,k}]$, we apply Lemma~\ref{lem:clip} with $b=b_{t,m,k}$, and
$\xi=\xi^{\tau_{t,m}}_k$) and get
\begin{equation}\label{eq:A1}
  \E[A_{1,m,k}]=\E[\langle b_{t,m,k},\mu_{t,m,k}\rangle]
  \ge\tfrac18\E\big[\Phi(\|b_{t,m,k}\|)\big]-8c_\alpha^2B^2 .
\end{equation}
 
Now, to bound $\E[A_{2,m,k}]$, 
we similarly have
\begin{equation}\label{eq:A2}
\begin{aligned}
  \E[A_{2,m,k}]&=\E\langle a_t-b_{t,m,k},\mu_{t,m,k}\rangle\\&\ge-\tfrac1{128}\E\|\mu_{t,m,k}\|^2-32\,\E\|a_t-b_{t,m,k}\|^2\\
  &\ge-\tfrac1{16}\E\big[\Phi(\|b_{t,m,k}\|)\big]-\tfrac{4^\alpha}{64}B^2
  -32L^2\,\E[\mathrm{dist}_{t,m,k}^2].
\end{aligned}
\end{equation}
by Young's
inequality and Eq.~\ref{eq:clip-ii} and $L$-smoothness.
 
By Eq.~\ref{eq:Phi-shift}, we have $\Phi(\|b_{t,m,k}\|)\ge\tfrac12\Phi(\|a_t\|)-L^2\mathrm{dist}_{t,m,k}^2$. Combining this with the previous bounds in Eq.~\ref{eq:A1} and Eq.~\ref{eq:A2},
we get
\begin{equation}\label{eq:A1A2}
\begin{aligned}
  \E[A_{1,m,k}+A_{2,m,k}]
  &\ge\tfrac1{16}\E\big[\Phi(\|b_{t,m,k}\|)\big]-C_1B^2
  -32L^2\E[\mathrm{dist}_{t,m,k}^2]\\
  &\ge\tfrac1{32}\Phi_t-C_1B^2-33L^2\E[\mathrm{dist}_{t,m,k}^2].
\end{aligned}
\end{equation}
 
Finally, we still need to bound $\E\|\bar g_t\|^2$:

\begin{equation}\label{eq:g-sq}
\begin{aligned}
  \E\|\bar g_t\|^2
  &\le\frac1M\sum_{m=1}^M\E\|g_{t,m}\|^2
  \le\frac{\eta_\ell^2K}{M}\sum_{m=1}^M\sum_{k=1}^K\E\|\psi_{t,m,k}\|^2\\
  &\le\frac{\eta_\ell^2K}{M}\sum_{m=1}^M\sum_{k=1}^K
  \Big(4\E\big[\Phi(\|b_{t,m,k}\|)\big]+2V\Big)\\
  &\le\frac{\eta_\ell^2K}{M}\sum_{m=1}^M\sum_{k=1}^K
  \Big(8\Phi_t+8L^2\E[\mathrm{dist}_{t,m,k}^2]+2V\Big).
\end{aligned}
\end{equation}

where we used Eq.~\ref{eq:clip-iii}, and Eq.~\ref{eq:Phi-shift} as
$\Phi(\|b_{t,m,k}\|)\le2\Phi(\|a_t\|)+2L^2\mathrm{dist}_{t,m,k}^2$.

Now, say $R_t$ is a bound for $\E[\mathrm{dist}_{t,m,k}^2]$. By substituting
Eq.~\ref{eq:A1A2} and Eq.~\ref{eq:g-sq} into Eq.~\ref{eq:descent}, and recalling $\gamma=\eta\eta_\ell$, we have
\[
  \E F(x_t)\le\E F(x_{t-1})-\gamma K\Big(\tfrac1{32}-4L\gamma K\Big)\Phi_t
  +\gamma KC_1B^2+\gamma K\big(33+4L\gamma K\big)L^2R_t+L\gamma^2K^2V .
\]
By the condition $\eta\ell_\ell K\le \frac{1}{256L}$, we have $4L\gamma K\le1/64$. Therefore,
\begin{equation}\label{eq:one-step}
  \E F(x_t)\le\E F(x_{t-1})-\frac{\gamma K}{64}\Phi_t+\gamma KC_1B^2
  +34\gamma KL^2R_t+L\gamma^2K^2V .
\end{equation}

It remains to calculate $R_t$. We decompose the distance into the drift of the local
optimization and the staleness of the global model:
\begin{align}
    \E[\mathrm{dist}_{t,m,k}^2]&\le2\E\|x^{\tau_{t,m}}_{k-1}-x_{\tau_{t,m}}\|^2
  +2\E\|x_{\tau_{t,m}}-x_{t-1}\|^2\\
  &\le2(\sum_{i=1}^{k-1}\eta_\ell\|\psi_{t,m,i}\|)^2+2(\eta\big\|\sum_{s=\tau_{t,m}+1}^{t-1}\bar g_s\big\|)^2\\
  &\le 2K^2\eta_\ell^2 u^2+2\tau\eta^2\sum_{s\in W_t}\|\bar g_s\|^2
\end{align}
So we can let
\begin{align}
    R_t:= 2K^2\eta_\ell^2 u^2+2\tau\eta^2\sum_{s\in W_t}\|\bar g_s\|^2.
\end{align}

Now, we simply applying Eq.~\ref{eq:g-sq} at round $s$, which is
$\E\|\bar g_s\|^2\le\eta_\ell^2K^2(8\Phi_s+8L^2R_s+2V)$, and get 
\begin{equation}\label{eq:Rt-rec}
  R_t\le2K^2\eta_\ell^2u^2+2\tau\gamma^2K^2\sum_{s\in W_t}
  \big(8\Phi_s+8L^2R_s+2V\big).
\end{equation}
We now sum this equation over $t=1,\dots,T$ and notice that $s\in W_t$
requires $t\in\{s+1,\dots,s+\tau-1\}$, so each index $s$ appears in at most
$\tau$ of the sets $W_t$. This gives us that
\[
  \sum_{t=1}^TR_t\le2TK^2\eta_\ell^2u^2+2\tau^2\gamma^2K^2\sum_{s=1}^T
  \big(8\Phi_s+8L^2R_s+2V\big).
\]
Notice that the delay condition Eq.~\ref{eq:delay-cond} implies $16L^2\tau^2\gamma^2K^2\le1/2$, so we
can move the $R_s$ terms to the left-hand side and obtain
\begin{equation}\label{eq:Rt-sum}
  \sum_{t=1}^TR_t\le4TK^2\eta_\ell^2u^2+32\tau^2\gamma^2K^2\sum_{t=1}^T\Phi_t
  +8T\tau^2\gamma^2K^2V .
\end{equation}

With the bound for $\sum_{t=1}^TR_t$ settled, we can simply examine the telescope sum of Eq.~\ref{eq:one-step} and get 
\[
\begin{aligned}
  \frac{\gamma K}{64}\sum_{t=1}^T\Phi_t
  \le{}&\Delta+T\gamma KC_1B^2+T L\gamma^2K^2V\\
  &+34\gamma KL^2\Big(4TK^2\eta_\ell^2u^2+32\tau^2\gamma^2K^2\sum_{t=1}^T\Phi_t
  +8T\tau^2\gamma^2K^2V\Big).
\end{aligned}
\]
The delay condition Eq.~\ref{eq:delay-cond} gives
$1088\,(L\gamma K\tau)^2\le1/128$, so the terms on the right-hand side that involve  $\Phi_t$
are at most $\tfrac{\gamma K}{128}\sum_t\Phi_t$. Therefore, we can similarly move them to the left-hand
side and obtain
\[
  \frac{\gamma K}{128}\sum_{t=1}^T\Phi_t
  \le\Delta+T\gamma KC_1B^2+136\,T\gamma L^2K^3\eta_\ell^2u^2
  +T L\gamma^2K^2V\big(1+272\,L\gamma K\tau^2\big).
\]
Dividing both sides by $T\gamma K/128$  proves
Eq.~\ref{eq:main-bound}.
\end{proof}

From this, if we let 
$\eta=\Theta(1)$, $\eta_\ell=\Theta(T^\nu)$, and $u=\Theta(T^\zeta)$ with
\begin{equation}\label{eq:exponents}
  \nu=-\frac{\alpha}{3\alpha-2},\qquad\zeta=\frac{1}{3\alpha-2}.
\end{equation}

Then if the delay satisfies $\tau\le O\big(T^{\frac{\alpha}{2(3\alpha-2)}}\big)$, it can be checked that
\[
  \frac1T\sum_{t=1}^T\Phi_t\le O\Big(T^{-\frac{2(\alpha-1)}{3\alpha-2}}\Big).
\]

By Lemma~\ref{lem:Phi} with
$\theta=T^{-\frac{\alpha-1}{3\alpha-2}}$, which satisfies $\theta\le u/2$ for
sufficiently large $T$, we have that
\[ \label{eq:sgdclip_rate}
  \frac1T\sum_{t=1}^T\E\|\nabla F(x_{t-1})\|
  \le\theta+\frac1{\theta}\cdot\frac1T\sum_{t=1}^T\Phi_t
  \le O\Big(T^{-\frac{\alpha-1}{3\alpha-2}}\Big).
\]

\begin{remark}
    We notice that the convergence guarantee matches the lower bound for stochastic
    first-order methods under heavy-tailed noise with bounded $\alpha$-th 
    moment ~\citep{zhang2020adaptive}. Moreover, as $\alpha\to 2$, we have
    $O(T^{-1/4})$, which is the same as the optimal guarantee achieved by SGD under the bounded variance assumption.
\end{remark}

Moreover, we notice that $M$ only contributed an additional averaging to the proof above. Therefore, for the following proofs in this section, we will simplify the setting and conduct all the analysis in this section with $M=1$, given that a general $M$ simply requires additional averaging in the analysis and complicates notation. The analysis can be easily generalized similarly.

\subsubsection{Vanilla Asynchronous $Clip^2$}
Under the setting of $Clip^2$, given the client-side clipping, we have $\|g_t\|\le K\eta_\ell u$. Therefore, the
server-side clipping can only be active when $\tilde u<K\eta_\ell u$. We here assume that
\begin{equation}\label{eq:server-thr}
  \tilde u\ge\kappa K\eta_\ell u\qquad\text{for some constant }\kappa\in(0,1].
\end{equation}
Moreover, we write $Clip(\tilde u,g_t)=\lambda_tg_t$ with
$\lambda_t:=\min\{1,\tilde u/\|g_t\|\}$. Immediately, we can see that by Eq.~\ref{eq:server-thr}, we have
$\kappa\le\lambda_t\le1$.

\begin{theorem}\label{thm:clip2}
Assume smoothness of $F(\cdot)$. Assume the stochastic noise satisfies that $\mathbb{E}[\|\xi\|^{\alpha}] \leq D^{\alpha}$ for $\alpha \in (1,2)$. Suppose we use Clip as both the inner optimizer the outer optimizer with a server-side threshold
satisfying Eq.~\ref{eq:server-thr}. Moreover, suppose
\begin{equation}\label{eq:step-cond-clip2}
  \eta\,\eta_\ell K\le\frac{\kappa}{512\,L},\qquad
  \eta_\ell K\le\frac{\kappa}{46\,L},\qquad
  u\ge D\Big(\frac{1024}{\kappa^2}\Big)^{1/\alpha},
\end{equation}
then if 
\begin{equation}\label{eq:delay-cond-clip2}
  \tau\le\frac{\kappa}{536\,L\,\eta\,\eta_\ell K},
\end{equation}
we have that
\begin{equation}\label{eq:main-bound-clip2}
\begin{aligned}
  \frac1T\sum_{t=1}^T\E\big[\Phi(\|\nabla F(x_{t-1})\|)\big]
  \le{}&\frac{256\,\Delta}{\kappa\,T\eta\eta_\ell K}
  +\frac{256\,(C_1+512)}{\kappa^4}D^{2\alpha}u^{2-2\alpha}
  +\frac{34832}{\kappa^2}L^2K^2\eta_\ell^2u^2\\
  &+\frac{256}{\kappa}LK\eta\eta_\ell D^\alpha u^{2-\alpha}
  \Big(1+\frac{280}{\kappa}LK\eta\eta_\ell\,\tau^2\Big),
\end{aligned}
\end{equation}
\end{theorem}

\begin{proof}
Notice that
$x_t-x_{t-1}=-\eta\lambda_t\overline{g}_t$, then by $L$-smoothness and $\lambda_t\le1$, we have
\begin{equation}\label{eq:descent-clip2}
  \E F(x_t)\le\E F(x_{t-1})-\eta\,\E\big[\lambda_t\langle\nabla F(x_{t-1}),\overline{g}_t\rangle\big]
  +\frac{L\eta^2}{2}\E\|\overline{g}_t\|^2 .
\end{equation}
The last term is bounded in Eq.~\ref{eq:g-sq} already. Therefore, we focus on the middle term. We note that $\lambda_t$ is random and depends on the noise in
$\overline g_t$. To handle it, let $G^0_t$ be the $\sigma$-algebra generated by $x_0$, and all stochastic noise up to round $t-1$. As a result, 
$x_{t-1}$ and $x_{\tau_t}$ are $G^0_t$-measurable and
$G^0_t\subseteq G_{t,k}$. We denote
\[
  \tilde g_t:=\E[\overline{g}_t\mid G^0_t],\qquad
  \tilde\lambda_t:=\min\{1,\tilde u/\|\tilde g_t\|\},\qquad
  \bar d_t:=\frac1K\sum_{k=1}^K\E\big[\mathrm{dist}_{t,k}^2\mid G^0_t\big].
\]
We then decompose the middle term as follows:
\begin{equation}\label{eq:decomp-clip2}
  \lambda_t\langle\nabla F(x_{t-1}),\overline{g}_t\rangle
  =\underbrace{\lambda_t\langle a_t,\tilde g_t\rangle}_{B_1}
  +\underbrace{\tilde\lambda_t\langle a_t,\overline{g}_t-\tilde g_t\rangle}_{B_2}
  +\underbrace{(\lambda_t-\tilde\lambda_t)\langle a_t,\overline{g}_t-\tilde g_t\rangle}_{B_3}.
\end{equation}
Now, our goal becomes bounding $B_1$, $B_2$, and $B_3$.
 
Firstly, we bound $B_1$. By the tower property, $\langle a_t,\tilde g_t\rangle
=\eta_\ell\sum_{k=1}^K\E[\langle a_t,\mu_{t,k}\rangle\mid G^0_t]$. We here use
Eq.~\ref{eq:A1A2} in the proof of Theorem~\ref{thm:sgdclip},
i.e., $\langle a_t,\mu_{t,k}\rangle\ge\frac1{32}\Phi(\|a_t\|)-C_1B^2
-33L^2\mathrm{dist}_{t,k}^2$, and get
$\langle a_t,\tilde g_t\rangle\ge K\eta_\ell\big(\frac1{32}\Phi(\|a_t\|)-E_t\big)$
with $E_t:=C_1B^2+33L^2\bar d_t\ge0$. Since $\kappa\le\lambda_t\le1$, we have
\begin{equation}\label{eq:B1}
  B_1\ge K\eta_\ell\Big(\frac{\kappa}{32}\Phi(\|a_t\|)-C_1B^2-33L^2\bar d_t\Big).
\end{equation}
 
We now move to bound $B_2$.
we simply notice that $\tilde\lambda_t$ and $a_t$ are $G^0_t$-measurable. Therefore, we have
$\E[B_2]=\E[\tilde\lambda_t\langle a_t,\E[\overline{g}_t-\tilde g_t\mid G^0_t]\rangle]=0$.
 
Finally, we bound $B_3$.
It can be checked that 
$|\lambda_t-\tilde\lambda_t|\le\|\overline{g}_t-\tilde g_t\|/\tilde u$, which gives us
\[
  \E[B_3\mid G^0_t]\ge-\frac{\|a_t\|}{\tilde u}\E\big[\|\overline{g}_t-\tilde g_t\|^2\mid G^0_t\big].
\]
To bound the conditional variance, we compare
$\overline g_t$ with $K\eta_\ell Clip(u,\nabla F(x_{\tau_t}))$. Since $Clip(u,\cdot)$ is $1$-Lipschitz, we have that
\begin{align}
    \|\psi_{t,k}- Clip(u,\nabla F(x_{\tau_t}))\|
\le L\|x^{\tau_t}_{k-1}-x_{\tau_t}\|+\min\{\|\xi^{\tau_t}_k\|,2u\}
\le LK\eta_\ell u+\min\{\|\xi^{\tau_t}_k\|,2u\}
\end{align}

On the other hand, it's obvious that $\|\psi_{t,k}- Clip(u,\nabla F(x_{\tau_t}))\|\leq 2u$. Combining these together gives us that
\begin{align}
    \|\psi_{t,k}- Clip(u,\nabla F(x_{\tau_t}))\|\le \min \{2u, L\,\mathrm{dist}_{t,k}+\|\xi_k^{\tau_t}\|\}\le L\,\mathrm{dist}_{t,k} + \min \{\|\xi_k^{\tau_t}\|, 2u\}.
\end{align}

Meanwhile, we have that $\E[\min\{\|\xi^{\tau_t}_k\|,2u\}^2\mid G_{t,k}]\le D^\alpha(2u)^{2-\alpha}\le2V$.
Therefore, 
\[
  \E\big[\|\overline{g}_t-\tilde g_t\|^2\mid G^0_t\big]
  \le\E\big[\|\overline{g}_t-K\eta_\ell Clip(u,a_t)\|^2\mid G^0_t\big]
  \le K^2\eta_\ell^2Z,
  \qquad Z:=2L^2K^2\eta_\ell^2u^2+4V.
\]
Combining with $\tilde u\ge\kappa K\eta_\ell u$, we obtain
$\E[B_3\mid G^0_t]\ge-\frac{K\eta_\ell\,\|a_t\|Z}{\kappa u}$. 

Now, we first notice that
$Z\le\kappa^2u^2/128$, since Eq.~\ref{eq:step-cond-clip2} gives
$2L^2K^2\eta_\ell^2u^2\le\kappa^2u^2/256$ and $4V\le\kappa^2u^2/256$.

Then, we consider two cases.
If $\|a_t\|\le u/2$, then $\Phi(\|a_t\|)=\|a_t\|^2$, and Young's inequality gives
$\frac{\|a_t\|Z}{\kappa u}\le\frac{\kappa}{64}\Phi(\|a_t\|)+\frac{16Z^2}{\kappa^3u^2}$.
If $\|a_t\|>u/2$, then $\Phi(\|a_t\|)=\frac u2\|a_t\|$, and
$Z\le\kappa^2u^2/128$ gives
$\frac{\|a_t\|Z}{\kappa u}\le\frac{\kappa}{64}\Phi(\|a_t\|)$. Moreover, using
$V^2/u^2=B^2$ and $L^2K^2\eta_\ell^2\le\kappa^2/2116$, we have
$\frac{16Z^2}{\kappa^3u^2}\le\frac{512V^2}{\kappa^3u^2}+\frac{128L^4K^4\eta_\ell^4u^2}{\kappa^3}
\le\frac{512}{\kappa^3}B^2+\frac1{16\kappa}L^2K^2\eta_\ell^2u^2$ Combining these two cases, we have
\begin{equation}\label{eq:B3}
  \E[B_3\mid G^0_t]\ge-K\eta_\ell\Big(\frac{\kappa}{64}\Phi(\|a_t\|)
  +\frac{512}{\kappa^3}B^2+\frac1{16\kappa}L^2K^2\eta_\ell^2u^2\Big).
\end{equation}
 
Finally, we can combine the bounds for $B_1$ to $B_3$. Since $\kappa\le1$ and
$\E[\bar d_t]\le R_t$, we get
\[
  \E\big[\lambda_t\langle\nabla F(x_{t-1}),\overline{g}_t\rangle\big]
  \ge K\eta_\ell\Big(\frac{\kappa}{64}\Phi_t-\frac{C_1+512}{\kappa^3}B^2
  -33L^2R_t-\frac1{16\kappa}L^2K^2\eta_\ell^2u^2\Big).
\]
Substituting this and Eq.~\ref{eq:g-sq} into Eq.~\ref{eq:descent-clip2}, and using
$4L\gamma K\le\kappa/128$ from Eq.~\ref{eq:step-cond-clip2}, we have
\begin{equation}\label{eq:one-step-clip2}
  \E F(x_t)\le\E F(x_{t-1})-\frac{\kappa\gamma K}{128}\Phi_t
  +\frac{C_1+512}{\kappa^3}\gamma KB^2+34\gamma KL^2R_t
  +\frac{\gamma K}{16\kappa}L^2K^2\eta_\ell^2u^2+L\gamma^2K^2V .
\end{equation}
Notice that this is simply Eq.~\ref{eq:one-step} with different constants. Therefore, the rest of the proof is
the same as that of Theorem~\ref{thm:sgdclip}. Since
$\|Clip(\tilde u,\overline{g}_s)\|\le\|\overline{g}_s\|$, the distance bound Eq.~\ref{eq:Rt-sum} still
holds. We eventually obtain
\[
  \frac{\kappa\gamma K}{256}\sum_{t=1}^T\Phi_t
  \le\Delta+\frac{C_1+512}{\kappa^3}T\gamma KB^2
  +\Big(136+\frac1{16\kappa}\Big)T\gamma KL^2K^2\eta_\ell^2u^2
  +TL\gamma^2K^2V\big(1+272\,L\gamma K\tau^2\big).
\]
Dividing both sides by $\kappa T\gamma K/256$ proves Eq.~\ref{eq:main-bound-clip2}.
\end{proof}
 
It can be checked that if we let $\eta$,
$\eta_\ell$, and $u$ be chosen as in the SGDClip case, and let
$\tilde u=\kappa K\eta_\ell u$ for a constant $\kappa\in(0,1]$, then when
$\tau\le O\big(T^{\frac{\alpha}{2(3\alpha-2)}}\big)$, then for all sufficiently
large $T$, we again have
$\frac1T\sum_{t=1}^T\E\|\nabla F(x_{t-1})\|\le O\big(T^{-\frac{\alpha-1}{3\alpha-2}}\big)$.

\subsection{Staleness-Aware Downplaying} \label{app:sub:SA}

\subsubsection{SGDClip with Staleness-Aware Downplaying}

Let $p_t:=t-\tau_t\ge1$ denote the delay of the model update
received at round $t$. Staleness-aware downplaying replaces the server-side
learning rate $\eta$ by $\eta_t:=\eta/p_t$, i.e.,
\begin{equation}\label{eq:server-sd}
  x_t=x_{t-1}-\frac{\eta}{p_t}\,\overline{g}_t,
  \qquad
  \overline{g}_t=\eta_\ell\sum_{k=1}^K\psi_{t,k}.
\end{equation}
Notation-wise, we write $\gamma:=\eta\eta_\ell$ and $\gamma_t:=\eta_t\eta_\ell=\gamma/p_t$. Moreover, we
denote
\[
  P_1:=\sum_{t=1}^T\frac1{p_t},\qquad
  P_2:=\sum_{t=1}^T\frac1{p_t^2},\qquad
  \Lambda:=\max_{s\in[T]}\ \sum_{t:\,\tau_t<s<t}\ \frac1{p_t}\sum_{i\in W_t}\frac1{p_i}.
\]

\begin{theorem}\label{thm:sd-sgdclip}
Assuming $F(\cdot)$ being $L$-smooth. Let $u$ denote the client-side upper-clipping threshold, $K$ be the client epoch, $\eta$ be the server-side learning rate, and $\eta_\ell$ be the client-side learning rate. Thestochastic noise satisfies that $\mathbb{E}[\|\xi\|^{\alpha}] \leq D^{\alpha}$ for $\alpha \in (1,2)$, $SGDClip$ with staleness-aware downplaying satisfies. If we set
\begin{equation}\label{eq:sd-stepsize}
  \eta\eta_\ell=\min\bigg\{\sqrt{\frac{\Delta}{LK^2D^\alpha u^{2-\alpha}P_2}},\
  \frac{1}{374\,LK\sqrt{\max\{1,\Lambda\}}}\bigg\},
\end{equation}
then we have that
\begin{equation}\label{eq:sd-bound}
\begin{aligned}
  \min_{t\in[T]}\E\big[\Phi(\|\nabla F(x_{t-1})\|)\big]
  \le{}&\frac{256\sqrt{L\Delta D^\alpha u^{2-\alpha}P_2}}{P_1}
  +\frac{47872\,L\Delta\sqrt{\max\{1,\Lambda\}}}{P_1}\\
  &+128\,C_1D^{2\alpha}u^{2-2\alpha}+17408\,L^2K^2\eta_\ell^2u^2
  +\frac{34816\,L\Delta\Lambda}{P_2}.
\end{aligned}
\end{equation}
Moreover, if the delays satisfy
\begin{equation}\label{eq:sd-delay-cond}
  P_2\ge\frac{139876\,L\Delta\max\{1,\Lambda\}}{D^\alpha u^{2-\alpha}},
\end{equation}
then the minimum in Eq.~\ref{eq:sd-stepsize} will be attained by its first term, and
the second term in Eq.~\ref{eq:sd-bound} can be dropped.
\end{theorem}
 
\begin{proof}
Substituting $\eta$ with $\eta_t=\eta/p_t$, $\gamma$ with
$\gamma_t$, into the final step of the SGDClip
convergence analysis, i.e., Eq.~\ref{eq:one-step}, we have that
\begin{equation}\label{eq:sd-one-step}
  \E F(x_t)\le\E F(x_{t-1})-\frac{\gamma_tK}{64}\Phi_t+\gamma_tKC_1B^2
  +34\gamma_tKL^2R_t+L\gamma_t^2K^2V.
\end{equation}
Recall that this is valid as long as $\gamma K\le1/(256L)$.
We will verify this condition at the end of the proof.

Now, simlarly to the proof of Theorem~\ref{thm:sgdclip}, we have
$\mathrm{dist}_{t,k}\le K\eta_\ell u+\delta_t$ with
$\delta_t=\|x_{t-1}-x_{\tau_t}\|$, so we can take
$R_t:=2K^2\eta_\ell^2u^2+2\E[\delta_t^2]$. The difference is that the stale
server-side steps are now downplayed. Consequently, we have
$\delta_t=\|\sum_{s\in W_t}\eta_s\overline{g}_s\|$, and by the
Cauchy-Schwarz inequality,
\[
  \delta_t^2\le\Big(\sum_{s\in W_t}\eta_s\Big)\sum_{s\in W_t}\eta_s\|\overline{g}_s\|^2 .
\]
We now apply the bound for $\|\overline g_s\|$ Eq.~\ref{eq:g-sq}. This gives us that
\begin{equation}\label{eq:sd-Rt}
  R_t\le2K^2\eta_\ell^2u^2+2K^2\Big(\sum_{s\in W_t}\gamma_s\Big)
  \sum_{s\in W_t}\gamma_s\big(8\Phi_s+8L^2R_s+2V\big).
\end{equation}
Since Eq.~\ref{eq:sd-one-step} is weighted by $\gamma_t$, we need to bound
$\sum_t\gamma_tR_t$. We notice that
\[
  \sum_{t=1}^T\gamma_t\Big(\sum_{s\in W_t}\gamma_s\Big)\sum_{s\in W_t}\gamma_sX_s
  =\gamma^2\sum_{s=1}^T\gamma_sX_s\sum_{t:\,\tau_t<s<t}\frac1{p_t}\sum_{i\in W_t}\frac1{p_i}
  \le\gamma^2\Lambda\sum_{s=1}^T\gamma_sX_s
\]
for any $X_s\ge0$. We now combine this with Eq.~\ref{eq:sd-Rt} and get
\[
  \sum_{t=1}^T\gamma_tR_t\le2K^2\eta_\ell^2u^2\sum_{t=1}^T\gamma_t
  +2K^2\gamma^2\Lambda\sum_{s=1}^T\gamma_s\big(8\Phi_s+8L^2R_s+2V\big).
\]
If $16L^2K^2\gamma^2\Lambda\le1/2$, we can again move the $R_s$ terms to the left-hand
side and obtain
\begin{equation}\label{eq:sd-Rt-sum}
  \sum_{t=1}^T\gamma_tR_t\le4K^2\eta_\ell^2u^2\sum_{t=1}^T\gamma_t
  +32K^2\gamma^2\Lambda\sum_{t=1}^T\gamma_t\Phi_t
  +8K^2\gamma^2\Lambda V\sum_{t=1}^T\gamma_t .
\end{equation}

Finally, we again use the telescoping sum of Eq.~\ref{eq:sd-one-step}, together
with Eq.~\ref{eq:sd-Rt-sum}. This results in
\[
  \frac K{64}\sum_{t=1}^T\gamma_t\Phi_t
  \le\Delta+KC_1B^2\sum_{t=1}^T\gamma_t+LK^2V\sum_{t=1}^T\gamma_t^2
  +34KL^2\Big(4K^2\eta_\ell^2u^2\sum_{t=1}^T\gamma_t
  +32K^2\gamma^2\Lambda\sum_{t=1}^T\gamma_t\Phi_t
  +8K^2\gamma^2\Lambda V\sum_{t=1}^T\gamma_t\Big).
\]
If $1088\,L^2K^2\gamma^2\Lambda\le1/128$, which also implies
$16L^2K^2\gamma^2\Lambda\le1/2$, the $\Phi_t$ terms on the right-hand side can be moved to the left-hand side, and we have
\[
  \frac K{128}\sum_{t=1}^T\gamma_t\Phi_t
  \le\Delta+\gamma P_1\Big(KC_1B^2+136K^3L^2\eta_\ell^2u^2
  +272K^3L^2\gamma^2\Lambda V\Big)+LK^2\gamma^2P_2V .
\]
Since $\min_t\Phi_t\le\sum_t\gamma_t\Phi_t/\sum_t\gamma_t$,  we divide both sides
by $K\gamma P_1/128$ and obtain
\begin{equation}\label{eq:sd-general}
  \min_{t\in[T]}\Phi_t
  \le\frac{128\,\Delta}{K\gamma P_1}+128\,C_1B^2+17408\,L^2K^2\eta_\ell^2u^2
  +128\,LK\gamma V\Big(\frac{P_2}{P_1}+272\,LK\gamma\Lambda\Big).
\end{equation}
 
Now, what remains is to choose the learning rates. Let $\gamma_1$ and $\gamma_2$ denote the two terms in the minimum
Eq.~\ref{eq:sd-stepsize}, so that $\gamma=\min\{\gamma_1,\gamma_2\}$. Since
$\gamma\le\gamma_2$, we have $\gamma K\le1/(374L)\le1/(256L)$ and
$1088\,L^2K^2\gamma^2\Lambda\le1088/374^2\le1/128$, so the two conditions used
above hold, and Eq.~\ref{eq:sd-general} is valid without any assumption on the
delays. We will now bound its terms to conclude the proof. Since
$\frac1\gamma\le\frac1{\gamma_1}+\frac1{\gamma_2}$, the first term is at most
$\frac{128\sqrt{L\Delta VP_2}}{P_1}+\frac{47872\,L\Delta\sqrt{\max\{1,\Lambda\}}}{P_1}$.
Since $\gamma\le\gamma_1$, the $P_2/P_1$ term is at most
$\frac{128\sqrt{L\Delta VP_2}}{P_1}$, and the $\Lambda$ term is at most
$128\cdot272\,L^2K^2\gamma_1^2V\Lambda=\frac{34816\,L\Delta\Lambda}{P_2}$.
This gives
Eq.~\ref{eq:sd-bound}. 
\end{proof}

\begin{remark}\label{rem:sd-Lambda}
We notice that $\Lambda$ is always bounded by the number of clients $N$,
regardless of the largest delay because $p_i\ge1$, and each summand satisfies
$\frac1{p_t}\sum_{i\in W_t}\frac1{p_i}\le\frac{p_t-1}{p_t}<1$. Since each client computes at most one
update at a time, there are at most $N$ of the summands. Therefore, $\Lambda\le N$, and
the bound Eq.~\ref{eq:sd-bound} and the condition Eq.~\ref{eq:sd-delay-cond}
depend on the delays only through $P_1$, $P_2$. In
particular, a few very large delays have no significant impact on the values of $P_1$ and $P_2$, which grants
staleness-aware downplaying the ability to converge even when the largest delay is unbounded.
\end{remark}

The theorem above and the corollary below together give Theorem~\ref{thm:sasgd}.

\begin{corollary}\label{cor:sd-sgdclip}
If we let $u=\Theta(T^\zeta)$ and $\eta_\ell=\Theta(T^\nu)$ with $\zeta>0$ and
$\nu<0$, and choose $\eta\eta_\ell$ as in Eq.~\ref{eq:sd-stepsize}, then SGDClip
with staleness-aware downplaying satisfies
\begin{equation}\label{eq:sd-cor}
  \min_{t\in[T]}\E\big[\Phi(\|\nabla F(x_{t-1})\|)\big]
  \le O\bigg(T^{\frac{2-\alpha}{2}\zeta}\,\frac{\sqrt{P_2}}{P_1}
  +\frac{\sqrt{\max\{1,\Lambda\}}}{P_1}
  +T^{2(1-\alpha)\zeta}+T^{2\nu+2\zeta}+\frac{\Lambda}{P_2}\bigg).
\end{equation}
In particular, if all $p_t$'s take the same value $p\le O(T^{\frac{\alpha}{2(3\alpha-2)}})$,
then setting $\zeta=\frac1{3\alpha-2}$ and $\nu=-\frac{\alpha}{3\alpha-2}$ recovers
\begin{equation}\label{eq:sd-rate}
  \min_{t\in[T]}\E\|\nabla F(x_{t-1})\|\le O\Big(T^{-\frac{\alpha-1}{3\alpha-2}}\Big).
\end{equation}
\end{corollary}
 
\begin{proof}
The bound Eq.~\ref{eq:sd-cor} follows directly from Eq.~\ref{eq:sd-bound} by
plugging in $u=\Theta(T^\zeta)$ and $\eta_\ell=\Theta(T^\nu)$. If all $p_t$'s are
equal to $p$, then $P_1=T/p$, $P_2=T/p^2$, and hence $\sqrt{P_2}/P_1=T^{-1/2}$.
Moreover, we can also see that $\Lambda/P_2\le p^2/T$. With
$\zeta=\frac1{3\alpha-2}$ and $\nu=-\frac{\alpha}{3\alpha-2}$, we can check that
\[
  \frac{(2-\alpha)\zeta-1}{2}=2(1-\alpha)\zeta=2\nu+2\zeta=-\frac{2(\alpha-1)}{3\alpha-2},
\]
and $p^2/T\le O(T^{\frac{\alpha}{3\alpha-2}-1})=O(T^{-\frac{2(\alpha-1)}{3\alpha-2}})$
by the assumption on $p$. Therefore,
$\min_t\E[\Phi(\|\nabla F(x_{t-1})\|)]\le O(T^{-\frac{2(\alpha-1)}{3\alpha-2}})$.
It can also be checked that the condition Eq.~\ref{eq:sd-delay-cond} is satisfied. Finally, we apply Lemma~\ref{lem:Phi} with
$\theta=T^{-\frac{\alpha-1}{3\alpha-2}}\le u/2$ to the round $t$ attaining the
minimum gives Eq.~\ref{eq:sd-rate}.
\end{proof}

The following proposition is used to compare the bound in Corollary~\ref{cor:sd-sgdclip} and that of FADAS, which is \begin{align} \label{eq:fadas}
    \frac{1}{\sum\limits_{t=1}^{T} \eta_t} \sum\limits_{t=1}^{T} \eta_t \mathbb{E}[ \lVert \nabla f(\mathbf{x}_t) \rVert^2 ] \leq \mathcal{O} \left( \frac{\sqrt{\mathcal{F}\sigma}}{\sqrt{TKM}} + \frac{\sqrt{\mathcal{F}\sigma_g}}{\sqrt{TM}} + \frac{\mathcal{F}G\tau_c}{T\sqrt{M}} + \frac{\mathcal{F}\tau_{\text{avg}}}{T} + \frac{\mathcal{F}(\tau_c^2 + \tau_c\tau_{\text{avg}})}{T} \right).
\end{align}
where $\mathcal{F}=F(x_0)-F^*$ for $F^*=min_x F(x)$, $\tau_{avg}$ and $\tau_c$ are the average delay and median delay, respectively, and if we assume that each stochastic gradient is unbiased and has a local variance bounded by $\sigma$, and the loss function on each client has a global variance bound $\sigma_g$. 

On the other hand, our proposition will show that the bound in
Theorem~\ref{thm:sd-sgdclip} converges even when the largest delay, and the
average and median delays, are large, as long as a fraction of the delays is
small.

\begin{proposition*}[\ref{prop:sasgd}]
    Suppose that among all $p_t$'s for $1\le t\le T$, there is a fraction $q\in(0,1]$
of delays that satisfy $p_t\le A$, and let $T_{\rm eff}:=qT/A$. Let
$u=\Theta(T_{\rm eff}^{\frac1{3\alpha-2}})$ and
$\eta_\ell=\Theta(T_{\rm eff}^{-\frac{\alpha}{3\alpha-2}})$, and choose $\eta\eta_\ell$
as in Eq.~\ref{eq:sd-stepsize}. Then for all sufficiently large $T_{\rm eff}$, SGDClip
with staleness-aware downplaying satisfies
\begin{equation}\label{eq:prop-sd}
  \min_{t\in[T]}\E\|\nabla F(x_{t-1})\|
  \le O\bigg(T_{\rm eff}^{-\frac{\alpha-1}{3\alpha-2}}+\sqrt{\frac{NA}{T_{\rm eff}}}\bigg).
\end{equation}
In particular, if $A=O(1)$, and $N\le O(T^{\frac{\alpha}{3\alpha-2}})$,
then $\min_t\E\|\nabla F(x_{t-1})\|\le O(T^{-\frac{\alpha-1}{3\alpha-2}})$, regardless of
the remaining delays. More generally, if $A\le O(T^c)$ with
$c<\frac12$, then
$\min_t\E\|\nabla F(x_{t-1})\|\le O\big(T^{-(1-c)\frac{\alpha-1}{3\alpha-2}}+\sqrt N\,T^{c-\frac12}\big)$.
\end{proposition*}

\begin{proof}
By the assumption on the delays, we note that $P_1\ge qT/A=T_{\rm eff}$ and
$P_2\ge qT/A^2=T_{\rm eff}/A$. Moreover, $P_2\le P_1$ since $p_t\ge1$, and
$\max\{1,\Lambda\}\le N$ by Remark~\ref{rem:sd-Lambda}. Therefore, we can substitute these bounds into
Eq.~\ref{eq:sd-bound} and get
\[
  \frac{\sqrt{L\Delta VP_2}}{P_1}\le\sqrt{\frac{L\Delta V}{T_{\rm eff}}},\qquad
  \frac{L\Delta\sqrt{\max\{1,\Lambda\}}}{P_1}\le\frac{L\Delta\sqrt N}{T_{\rm eff}},\qquad
  \frac{L\Delta\Lambda}{P_2}\le\frac{L\Delta NA}{T_{\rm eff}} .
\]
Now if we set $u=\Theta(T_{\rm eff}^{\frac1{3\alpha-2}})$ and
$\eta_\ell=\Theta(T_{\rm eff}^{-\frac{\alpha}{3\alpha-2}})$, the exact same computation as in
the bound of SGDClip with $T$ replaced by $T_{\rm eff}$ shows that
$\sqrt{V/T_{\rm eff}}$, $D^{2\alpha}u^{2-2\alpha}$, and $L^2K^2\eta_\ell^2u^2$ all equal
$O(T_{\rm eff}^{\frac{2(1-\alpha)}{3\alpha-2}})$. Finally, we just notice that $\frac{\sqrt N}{T_{\rm eff}}\le\frac{NA}{T_{\rm eff}}$ and we have that the right-hand side $R$ of
Eq.~\ref{eq:sd-bound} satisfies
$R\le O\big(T_{\rm eff}^{\frac{2(1-\alpha)}{3\alpha-2}}+\frac{NA}{T_{\rm eff}}\big)$.
Now, to show Eq.~\ref{eq:prop-sd}, we simply apply Lemma~\ref{lem:Phi}(d)  as in
Corollary~\ref{cor:sd-sgdclip}.
\end{proof}

We next show that the assumptions on the delays can be removed entirely if we
further adapt the downplaying to the total number of clients $N$.
 
\begin{theorem}\label{thm:sd-N}
With the same assumption as Theorem~\ref{thm:sasgd}, if we
change the downplaying in Eq.~\ref{eq:server-sd} to
$\eta_t=\eta\,w_t$ with $w_t:=\min\{1/N,1/p_t\}$, and set
\begin{equation}\label{eq:sd-N-stepsize}
  \eta\eta_\ell=\min\bigg\{N\sqrt{\frac{\Delta}{LK^2VT}},\ \frac1{374\,LK}\bigg\},
\end{equation}
then SGDClip with staleness-aware
downplaying satisfies
\begin{equation}\label{eq:sd-N-bound}
\begin{aligned}
  \min_{t\in[T]}\E\big[\Phi(\|\nabla F(x_{t-1})\|)\big]
  \le{}&640\sqrt{\frac{L\Delta D^\alpha u^{2-\alpha}}{T}}
  +\frac{191488\,L\Delta N}{T}+\frac{34816\,L\Delta N^2}{T}\\
  &+128\,C_1D^{2\alpha}u^{2-2\alpha}+17408\,L^2K^2\eta_\ell^2u^2 .
\end{aligned}
\end{equation}
\end{theorem}

\begin{proof}
We first notice that the proof of Theorem~\ref{thm:sd-sgdclip} uses the
downplaying weights $1/p_t$ only through $\gamma_t=\gamma/p_t\le\gamma$, $\gamma_t\sum_{s\in W_t}\gamma_s=\gamma^2\frac1{p_t}\sum_{i\in W_t}\frac1{p_i}$,
and $\sum_t\gamma_t=\gamma P_1$, $\sum_t\gamma_t^2=\gamma^2P_2$. Therefore, we can have
Eq.~\ref{eq:sd-general} for any weights $w_t\in(0,1]$ instead of $1/p_t$. The task now is simply to bound these
three quantities for $w_t=\min\{1/N,1/p_t\}$.
 
Firstly, we bound $\Lambda$. Since $w_i\le1/N$ and $|W_t|=p_t-1$, we have $w_t\sum_{i\in W_t}w_i\le\frac1{p_t}\cdot\frac{p_t-1}{N}<\frac1N$ for each summand. As
in Remark~\ref{rem:sd-Lambda}, there are at most $N$ rounds $t$ with
$\tau_t<s<t$. Therefore, we obtain that $\Lambda<1$.
 
We now move to bound $P_1$ and $P_2$. Similarly to \citep{mishchenko2023asynchronoussgdbeatsminibatch}, we notice
that $p_t=|\{s:\tau_t<s\le t\}|$, so
\[
  \sum_{t=1}^Tp_t=\sum_{s=1}^T\big|\{t:\tau_t<s\le t\}\big|\le NT.
\]
Therefore, at most
$T/2$ rounds have $p_t>2N$. By the definition of $w_t$, this implies that at least $T/2$ rounds have
$w_t\ge\frac1{2N}$. This gives us that $P_1\ge\frac T{4N}$. By Cauchy-Schwarz,
$P_2\ge P_1^2/T\ge\frac T{16N^2}$.
 
We now combine all the bounds. Similarly, we let $\gamma_1$ and $\gamma_2$ denote the two terms in the
minimum Eq.~\ref{eq:sd-N-stepsize}. Again, since $\gamma\le\gamma_2$ and $\Lambda<1$, we
have $\gamma K\le1/(256L)$ and $1088\,L^2K^2\gamma^2\Lambda\le1/128$, so
Eq.~\ref{eq:sd-general} holds. We now bound its terms as follows. Moreover, by
$P_1\ge\frac T{4N}$ and $\frac1\gamma\le\frac1{\gamma_1}+\frac1{\gamma_2}$, we have that the
first term is at most
$\frac{512N\Delta}{KT}\big(\frac1{\gamma_1}+\frac1{\gamma_2}\big)
=512\sqrt{\frac{L\Delta V}T}+\frac{191488\,L\Delta N}{T}$.
And by $P_2/P_1\le1/N$ and $\gamma\le\gamma_1$, we have that the $P_2/P_1$ term is at most
$\frac{128LK\gamma_1V}{N}=128\sqrt{\frac{L\Delta V}T}$. Finally, by $\Lambda<1$ and
$\gamma\le\gamma_1$, the $\Lambda$ term is at most
$34816\,L^2K^2\gamma_1^2V=\frac{34816\,L\Delta N^2}{T}$. Combining all these bounds
gives us Eq.~\ref{eq:sd-N-bound}.
\end{proof}

\begin{corollary*}[\ref{thm:sasgdwithN}]
Under the same assumptions as Theorem~\ref{thm:sd-N}, let $u=\Theta(T^\zeta)$
and $\eta_\ell=\Theta(T^\nu)$ with $\zeta=\frac1{3\alpha-2}$ and
$\nu=-\frac{\alpha}{3\alpha-2}$. Then
\[
  \min_{t\in[T]}\E\big[\Phi(\|\nabla F(x_{t-1})\|)\big]
  \le O\Big(T^{-\frac{2(\alpha-1)}{3\alpha-2}}+\frac{N^2}{T}\Big).
\]
In particular, if
$N\le O(T^{\frac{\alpha}{2(3\alpha-2)}})$, then
$\min_{t\in[T]}\E\|\nabla F(x_{t-1})\|\le O(T^{-\frac{\alpha-1}{3\alpha-2}})$.
\end{corollary*}
 
\begin{proof}
The proof follows from a simple substitution.
\end{proof}

% \begin{corollary*}[\ref{thm:sasgdwithN}]
%     With the same assumption as Theorem~\ref{thm:sasgdwithN}, with proper scheduling, we have
%     \begin{align}
%         \min_{t\in[T]}\E[\|\nabla F(x_t)\|^2]\leq O((N+\frac{1}{N})\cdot T^\frac{1-\alpha}{\alpha}+N^2\cdot T^{\frac{2}{\alpha}-3} ).
%     \end{align}
    
% \end{corollary*}
% \begin{proof}
%     We notice that by scheduling $u=\Theta(T^\zeta), \eta=\Theta(T^\omega),\lle=\Theta(T^\nu)$, we obtain
%     \begin{align}
%         &\min_{t\in[T]}\E[\|\nabla F(x_t)\|^2] \\\leq& O(N\cdot T^{-1-\omega-\nu}+T^{(1-\alpha)\zeta}+T^\nu+N^2\cdot T^{2\nu+\omega+(2-\al)\zeta-1}+T^{2\nu+\omega+(2-\al)\zeta}+\frac{1}{N}\cdot T^{\nu+\omega+(2-\al)\zeta})
%     \end{align}

% Now, if we assign
% \begin{align}
%     \omega=\frac{\alpha-2}{\alpha},\:\nu=\frac{1-\alpha}{\alpha},\: \zeta=\frac{1}{\alpha},
% \end{align}
% the RHS bound becomes exactly the bound in the statement.
% \end{proof}

\subsubsection{$Clip^2$ with Staleness-Aware Downplaying}
In this section, the proofs are highly similar to those for $SGDClip$ with staleness-aware downplaying; we therefore only give proof sketches for some of them and leave the details to the reader.

\begin{theorem}\label{thm:sd-clip2}
With the same assumption as in Theorem~\ref{thm:sasgd}, $\tilde{u}$ being the server-side upper-clipping threshold, if we let
$\eta_\ell K\le\frac{\kappa}{46L}$ and $u\ge D(\frac{1024}{\kappa^2})^{1/\alpha}$.
If we set
\begin{equation}\label{eq:sd2-stepsize}
  \eta\eta_\ell=\min\bigg\{\sqrt{\frac{\Delta}{LK^2D^\alpha u^{2-\alpha}P_2}},\
  \frac{\kappa}{536\,LK\sqrt{\max\{1,\Lambda\}}}\bigg\},
\end{equation}

 then Clip$^2$ with staleness-aware downplaying with
$w_t=1/p_t$ and a server-side threshold satisfying Eq.~\ref{eq:server-thr}.
gives us that
\begin{equation}\label{eq:sd2-bound}
\begin{aligned}
  \min_{t\in[T]}\E\big[\Phi(\|\nabla F(x_{t-1})\|)\big]
  \le{}&\frac{512\sqrt{L\Delta D^\alpha u^{2-\alpha}P_2}}{\kappa P_1}
  +\frac{137216\,L\Delta\sqrt{\max\{1,\Lambda\}}}{\kappa^2P_1}\\
  &+\frac{256\,(C_1+512)}{\kappa^4}D^{2\alpha}u^{2-2\alpha}
  +\frac{34832}{\kappa^2}L^2K^2\eta_\ell^2u^2
  +\frac{69632\,L\Delta\Lambda}{\kappa P_2}.
\end{aligned}
\end{equation}
Moreover, if the delays satisfy
\begin{equation}\label{eq:sd2-delay-cond}
  P_2\ge\frac{287296\,L\Delta\max\{1,\Lambda\}}{\kappa^2D^\alpha u^{2-\alpha}},
\end{equation}
then the minimum in Eq.~\ref{eq:sd2-stepsize} will be attained by its first term, and
the second term in Eq.~\ref{eq:sd2-bound} can be dropped.
\end{theorem}

\begin{proof}[Proof sketch]
Since
$Clip(\tilde u,w_t\overline{g}_t)=w_tClip(\tilde u/w_t,\overline{g}_t)$, the update
is effectively Clip$^2$ with server-side learning rate
$\eta_t=\eta w_t$ and server-side threshold $\tilde u/w_t\ge\tilde u$. Therefore, Eq.~\ref{eq:server-thr} still holds true. As a result, we similarly have that
Eq.~\ref{eq:one-step-clip2} holds with $\gamma$ replaced by
$\gamma_t=\gamma w_t$, as long as $\gamma K\le\kappa/(512L)$. Combining it with the
weighted distance bound in the proof of Theorem~\ref{thm:sd-sgdclip}, and following the
same telescoping sum argument, we obtain
\begin{equation}\label{eq:sd2-general}
  \min_{t\in[T]}\Phi_t
  \le\frac{256\,\Delta}{\kappa K\gamma P_1}+\frac{256\,(C_1+512)}{\kappa^4}B^2
  +\frac{34832}{\kappa^2}L^2K^2\eta_\ell^2u^2
  +\frac{256}{\kappa}LK\gamma V\Big(\frac{P_2}{P_1}+272\,LK\gamma\Lambda\Big),
\end{equation}
whenever $1088\,L^2K^2\gamma^2\Lambda\le\kappa/256$, where $P_1$, $P_2$, and
$\Lambda$ are defined with the weights $w_t$ as in the proof of
Theorem~\ref{thm:sd-N}. The remaining checking of the conditions and bounding of the terms are highly identical to the proof of Theorem~\ref{thm:sd-sgdclip}.
\end{proof}

The theorem above and the corollary below together give Theorem~\ref{thm:saclip2}.

\begin{corollary}\label{cor:sd-clip2}
If we let $\eta_\ell=\Theta(T^\nu)$,$u=\Theta(T^\zeta)$, and let $\tilde u=\kappa K\eta_\ell u$, and
choose $\eta\eta_\ell$ as in Eq.~\ref{eq:sd2-stepsize}. Then Clip$^2$ with
staleness-aware downplaying satisfies
\[
  \min_{t\in[T]}\E\big[\Phi(\|\nabla F(x_{t-1})\|)\big]
  \le O\bigg(T^{\frac{2-\alpha}{2}\zeta}\,\frac{\sqrt{P_2}}{P_1}
  +\frac{\sqrt{\max\{1,\Lambda\}}}{P_1}
  +T^{2(1-\alpha)\zeta}+T^{2\nu+2\zeta}+\frac{\Lambda}{P_2}\bigg).
\]
In particular, if all $p_t$'s take the same value
$p\le O(T^{\frac{\alpha}{2(3\alpha-2)}})$, then setting $\zeta=\frac1{3\alpha-2}$
and $\nu=-\frac{\alpha}{3\alpha-2}$ gives us that
$\min_{t\in[T]}\E\|\nabla F(x_{t-1})\|\le O(T^{-\frac{\alpha-1}{3\alpha-2}})$.
\end{corollary}

\begin{proposition*}[\ref{prop:saclip2}]
% Suppose that among all $p_t$'s for $1\le t\le T$, there is a fraction
% $q\in(0,1]$ of delays that satisfy $p_t\le A$, where $A\le O(T^c)$. Let
% $u=\Theta(T^\zeta)$ and $\eta_\ell=\Theta(T^\nu)$ with $\zeta>0$ and $\nu<-\zeta$,
% and choose $\eta\eta_\ell$ as in Eq.~\ref{eq:sd2-stepsize}. Then when
% $c<\min\{\frac12,\ 1-(2-\alpha)\zeta\}$, Clip$^2$ with staleness-aware
% downplaying satisfies $\min_{t\in[T]}\E\|\nabla F(x_{t-1})\|\to0$ as
% $T\to\infty$. In particular, this holds if the minimum delay $a\le O(T^c)$
% takes up a fixed fraction $q$ of all delays.

Suppose that among all $p_t$'s for $1\le t\le T$, there is a fraction $q\in(0,1]$
of delays that satisfy $p_t\le A$, and let $T_{\rm eff}:=qT/A$. If we make
$u=\Theta(T_{\rm eff}^{\frac1{3\alpha-2}})$,
$\eta_\ell=\Theta(T_{\rm eff}^{-\frac{\alpha}{3\alpha-2}})$, and
$\tilde u=\kappa K\eta_\ell u$ for a constant $\kappa\in(0,1]$, and choose
$\eta\eta_\ell$ as in Eq.~\ref{eq:sd2-stepsize}. Then for sufficiently large
$T_{\rm eff}$, Clip$^2$ with staleness-aware downplaying satisfies
\[
  \min_{t\in[T]}\E\|\nabla F(x_{t-1})\|
  \le O\bigg(T_{\rm eff}^{\frac{1-\alpha}{3\alpha-2}}+\sqrt{\frac{NA}{T_{\rm eff}}}\bigg).
\]
In particular, if $A=O(1)$, and $N\le O(T^{\frac{\alpha}{3\alpha-2}})$,
then $\min_t\E\|\nabla F(x_{t-1})\|\le O(T^{\frac{1-\alpha}{3\alpha-2}})$, regardless of
the remaining delays. More generally, if $A\le O(T^c)$ with
$c<\frac12$, then
$\min_t\E\|\nabla F(x_{t-1})\|\le O\big(T^{-(1-c)\frac{\alpha-1}{3\alpha-2}}+\sqrt N\,T^{c-\frac12}\big)$.

\end{proposition*}

Similarly, the assumptions on the delays can be removed if we adapt the
downplaying to $N$.
 
\begin{theorem}\label{thm:sd-clip2-N}
With the same assumption as Theorem~\ref{thm:sasgd}, if we
change the downplaying in Eq.~\ref{eq:server-sd} to
$\eta_t=\eta\,w_t$ with $w_t:=\min\{1/N,1/p_t\}$, let $\eta_\ell K\le\frac{\kappa}{46L}$ and
$u\ge D(\frac{1024}{\kappa^2})^{1/\alpha}$, and set
\begin{equation}\label{eq:sd2-N-stepsize}
  \eta\eta_\ell=\min\bigg\{N\sqrt{\frac{\Delta}{LK^2VT}},\ \frac{\kappa}{536\,LK}\bigg\},
\end{equation}
then we have
\begin{equation}\label{eq:sd2-N-bound}
\begin{aligned}
  \min_{t\in[T]}\E\big[\Phi(\|\nabla F(x_{t-1})\|)\big]
  \le{}&\frac{1280}{\kappa}\sqrt{\frac{L\Delta D^\alpha u^{2-\alpha}}{T}}
  +\frac{548864\,L\Delta N}{\kappa^2T}+\frac{69632\,L\Delta N^2}{\kappa T}\\
  &+\frac{256\,(C_1+512)}{\kappa^4}D^{2\alpha}u^{2-2\alpha}
  +\frac{34832}{\kappa^2}L^2K^2\eta_\ell^2u^2 .
\end{aligned}
\end{equation}
\end{theorem}

\begin{corollary*}[\ref{thm:clip2withN}]
Under the same assumptions as Theorem~\ref{thm:sd-clip2-N}, let
$u=\Theta(T^\zeta)$ and $\eta_\ell=\Theta(T^\nu)$ with $\zeta=\frac1{3\alpha-2}$
and $\nu=-\frac{\alpha}{3\alpha-2}$. Then
\[
  \min_{t\in[T]}\E\big[\Phi(\|\nabla F(x_{t-1})\|)\big]
  \le O\Big(T^{-\frac{2(\alpha-1)}{3\alpha-2}}+\frac{N^2}{T}\Big).
\]
In particular, if
$N\le O(T^{\frac{\alpha}{2(3\alpha-2)}})$, then
$\min_{t\in[T]}\E\|\nabla F(x_{t-1})\|\le O(T^{-\frac{\alpha-1}{3\alpha-2}})$.

\end{corollary*}

For Proposition~\ref{prop:saclip2} and Theorem~\ref{thm:clip2withN}, the proofs are highly similar to those for $SGDClip$ with staleness-aware downplaying; we therefore leave the details to the reader.

\subsection{Delay Compensation} \label{app:sub:dc}

% \tian{$w_{loc}$ is not defined; $(Clip(u_t, g_{k}^t))^{\odot 2}$ denotes $Clip(u_t, g_{k}^t) \odot Clip(u_t, g_{k}^t)$? need to clarify}
Here, again, for notation simplicity, we assume the client buffer size $M=1$. In other words, the server updates the global model whenever it receives a new model update from any single client. Our proofs extend naturally to $M>1$. 

\textbf{Notations.} We use $C_{t,k}:=\mathrm{Diag}\big(\lambda\,\E[\psi_{t,k}\odot\psi_{t,k}\mid G_{t,k}]\big)$ to denote the conditional curvature estimate, and let $\rho:=\lambda u^2$. Note that since
$|\psi_{t,k,j}|\le u$, we have $0\preceq C_{t,k}\preceq\rho I$. Moreover, we define the curvature estimate accuracy
\[
  \varepsilon_H:=\operatorname*{ess\,sup}_{t,k}\big\|\nabla^2F(x^{\tau_t}_{k-1})-C_{t,k}\big\|
  \le L+\rho .
\] Therefore, by definition, we have $\|C_{t,k}\|\le L+\varepsilon_H$. Additionally, for the local step $k$ of round $t$, we denote the
conditional mean of the compensated direction and the compensation residual by
\[
  m_{t,k}:=\E\big[\psi_{t,k}+\lambda\,\psi_{t,k}\odot\psi_{t,k}\odot\delta_t\mid G_{t,k}\big]
  =\mu_{t,k}+C_{t,k}\delta_t,
  \qquad
  e_{t,k}:=a_t-b_{t,k}-C_{t,k}\delta_t.
\]

Recall that, similarly to \citep{zheng2017asynchronous}, we hereby assume that the Hessian of $F$ is
$L'$-Lipschitz, i.e., $\|\nabla^2F(x)-\nabla^2F(y)\|\le L'\|x-y\|$.
% Similar as main text, here we use $(\cdot)^{\odot 2}$ to denote  $(\cdot) \odot  (\cdot)$, where $\odot$ is element-wise multiplication. For instance, $(Clip(u_t, g_{k}^t))^{\odot 2}$ represents $Clip(u_t, g_{k}^t) \odot Clip(u_t, g_{k}^t)$.
% With a slight switch of notations, we use $w$ instead of $x$ as the model parameter. $w_k^t$ denotes the local model at the $k$-th local step, starting from a global model $w_{t}$. The corresponding stochastic gradient evaluated at $w_k^t$ is $g_k^t$.
% \tian{so here $w_{loc}$ denotes any of the local optimal model? and assume $w_{t-1} \in B_r(w_{loc})$ means that we assume for any $t$, the global model is in some $r$-radius ball of any, or one of the $w_{loc}$?} 

\begin{lemma} \label{lem:dc-align}
Suppose we have that $L\,\mathrm{dist}_{t,k}\le u/16$ and $(L+\varepsilon_H)\delta_t\le u/16$.
Then
\[
  \langle a_t,m_{t,k}\rangle\ge\frac1{18}\Phi (\|a_t\|)-22c_\alpha^2B^2-\|e_{t,k}\|^2 .
\]
\end{lemma}
\begin{proof}
For simplicity, we here write $a=a_t$, $b=b_{t,k}$, $\mu=\mu_{t,k}$, $C=C_{t,k}$, and $e=e_{t,k}$.

Again, we will approach this proof by analyzing two different cases with small or large gradients. Firstly, for the small gradients case, we have $\|b\|\le u/2$. Then by Lemma~\ref{lem:clip}, we have
$\|b-\mu\|\le2^\alpha B$. By the definition of $e$, we have
$m_{t,k}=\mu+C\delta_t=a-(b-\mu)-e$. Therefore, by Young's inequality, we have
\[
  \langle a,m_{t,k}\rangle=\|a\|^2-\langle a,(b-\mu)+e\rangle
  \ge\frac12\|a\|^2-\|(b-\mu)\|^2-\|e\|^2
  \ge\frac12\Phi(\|a\|)-4^\alpha B^2-\|e\|^2 .
\]
This is the case where the compensation is effective since the staleness only
enters through the residual $e$.
 
Now, for the large gradients case in which  $\|b\|>u/2$, we decompose 
$\langle a,m_{t,k}\rangle=\langle b,\mu\rangle+\langle a-b,\mu\rangle+\langle a,C\delta_t\rangle$ and bound each term separately. We recall that the proof of Lemma~\ref{lem:clip} for the large-gradient case gives
$\langle b,\mu\rangle\ge\frac u4\|b\|-\frac{c_\alpha Bu}2$.

For the second term, since $\|\mu\|\le u$ and
$u<2\|b\|$, we have
$\langle a-b,\mu\rangle\ge-L\,\mathrm{dist}_{t,k}\,u\ge-\frac{u^2}{16}\ge-\frac u8\|b\|$ where we used the assumption of the Lemma for the second inequality.

For the third term, notice that we have $\|a\|\le\|b\|+L\,\mathrm{dist}_{t,k}\le\frac98\|b\|$, so
$\langle a,C\delta_t\rangle\ge-\|a\|(L+\varepsilon_H)\delta_t\ge-\frac{9}{128}u\|b\|$.

Now we combine all three terms and get $\langle a,m_{t,k}\rangle\ge\frac7{128}u\|b\|-\frac{c_\alpha Bu}2$.
If $c_\alpha B\le\frac3{64}\|b\|$, then $\frac{c_\alpha Bu}2\le\frac3{128}u\|b\|$.
Otherwise, $u<2\|b\|<\frac{128}3c_\alpha B$, so $\frac{c_\alpha Bu}2\le22c_\alpha^2B^2$.
In either case, we obtain $\langle a,m_{t,k}\rangle\ge\frac1{32}u\|b\|-22c_\alpha^2B^2$.
Finally, we notice that $\Phi(\|a\|)\le\frac u2\|a\|\le\frac9{16}u\|b\|$, which gives
$\langle a,m_{t,k}\rangle\ge\frac1{18}\Phi(\|a\|)-22c_\alpha^2B^2$. This concludes the proof.
\end{proof}

\begin{lemma}\label{lem:dc-residual}
Suppose $\delta_t\le2\tau\gamma Ku$. Then we have
\[
  \|e_{t,k}\|^2\le3L^2K^2\eta_\ell^2u^2+6L'^2K^4\eta_\ell^4u^4+\varepsilon_{\rm eff}^2\,\delta_t^2,
  \qquad
  \varepsilon_{\rm eff}^2:=\min\Big\{2(2L+\varepsilon_H)^2,\
  3\varepsilon_H^2+24L'^2\tau^2\gamma^2K^2u^2\Big\}.
\]
\end{lemma}
 
\begin{proof}
Let $\ell_{t,k}:=x^{\tau_t}_{k-1}-x_{\tau_t}$, then we have
$\|\ell_{t,k}\|\le K\eta_\ell u$ and
$x_{t-1}-x^{\tau_t}_{k-1}=\delta_t-\ell_{t,k}$. We can now bound $e_{t,k}$ in two ways.
 
The first way is to bound without using the compensation. By $L$-smoothness and
$\|C_{t,k}\|\le L+\varepsilon_H$, we have
$\|e_{t,k}\|\le L\,\mathrm{dist}_{t,k}+(L+\varepsilon_H)\delta_t
\le LK\eta_\ell u+(2L+\varepsilon_H)\delta_t$, so
$\|e_{t,k}\|^2\le2L^2K^2\eta_\ell^2u^2+2(2L+\varepsilon_H)^2\delta_t^2$.
 
The second way is to incorporate compensation. With $H:=\nabla^2F(x^{\tau_t}_{k-1})$, we can apply Taylor's
theorem along with the Lipschitz Hessian condition to derive
$a_t-b_{t,k}=H(\delta_t-\ell_{t,k})+r_{t,k}$ with
$\|r_{t,k}\|\le\frac{L'}2\,\mathrm{dist}_{t,k}^2$. Therefore, we have
\[
  e_{t,k}=\underbrace{(H-C_{t,k})\delta_t}_{T_{1,1}}
  \underbrace{-H\ell_{t,k}+r_{t,k}}_{T_{1,2}}.
\]
We can see that $T_{1,1}$ is the error of the curvature estimate and $T_{1,2}$ consists of
the local drift and the Taylor remainder. Then by the definition of $\varepsilon_H$, we have
$\|T_{1,1}\|\le\varepsilon_H\delta_t$. And we also have  $\|T_{1,2}\|\le LK\eta_\ell u+\frac{L'}2(K\eta_\ell u+\delta_t)^2$.
Using $(x+y)^4\le8x^4+8y^4$ and $\delta_t^4\le4\tau^2\gamma^2K^2u^2\delta_t^2$, we have
\[
  \|e_{t,k}\|^2\le3\varepsilon_H^2\delta_t^2+3L^2K^2\eta_\ell^2u^2
  +6L'^2K^4\eta_\ell^4u^4+24L'^2\tau^2\gamma^2K^2u^2\delta_t^2 .
\]
The lemma is simply a combination of the two bounds we derived.
\end{proof}

With these Lemmas, we will prove the following theorem. And Theorem~\ref{thm:dc} follows directly.

\begin{theorem}\label{thm:dc_pre}
Suppose $F$ is $L$-smooth, the server-side threshold satisfies
assumption Eq.~\ref{eq:server-thr}, and the Hessian of $F$ is $L'$-Lipschitz. If we let the learning rates
and the client-side threshold satisfy
\begin{equation}\label{eq:dc-step-cond}
  \eta\eta_\ell K\le\frac{\kappa}{864\,L},\qquad
  \eta_\ell K\le\frac{\kappa}{42\,L},\qquad
  u\ge D\Big(\frac{3456}{\kappa^2}\Big)^{1/\alpha},
\end{equation}
and if the maximum delay satisfies
\begin{equation}\label{eq:dc-delay-cond}
  \tau\le\frac{\kappa}{416\,\eta\eta_\ell K\,(L+\varepsilon_H+\rho)},
\end{equation}
then Clip$^2$ with delay compensation 
\begin{equation}\label{eq:dc-bound}
\begin{aligned}
  \frac1T\sum_{t=1}^T\E\big[\Phi(\|\nabla F(x_{t-1})\|)\big]
  \le{}&\frac{216\,\Delta}{\kappa T\eta\eta_\ell K}
  +\frac{216\,C_3}{\kappa^4}D^{2\alpha}u^{2-2\alpha}
  +\frac{1296}{\kappa^2}L^2K^2\eta_\ell^2u^2
  +\frac{1296}{\kappa}L'^2K^4\eta_\ell^4u^4\\
  &+\frac{432}{\kappa}LK\eta\eta_\ell D^\alpha u^{2-\alpha}
  +\frac{1728}{\kappa}\,S_{\rm DC}\,\tau^2\eta^2\eta_\ell^2K^2D^\alpha u^{2-\alpha},
\end{aligned}
\end{equation}
where $C_3:=22c_\alpha^2+6912$ and
$S_{\rm DC}:=\varepsilon_{\rm eff}^2+LK\eta\eta_\ell(16L^2+\rho^2)$ with
$\varepsilon_{\rm eff}$ defined in Lemma~\ref{lem:dc-residual}.
\end{theorem}

\begin{proof}
Before proceeding with the proof, we need to first make sure that the conditions of the lemmas above hold. Since
$\mathrm{dist}_{t,k}\le K\eta_\ell u+2\tau\gamma Ku$, we notice that
Eq.~\ref{eq:dc-step-cond} and Eq.~\ref{eq:dc-delay-cond} directly give
$L\,\mathrm{dist}_{t,k}\le u/16$ and $(L+\varepsilon_H)\delta_t\le u/16$. Therefore, the conditions of the lemmas are indeed satisfied.
 
We define $\sigma_t:=\min\{1,\tilde u/\|d_t\|\}$. We notice that  Eq.~\ref{eq:server-thr} and $\|d_t\|\le\frac32K\eta_\ell u$ give $\frac23\kappa\le\sigma_t\le1$. The update rule then is
$x_t-x_{t-1}=-\eta\sigma_td_t$. By $L$-smoothness and $\sigma_t\le1$, we have
\begin{equation}\label{eq:dc-descent}
  \E F(x_t)\le\E F(x_{t-1})-\eta\,\E\big[\sigma_t\langle a_t,d_t\rangle\big]
  +\frac{L\eta^2}2\E\|d_t\|^2 .
\end{equation}
Similarly to the proof of Theorem~\ref{thm:clip2}, let
$\tilde d_t:=\E[d_t\mid G^0_t]$ and
$\tilde\sigma_t:=\min\{1,\tilde u/\|\tilde d_t\|\}$, and decompose
\[
  \sigma_t\langle a_t,d_t\rangle
  =\underbrace{\sigma_t\langle a_t,\tilde d_t\rangle}_{B_1}
  +\underbrace{\tilde\sigma_t\langle a_t,d_t-\tilde d_t\rangle}_{B_2}
  +\underbrace{(\sigma_t-\tilde\sigma_t)\langle a_t,d_t-\tilde d_t\rangle}_{B_3}.
\]

We now bound these three terms separately.

Firstly, we bound $B_1$. By the tower property,
$\langle a_t,\tilde d_t\rangle=\eta_\ell\sum_{k=1}^K\E[\langle a_t,m_{t,k}\rangle\mid G^0_t]$.
By Lemma~\ref{lem:dc-align}, we have
$\langle a_t,\tilde d_t\rangle\ge K\eta_\ell\big(\frac1{18}\Phi(\|a_t\|)-22c_\alpha^2B^2-\bar e_t\big)$
with $\bar e_t:=\frac1K\sum_k\E[\|e_{t,k}\|^2\mid G^0_t]\ge0$. Since
$\sigma_t\in[\frac23\kappa,1]$, we have
$B_1\ge K\eta_\ell\big(\frac\kappa{27}\Phi (\|a_t\|)-22c_\alpha^2B^2-\bar e_t\big)$.
 
For $B_2$, as in the proof of Theorem~\ref{thm:clip2}, we simply notice that
$\E[B_2]=0$.
 
We finally move to $B_3$.  As in the proof of Theorem~\ref{thm:clip2},
$\E[B_3\mid G^0_t]\ge-\frac{\|a_t\|}{\tilde u}\E[\|d_t-\tilde d_t\|^2\mid G^0_t]$.
To bound the conditional variance, we compare $d_t$ with  $K\eta_\ell(c_t+\lambda\,c_t\odot c_t\odot\delta_t)$,
where $c_t:= Clip(u,\nabla F(x_{\tau_t}))$. Since all entries of $\psi_{t,k}$ and
$c_t$ are bounded by $u$ in absolute value, we have
$\|\lambda(\psi_{t,k}\odot\psi_{t,k}-c_t\odot c_t)\odot\delta_t\|
\le2\lambda u\delta_t\|\psi_{t,k}-c_t\|$. Now, let $\rho:=\lambda u^2$, we notice that if $\rho\delta_s\le u/2$, then this is in turn at most $\|\psi_{t,k}-c_t\|$. We will prove this inductively: If $\rho\delta_s\le u/2$, then
$\|d_s\|\le K\eta_\ell u+K\eta_\ell\rho\delta_s\le\frac32K\eta_\ell u$. By
induction over $t$, this gives $\delta_t\le\frac32\tau\gamma Ku\le2\tau\gamma Ku$,
and Eq.~\ref{eq:dc-delay-cond} indeed gives $\rho\cdot2\tau\gamma Ku\le u/2$. This concludes the inductive proof.

Therefore, the same argument as in the proof of
Theorem~\ref{thm:clip2} gives
$\E[\|d_t-\tilde d_t\|^2\mid G^0_t]\le4K^2\eta_\ell^2Z$ with
$Z=2L^2K^2\eta_\ell^2u^2+4V$. Then the same
two-case argument (now with $Z\le\kappa^2u^2/432$, which follows from
Eq.~\ref{eq:dc-step-cond}) gives
\[
  \E[B_3\mid G^0_t]\ge-K\eta_\ell\Big(\frac\kappa{54}\Phi(\|a_t\|)
  +\frac{6912}{\kappa^3}B^2+\frac1\kappa L^2K^2\eta_\ell^2u^2\Big).
\]
 
Finally, we need to bound $\E\|d_t\|^2$. Since
$\|d_t\|^2\le2\|\overline{g}_t\|^2+2K^2\eta_\ell^2\rho^2\delta_t^2$,
Eq.~\ref{eq:g-sq} and $\E[\mathrm{dist}_{t,k}^2]\le2K^2\eta_\ell^2u^2+2q_t$ with
$q_t:=\E[\delta_t^2]$ give
\[
  \frac{L\eta^2}2\E\|d_t\|^2\le L\gamma^2K^2\big(8\Phi_t+16L^2K^2\eta_\ell^2u^2
  +(16L^2+\rho^2)q_t+2V\big).
\]
 
We then combine all the terms. By Lemma~\ref{lem:dc-residual}, we have
$\E[\bar e_t]\le3L^2K^2\eta_\ell^2u^2+6L'^2K^4\eta_\ell^4u^4+\varepsilon_{\rm eff}^2q_t$.
We can now substitute all the bounds into Eq.~\ref{eq:dc-descent}, and use
$8L\gamma K\le\kappa/108$ from Eq.~\ref{eq:dc-step-cond}. As a result, we have
\begin{equation}\label{eq:dc-one-step}
\begin{aligned}
  \E F(x_t)\le{}&\E F(x_{t-1})-\frac{\kappa\gamma K}{108}\Phi_t
  +\gamma K\Big(\frac{C_3}{\kappa^3}B^2+\frac5\kappa L^2K^2\eta_\ell^2u^2
  +6L'^2K^4\eta_\ell^4u^4\Big)\\
  &+\gamma K\,S_{\rm DC}\,q_t+2L\gamma^2K^2V .
\end{aligned}
\end{equation}
% Compared with Eq.~\ref{eq:one-step-clip2}, the staleness enters through
% $S_{\rm DC}\,q_t$ instead of $34L^2R_t$. This is where the delay compensation
% helps: $S_{\rm DC}$ is of the order of $\varepsilon_H^2$ instead of $L^2$ when
% the curvature estimate is accurate and the delays are moderate.
 
To bound the $q_t's$, we use a similar proof to
Theorem~\ref{thm:sgdclip}: Cauchy-Schwarz and the bound on
$\E\|d_s\|^2$ above give
\[
  q_t\le\tau\eta^2\sum_{s\in W_t}\E\|d_s\|^2
  \le\tau\gamma^2K^2\sum_{s\in W_t}\big(16\Phi_s+32L^2K^2\eta_\ell^2u^2
  +(32L^2+2\rho^2)q_s+4V\big).
\]
Summing over $t$ and observing that each $s$ belongs to at most $\tau$ of the sets
$W_t$, together with $\tau^2\gamma^2K^2(32L^2+2\rho^2)\le1/2$ from
Eq.~\ref{eq:dc-delay-cond}, we obtain
\[
  \sum_{t=1}^Tq_t\le32\tau^2\gamma^2K^2\sum_{t=1}^T\Phi_t
  +T\tau^2\gamma^2K^2\big(64L^2K^2\eta_\ell^2u^2+8V\big).
\]
 
We finally telescope together
with the bound on $\sum_tq_t$. We also notice that Eq.~\ref{eq:dc-step-cond} and
Eq.~\ref{eq:dc-delay-cond} give $S_{\rm DC}\le24(L+\varepsilon_H+\rho)^2$ and hence
$32\,S_{\rm DC}\tau^2\gamma^2K^2\le\kappa/216$. Therefore, the $\Phi_t$ terms on
the right-hand side are at most $\frac{\kappa\gamma K}{216}\sum_t\Phi_t$, and the
term $S_{\rm DC}\tau^2\gamma^2K^2\cdot64L^2K^2\eta_\ell^2u^2$ is at most
$\frac1\kappa L^2K^2\eta_\ell^2u^2$. As a result, we have
\[
  \frac{\kappa\gamma K}{216}\sum_{t=1}^T\Phi_t
  \le\Delta+T\gamma K\Big(\frac{C_3}{\kappa^3}B^2+\frac6\kappa L^2K^2\eta_\ell^2u^2
  +6L'^2K^4\eta_\ell^4u^4+8\,S_{\rm DC}\tau^2\gamma^2K^2V\Big)+2TL\gamma^2K^2V .
\]
Dividing both sides by $\kappa T\gamma K/216$ proves Eq.~\ref{eq:dc-bound}.
\end{proof}

\begin{corollary}\label{cor:dc}
Under the same assumptions as Theorem~\ref{thm:dc}, let $\eta$, $\eta_\ell$, and
$u$ be the same as in $SGDClip$, let
$\tilde u=\kappa K\eta_\ell u$, and suppose that $\varepsilon_H$, $L'$, and
$\rho=\lambda u^2$ are bounded by constants. Then if we have
$\tau\le O\big(T^{\frac{\alpha}{2(3\alpha-2)}}\big)$, the convergence rate is $O\big(T^{-\frac{\alpha-1}{3\alpha-2}}\big)$.
\end{corollary}

\section{Experiment Details} \label{app:exp_details}

In this section, we will explain the detailed setup of our experiments across three different tasks: image classification task with ViT on CIFAR-10, natural language processing task with pre-trained RoBERTa-base model on GLUE, and the machine translation task with T5 Generative model on TED.
Throughout our experiments, we will use simulated runtime for different clients. Before each global training round, each client samples a runtime from a fixed uniform distribution out of three types: small (runtime:1-2), medium (runtime: 3-5), and large. The large runtime distribution depends on the straggler mode. If the straggler mode is 'large', then the large runtime distribution corresponds to a runtime in the range 20-40, which corresponds to scenarios where there are large stragglers. Otherwise, the straggler mode is 'mild', and the large runtime distribution corresponds to a runtime in the range 5-8. All the experiments are conducted on 5 GPU servers with L40s machines. Each individual job was allocated 1 GPU, 10 CPU cores, and 48 GB of system RAM.

\subsection{Image Classification Task with CIFAR-10} 

We evaluate our method using the Vision Transformer (ViT), a model for image recognition developed by Google Research. ViT directly applies the standard Transformer architecture. It processes an image by splitting it into a sequence of fixed-size patches, making it a highly effective model for classification and an excellent baseline for evaluating our methods.

To evaluate image classification performance, we utilize the CIFAR-10 dataset, a widely recognized benchmark for computer vision research. Specifically, we fine-tune a pre-trained ViT model on the CIFAR-10 training set. The dataset consists of 60,000 32x32 color images distributed across 10 classes, including airplane, automobile, bird, and cat. These images present a diverse set of objects and backgrounds, providing a robust evaluation of a model's ability to learn core visual features and generalize effectively.

The experiments are conducted with $N=40$ clients, in which $17$ of them sample runtime from the small runtime distribution, $12$ of them sample runtime from the medium runtime distribution, and the remaining $11$ of them sample runtime from the large runtime distribution. We use $K=5$ client-side epochs, and $T=140$  global epochs. Each run takes about 10 hours on our machines. The range of $M$ that we explore is $\{1, 4, 10\}$. The $Clip$ that we use is coordinate-wise upper-clipping. The server/client-side learning rates, server/client-side clipping thresholds are determined by the hyperparameter sweep laid out in Appendix~\ref{app:sub:hyp_sweep} for different server-side and client-side optimizers. 

\subsection{Natural Language Processing Task with GLUE}
We evaluate the proposed method using the RoBERTa model, an encoder-only architecture derived from BERT. To rigorously assess natural language understanding capabilities, we employ the General Language Understanding Evaluation (GLUE) benchmark, a widely adopted suite of datasets for the training, evaluation, and analysis of natural language processing systems. GLUE comprises a diverse range of tasks, including sentiment classification, semantic textual similarity, textual entailment, and natural language inference, thereby offering a comprehensive assessment of model performance across multiple linguistic dimensions. For all experiments, we adhere to standard RoBERTa fine-tuning protocols, initializing with pretrained weights and optimizing task-specific objectives for each dataset within the benchmark.

The experiments are conducted with $N=10$, $M=4$, $K=1$ client-side epochs, and a maximum of $T=30$ global epochs. Among all clients,  $3$ of them sample runtime from the small runtime distribution, $3$ of them sample runtime from the medium runtime distribution, and the remaining $4$ of them sample runtime from the large runtime distribution.. Each run takes about 3.5 hours. The clipping operation ($Clip$) employed is coordinate-wise upper-clipping. The server- and client-side learning rates, along with the corresponding clipping thresholds, are specified in Appendix~\ref{app:sub:hyp_sweep} for different server-side and client-side optimizers. To demonstrate the applicability of our methods to few-shot/many-shot settings, we follow prior work \citep{lmbff,mvp,mezo} and sub-sample the training samples for tasks QQP, MNLI, QNLI, and SST-2 to 5,000 samples each. For STS-B, we report the Pearson correlation coefficient scaled to the range [0,1], while for the remaining tasks we report accuracy.

\subsection{Machine Translation Task with TED}
We evaluate our methods using the T5 model, a model for text-to-text transformer model developed by Google Research. T5 adopts the encoder-decoder Transformer structure, where both the inputs and the outputs are texts. This structure makes it highly efficient for sequence-to-sequence generation, which extends to a wide range of natural language tasks such as translations, summarizations, etc. Therefore, this model serves as an excellent baseline for evaluating our methods.

To evaluate machine translation performance, we utilize the WMT TED talk dataset, which is sourced from IWSLT  2016. This dataset is a collection of the original TED talks and their multilingual translated versions. These TED talks present a diverse set of technical terminologies, speaking styles, and colloquialisms, providing a robust evaluation of a model’s ability to learn complex cross-lingual correspondences under challenging conversational setups.

The experiments are conducted with $N=10$ clients, in which $3$ of them sample runtime from the small runtime distribution, $3$ of them sample runtime from the medium runtime distribution, and the remaining $4$ of them sample runtime from the large runtime distribution. We use $K=5$ client-side epochs, and $T=25$  global epochs. Each run takes about 4 hours on our machines. $M$ is set to be $4$. The $Clip$ that we use is coordinate-wise upper-clipping. The server/client-side learning rates, server/client-side clipping thresholds are determined by the hyperparameter sweep laid out in Appendix~\ref{app:sub:hyp_sweep} for different server-side and client-side optimizers. 

\subsection{Hyperparameter Sweep Grid and Optimal Hyperparameter} \label{app:sub:hyp_sweep}
For each experiment, we use a hyperparameter sweep grid to find the hyperparameter choice that gives the best result. We first present the hyperparameter sweep grid that we used for different algorithms and settings in Table~\ref {tab:hyp_sweep}.

\begin{table}[h!]
\vspace{\baselineskip}
\caption{\small Hyperparameter sweep grid.}
\label{tab:hyp_sweep}
\resizebox{\textwidth}{!}{
\small\begin{tabular}{lcccc}
\toprule
Algorithm & Server-side LR  & Client-side LR  & Server-side U & Client-side U  \\ 
\midrule

SGDClip & (0.1, 0.01, 0.001, 0.0001) & (0.1, 0.01, 0.001, 0.0001)  & -  & \text{np.linspace($10^{-4}$, 1.5, 4)}  \\
$Clip^2$ &(0.1, 0.01, 0.001, 0.0001) & (0.1, 0.01, 0.001, 0.0001)  & \text{np.linspace($10^{-4}$, 1.5, 4)}  & \text{np.linspace($10^{-4}$, 1.5, 4)}  \\
SD-SGDClip & (0.1, 0.01, 0.001, 0.0001) & (0.1, 0.01, 0.001, 0.0001)  & -  & \text{np.linspace($10^{-4}$, 1.5, 4)}  \\
SD-$Clip^2$ &(0.1, 0.01, 0.001, 0.0001) & (0.1, 0.01, 0.001, 0.0001)  & \text{np.linspace($10^{-4}$, 1.5, 4)}  & \text{np.linspace($10^{-4}$, 1.5, 4)}  \\
DC-SGDClip & (0.1, 0.01, 0.001, 0.0001) & (0.1, 0.01, 0.001, 0.0001)  & -  & \text{np.linspace($10^{-4}$, 1.5, 4)}  \\
DC-$Clip^2$ &(0.1, 0.01, 0.001, 0.0001) & (0.1, 0.01, 0.001, 0.0001)  & \text{np.linspace($10^{-4}$, 1.5, 4)}  & \text{np.linspace($10^{-4}$, 1.5, 4)}  \\
\bottomrule
\end{tabular}
}
\vspace{\baselineskip}
\end{table}

It should be noticed that due to time constraints, we only conduct the complete hyperparameter sweep for Sync-$Clip^2$ for each straggler mode. For the remaining $Clip^2$ experiments with the same straggler mode, we use the same optimal server-side and client-side upper clipping threshold obtained by the hyperparameter sweep of Sync-$Clip^2$ under the same straggler mode, given that the magnitude of the gradients is relatively stable under the same straggler mode.

Table~\ref{tab:optimal_hyp_cifar} below gives the optimal hyperparameters for all experiments over CIFAR-10.

\begin{table}[h!]
\vspace{\baselineskip}
\caption{\small Optimal Hyperparameters for experiments on CIFAR-10.}
\label{tab:optimal_hyp_cifar}
\resizebox{\textwidth}{!}{
\small\begin{tabular}{cccccc}
\toprule
Algorithm & Straggler mode &  Server-side LR  & Client-side LR  & Server-side U & Client-side U  \\ \midrule

 Sync SGDClip& large & 0.0001 & 0.01  & -  & 0.0001  \\
Sync SGDClip& mild  &0.0001 & 0.01  & -  & 1.0  \\ 
\midrule
 SC Async SGDClip& large & 0.0001 & 0.0001  & -  & 0.5  \\
 SC Async SGDClip& mild & 0.0001 & 0.01  & -  & 1.0  \\
 CC Async SGDClip& large & 0.0001 & 0.0001  & -  & 0.5  \\
 CC Async SGDClip& mild & 0.0001 & 0.001  & -  & 0.5  \\ 
 \midrule
 SC SD SGDClip& large  &0.01 & 0.1   & -  & 0.5  \\
 SC SD SGDClip& mild  &0.01 & 0.1   & -  & 1.0  \\
CC SD SGDClip& large  &0.01 & 0.01   & -  & 1.0  \\
CC SD SGDClip& mild  &0.01 & 0.001   & -  & 1.5  \\ 
\midrule

Sync $Clip^2$& large & 0.01 & 0.1  & 1.0  & 0.5  \\ 
Sync $Clip^2$& mild  &0.01 & 0.1  & 1.5  & 1.5  \\ 
\midrule
 SC Async $Clip^2$& large & 0.1 & 0.01  & 1.0  & 0.5  \\
 SC Async $Clip^2$& mild & 0.1 & 0.01  & 1.5  & 1.5  \\
 CC Async $Clip^2$& large & 0.1 & 0.01  & 1.0  & 0.5   \\
CC Async $Clip^2$& mild & 0.01 & 0.1  & 1.5  & 1.5 \\
\midrule
SC SD $Clip^2$& large & 0.1 & 0.01  & 1.0  & 0.5  \\
 SC SD $Clip^2$& mild & 0.1 & 0.01  & 1.5  & 1.5  \\
 CC SD $Clip^2$& large & 0.1 & 0.01  & 1.0  & 0.5   \\
 CC SD $Clip^2$& mild & 0.1 & 0.01  & 1.5  & 1.5 \\ 
 \midrule
 SC DC $Clip^2$& large & 0.1 & 0.01  & 1.5 & 1.5  \\
 SC DC $Clip^2$& mild & 0.1 & 0.01  & 0.5 & 1.0  \\
 CC DC $Clip^2$& large & 0.1 & 0.01  & 1.5 & 1.5  \\
 CC DC $Clip^2$& mild & 0.1 & 0.01  & 0.5 & 1.0  \\
\bottomrule
\end{tabular}
}

\vspace{\baselineskip}
\end{table}

\begin{table}[!h]
\vspace{\baselineskip}
\caption{\small Optimal Hyper-parameters for task: MNLI, QQP, RTE, SST-2 and SST-B in the GLUE benchmark.}
\label{tab:optimal_hyp_glue1}
\resizebox{\textwidth}{!}{
\small\begin{tabular}{lcccccc}
\toprule
Dataset & Algorithm & Straggler mode &  Server-side LR  & Client-side LR  & Server-side U & Client-side U  \\ 
\midrule
\multirow{3}{*}{MNLI} &  SGDClip& mild\&large & 1 & 0.56  & -  & 0.0001  \\
 &  Sync $Clip^2$& mild\&large & 1 & 0.5  & 0.75  & 0.0001  \\
 &  Async $Clip^2$& mild\&large & 1 & 0.5  & 1  & 0.0001  \\
\midrule
\multirow{10}{*}{QQP} &  SGDClip& mild\&large & 1 & 0.56  & -  & 0.0001  \\
& Sync $Clip^2$& mild\&large & 1 & 0.5  & 1.5  & 0.0001  \\
 &  SC Async $Clip^2$& mild\&large & 0.1 & 0.1  & 0.01  & 0.0001  \\
 &  CC Async $Clip^2$& mild\&large & 1 & 1  & 10.0  & 0.0001  \\
 &  SC SD $Clip^2$& mild\&large & 0.5 & 0.5  & 0.0001  & 0.0001  \\
 &  CC SD $Clip^2$& mild\&large & 1 & 1  & 10.0  & 0.0001  \\
 &  SC DC $Clip^2$& mild\&large & 0.5 & 0.5  & 0.01  & 0.0001  \\
 &  CC DC $Clip^2$& mild\&large & 1 & 1  & 10.0  & 0.0001  \\
\midrule
\multirow{2}{*}{RTE}& SGDClip&mild\&large&1&1&-&0.0001\\
& $Clip^2$ & mild\&large & 1 & 0.5  & 1.5 & 0.0001 \\
\midrule
\multirow{4}{*}{SST-2} &  SGDClip& mild\&large & 1 & 0.56  & -  & 0.0001  \\
 &  Sync $Clip^2$& large & 1 & 0.5  & 0.001  & 0.0001 \\
 &  Sync $Clip^2$& mild & 1 & 0.5  & 0.0001  & 0.0001 \\
 &  Async $Clip^2$& mild\&large & 1 & 0.5  & 0.75  & 0.0001  \\
\midrule
\multirow{2}{*}{SST-B} &  SGDClip& mild\&large & 1 & 0.44  & -  & 0.0001  \\
 &  $Clip^2$& mild\&large & 0.5 & 0.5  & 0.0001  & 0.0001  \\
\bottomrule
\end{tabular}
}
\vspace{\baselineskip}
\end{table}

\begin{table}[!h]
\vspace{\baselineskip}
\caption{\small Optimal Hyperparameters for task QNLI, MRPC, and CoLA on GLUE benchmark.}
\label{tab:optimal_hyp_glue2}
\resizebox{\textwidth}{!}{
\small\begin{tabular}{lcccccc}
\toprule
Dataset & Algorithm & Straggler mode & Server-side LR  & Client-side LR  & Server-side U & Client-side U  \\ 
\midrule

\multirow{13}{*}{QNLI} &  Sync SGDClip& mild\&large & 1 & 0.44  & -  & 0.0001   \\ 
 &  Async SGDClip& mild\&large & 1 & 0.5  & -  & 0.0001  \\
\cline{2-7}
 & Sync $Clip^2$& large  &0.5 & 0.5  & 0.001  & 0.0001  \\ 
 & Sync $Clip^2$& mild  &1 & 0.5  & 0.01  & 0.0001  \\ 
 \cline{2-7}
 &  SC Async $Clip^2$& large & 0.5 & 0.5  & 0.001  & 0.0001  \\
 &  SC Async $Clip^2$& mild & 0.5 & 0.5  & 0.01  & 0.0001  \\
 &  CC Async $Clip^2$& mild\&large & 1 & 0.5  & 0.0001 & 0.0001  \\
 \cline{2-7}
 &  SC SD $Clip^2$& large & 1 & 0.5  & 0.001  & 0.0001  \\
 &  SC SD $Clip^2$& mild & 1 & 0.5  & 0.0001  & 0.0001  \\
 &  CC SD $Clip^2$& mild\&large & 1 & 0.5  & 0.0001 & 0.0001  \\
 \cline{2-7}
 &  SC DC $Clip^2$& large & 0.5 & 0.5  & 0.001 & 0.0001  \\
 &  SC DC $Clip^2$& mild & 0.5 & 0.5  & 0.01 & 0.0001  \\
 &  CC DC $Clip^2$& mild\&large & 1 & 0.5  & 0.0001 & 0.0001  \\
\midrule
\multirow{14}{*}{MRPC} &  Sync SGDClip& mild\&large &1 & 0.89  & -  & 0.0001  \\
 &  SC Async SGDClip& mild\&large & 1 & 1  & -  & 0.0001  \\
 &  CC Async SGDClip& mild\&large &1 & 0.89  & -  & 0.0001  \\ 
 \cline{2-7}
 &  SC SD SGDClip& mild\&large  &1 & 1   & -  & 0.0001  \\
 &  CC SD SGDClip& mild\&large  &1 & 0.89  & -  & 0.0001  \\ 
 \cline{2-7}
 &  Sync $Clip^2$& large & 1 & 1  & 0.0001  & 0.0001  \\ 
 & Sync $Clip^2$& mild  &0.5 & 1  & 0.001  & 0.0001  \\ 
 \cline{2-7}
 &  SC Async $Clip^2$& mild\&large & 1 & 1  & 0.01  & 0.0001   \\
 &  CC Async $Clip^2$& mild\&large & 1 & 1  & 0.0001  & 0.0001   \\ 
 \cline{2-7}
 &  SC SD $Clip^2$& mild\&large & 1 & 1  & 0.001  & 0.0001   \\
 &  CC SD $Clip^2$& mild\&large & 1 & 1  & 0.0001  & 0.0001   \\ 
 \cline{2-7}
 &  SC DC $Clip^2$& large & 1 & 1  & 0.01 & 0.0001  \\
 &  SC DC $Clip^2$& mild & 0.5& 1  & 0.01 & 0.0001  \\
 &  CC DC $Clip^2$& mild\&large & 1 & 1  & 0.0001  & 0.0001   \\ 
\midrule
\multirow{13}{*}{CoLA} &  Sync SGDClip& mild\&large &1 & 0.89  & -  & 0.0001  \\
&  Async SGDClip& mild\&large & 1 & 0.5  & -  & 0.0001  \\ 
\cline{2-7}
 &  Sync $Clip^2$& large & 1 & 0.5  & 0.01  & 0.0001  \\ 
 & Sync $Clip^2$& mild  &1& 0.5  & 0.001  & 0.0001  \\ 
 \cline{2-7}
 &  SC Async $Clip^2$& large & 1 & 0.5  & 0.001  & 0.0001   \\
 &  SC Async $Clip^2$& mild & 1 & 0.5  & 0.0001  & 0.0001   \\
 &  CC Async $Clip^2$& mild\&large & 1 & 0.5  & 0.75  & 0.0001   \\ 
 \cline{2-7}
 &  SC SD $Clip^2$& large & 0.5 & 0.5  & 0.01  & 0.0001   \\
 &  SC SD $Clip^2$& mild & 0.5 & 0.5  & 0.001  & 0.0001   \\
 &  CC SD $Clip^2$& mild\&large & 1 & 0.5  & 0.75  & 0.0001   \\ 
 \cline{2-7}
 &  SC DC $Clip^2$& large & 1 & 0.5  & 0.0001 & 0.0001  \\
&  SC DC $Clip^2$& mild & 0.5& 0.5  & 0.001 & 0.0001  \\
 &  CC DC $Clip^2$& mild\&large & 1 & 0.5  & 0.75  & 0.0001   \\
\bottomrule
\end{tabular}
}
\vspace{\baselineskip}
\end{table}

Tables~\ref{tab:optimal_hyp_glue1} and \ref{tab:optimal_hyp_glue2} give the optimal hyperparameters for all experiments over the GLUE Benchmark. The hyperparameters that don't appear in the sweep grid are from the optimal hyperparameters of \citep{lee2025efficient}.

Table~\ref{tab:optimal_hyp_wmt} gives the optimal hyperparameters for the experiments over the TED Benchmark.

\begin{table}[h!]
\vspace{\baselineskip}
\caption{\small Optimal Hyperparameters for experiments on TED.}
\label{tab:optimal_hyp_wmt}
\resizebox{\textwidth}{!}{
\small\begin{tabular}{cccccc}
\toprule
Algorithm & Straggler mode &  Server-side LR  & Client-side LR  & Server-side U & Client-side U  \\ \midrule

 Sync SGDClip& large & 0.001 & 0.1  & -  & 0.5  \\
Sync SGDClip& mild  &0.0001 & 0.1  & -  & 1.5  \\ 
\midrule
 SC Async SGDClip& large & 0.1 & 0.1  & -  & 0.5  \\
 SC Async SGDClip& mild & 0.0001 & 0.1  & -  & 1.0  \\
 CC Async SGDClip& large & 0.0001 & 0.1  & -  & 0.5  \\
 CC Async SGDClip& mild &0.01& 0.1& -  &1.0 \\ 
 \midrule
 SC SD SGDClip& large  &0.1 & 0.1   & -  & 0.5  \\
 SC SD SGDClip& mild  &0.001 & 0.1   & -  & 1.0  \\
CC SD SGDClip& large  &0.0001 & 0.1   & -  & 1.0  \\
CC SD SGDClip& mild  &0.001 & 0.1 & -  &  1.5 \\ 
\midrule

Sync $Clip^2$& large & 0.1 & 0.1  & 1.0  & 0.5  \\ 
Sync $Clip^2$& mild  &0.1 & 0.1  & 1.5  & 0.5  \\ 
\midrule
 SC Async $Clip^2$& large & 0.1 & 0.1  & 1.0  & 0.5  \\
 SC Async $Clip^2$& mild &0.1 & 0.1  & 1.5  & 0.5  \\
 CC Async $Clip^2$& large & 0.1 & 0.1  & 1.0  & 0.5   \\
CC Async $Clip^2$& mild &0.1 & 0.1  & 1.5  & 0.5 \\
\midrule
SC SD $Clip^2$& large & 0.1 & 0.1  & 1.0  & 0.5  \\
 SC SD $Clip^2$& mild &0.1 & 0.1  & 1.5  & 0.5  \\
 CC SD $Clip^2$& large & 0.1 & 0.1  & 1.0  & 0.5   \\
 CC SD $Clip^2$& mild &0.1 & 0.1  & 1.5  & 0.5 \\ 
 \midrule
 SC DC $Clip^2$& large & 0.1 & 0.1  & 1.5 & 0.5  \\
 SC DC $Clip^2$& mild & 0.1 & 0.1  & 1.0 & 0.5  \\
 CC DC $Clip^2$& large & 0.1 & 0.1  & 1.5 & 0.5  \\
 CC DC $Clip^2$& mild & 0.1 & 0.01  & 0.5 & 1.0  \\
\bottomrule
\end{tabular}
}

\vspace{\baselineskip}
\end{table}

\section{Experimental Results}
\label{app:exp_results}

In this section, we present all the specific result tables of the experiments mentioned in Section~\ref{sec:exp}. The accuracies are the best accuracies obtained during the sweep of hyperparameters.

\subsection{Accuracies and Runtime Across Different Methods for CIFAR-10} \label{app:sub:exp_results}
% \begin{table}[tb]

%     \centering
%     \resizebox{\textwidth}{!}{
%     \begin{tabular}{c|c | c | c | c | c | c}
%       \toprule
%       Methods   &  \multicolumn{2}{c|}{CIFAR10} &  \multicolumn{2}{|c|}{GLUE} & \multicolumn{2}{|c}{WMT} \\
%        & loss & runtime & loss & runtime & loss & runtime \\
%       \hline
%       Sync SGDClip (mild straggler) & 0.0682 & 818&  &  & & \\
%        Server-centric Async SGDClip (mild straggler) & 0.0696 & 40 & & & & \\
%        Client-centric Async SGDClip (mild straggler)  & 0.0638 & 36 & &  & &  \\
%        Sync SGDClip (large straggler) & 0.0676 & 3170&  &  & & \\
%        Server-centric Async SGDClip (large straggler) & 0.0680 & 44 & & & & \\
%        Client-centric Async SGDClip (large straggler)  & 0.0725 & 40 & &  & &  \\
%     \end{tabular}
%     }
%     \caption{Best Loss and Runtime comparisons of Sync SGDClip and Async SGDClip. Asynchronous training achieves similar test performance compared with the synchronous case but requires much less runtime.}
%     \label{tab:1}
% \end{table}

Please see results in Table~\ref{tab:1}-Table~\ref{tab:cifar_baseline}. 

\begin{table}[!h]
\vspace{\baselineskip}
\caption{\small Best Accuracy and Runtime comparisons of Sync SGDClip and Async SGDClip. Asynchronous training achieves similar test performance compared with the synchronous case but requires much less runtime.}
    \label{tab:1}
    \centering
    \resizebox{0.9\textwidth}{!}{
    \begin{tabular}{c|c | c | c | c | c | c}
      \toprule
      Methods   &  \multicolumn{2}{c}{CIFAR-10} &  \multicolumn{2}{c}{GLUE}  & \multicolumn{2}{c}{TED}\\
       & Acc. & runtime & Acc. & runtime & BLEU & runtime \\
      \midrule
      Sync SGDClip (mild straggler) & 98.6 & 818&82.5&93 &29.6 & 177  \\
       Server-centric Async SGDClip (mild straggler) & 97.7 & 40 &83.4&21 &29.6 & 49   \\
       Client-centric Async SGDClip (mild straggler)  & 98.4 & 36 & 83.0& 20 &30.1&34\\
       Sync SGDClip (large straggler) & 98.6 & 3170&82.5& 445 &29.8 & 422  \\
       Server-centric Async SGDClip (large straggler) & 97.7 & 44 &82.2&36 &29.7 & 46 \\
       Client-centric Async SGDClip (large straggler)  & 97.9 & 40 &83.4&26 & 28.9 &36 \\
       \bottomrule
    \end{tabular}
    }
    \vspace{\baselineskip}
\end{table}

% \begin{table}[tb]
%     \centering
%     \resizebox{\textwidth}{!}{
%     \begin{tabular}{c|c | c | c | c | c | c}
%       \toprule
%       Methods   &  \multicolumn{2}{c|}{CIFAR10} &  \multicolumn{2}{|c|}{GLUE} & \multicolumn{2}{|c}{WMT} \\
%        & loss & runtime & loss & runtime & loss & runtime \\
%       \hline
%       Sync $Clip^2$ (mild straggler) & 0.0567 & 860&  &  & & \\
%        Server-centric Async $Clip^2$ (mild straggler) & 0.0581 & 40 & & & & \\
%        Client-centric Async $Clip^2$ (mild straggler)  & 0.0629 & 37 & &  & &  \\
%        Sync $Clip^2$ (large straggler) & 0.0562 & 3139&  &  & & \\
%        Server-centric Async $Clip^2$ (large straggler) & 0.0658 & 42 & & & & \\
%        Client-centric Async $Clip^2$ (large straggler)  & 0.0654 & 39 & &  & &  \\ 
%     \end{tabular}
%     }
%     \caption{Best Loss and Runtime comparisons of Sync $Clip^2$ and Async $Clip^2$. Asynchronous training achieves similar test performance compared with the synchronous case but requires much less runtime.}
%     \label{tab:2}
% \end{table}

\begin{table}[!h]
\vspace{\baselineskip}
\caption{\small Best Accuracy and Runtime comparisons of Sync $Clip^2$ and Async $Clip^2$. Asynchronous training achieves similar test performance compared with the synchronous case but requires much less runtime.}
    \label{tab:2}
    \centering
    \resizebox{0.9\textwidth}{!}{
    \begin{tabular}{c|c | c | c | c |c |c}
      \toprule
      Methods   &  \multicolumn{2}{c|}{CIFAR-10} &  \multicolumn{2}{c}{GLUE} & \multicolumn{2}{c}{TED} \\
       & Acc. & runtime & Acc. & runtime & BLEU & runtime \\
      \midrule
      Sync $Clip^2$ (mild straggler) & 98.4 & 860&82.7&95 & 29.3 & 154 \\
       Server-centric Async $Clip^2$ (mild straggler) & 98.3 & 40 &82.9&23 & 29.1 & 43 \\
       Client-centric Async $Clip^2$ (mild straggler)  & 98.1 & 37 &83.9& 26 & 28.9 & 37 \\
       Sync $Clip^2$ (large straggler) & 98.4 & 3139&82.6  &380 & 28.8 & 491 \\
       Server-centric Async $Clip^2$ (large straggler) & 98.2 & 42 &82.0&21 & 29.1 & 45 \\
       Client-centric Async $Clip^2$ (large straggler)  & 98.3 & 39 &83.3& 31 & 29.3 & 36  \\ 
       \bottomrule
    \end{tabular}
    }
    
    \vspace{\baselineskip}
\end{table}

% \begin{table}[tb]
%     \centering
%     \resizebox{\textwidth}{!}{
%     \begin{tabular}{c|c|c|c|c|c|c|c}
%       \toprule
%       Async Mode & Methods   & \multicolumn{2}{c|}{CIFAR10} & \multicolumn{2}{c|}{GLUE} & \multicolumn{2}{c|}{WMT} \\
%       \cline{3-8}
%       &           & loss & runtime & loss & runtime & loss & runtime \\
%       \hline
      
%       % The new multi-row column now has normal, horizontal text.
%       \multirow{4}{*}{\parbox{2.5cm}{\centering Server-centric}}
%       & SGDClip (mild straggler) & 0.0696 & 40 & & & & \\
%       & \textbf{SD-SGDClip (mild straggler)} & 0.0659 & 41 & & & & \\
%       & SGDClip (large straggler) & 0.0680 & 44 & & & & \\
%       & \textbf{SD-SGDClip (large straggler)} & 0.0802 & 44 & & & & \\ 
%       \hline
%       \multirow{4}{*}{\parbox{2.5cm}{\centering Client-centric}}
%       & SGDClip (mild straggler)  & 0.0638 & 36 & &  & &  \\
%       & \textbf{SD-SGDClip (mild straggler)} & 0.0686 & 36 & &  & &  \\
%       & SGDClip (large straggler)  & 0.0725 & 40 & &  & &  \\
%       & \textbf{SD-SGDClip (large straggler)}  & 0.0852 & 39 & &  & &  \\
%       \hline
%     \end{tabular}
%     }
%     \caption{Best Loss and Runtime comparisons of Async SGDClip and Async SD-SGDClip.}
%     \label{tab:3}
% \end{table}

\begin{table}[!h]
\vspace{\baselineskip}
\caption{\small Best Accuracy and Runtime comparisons of Async SGDClip and Async SD-SGDClip.}
    \label{tab:3}
    \centering
    \resizebox{0.9\textwidth}{!}{
    \begin{tabular}{c|c|c|c|c|c|c|c}
      \toprule
      Async Mode & Methods   & \multicolumn{2}{c|}{CIFAR-10} & \multicolumn{2}{c}{GLUE} & \multicolumn{2}{c}{TED} \\
      &           & Acc. & runtime & Acc. & runtime & BLEU & runtime \\
      \midrule
      
      % The new multi-row column now has normal, horizontal text.
      \multirow{4}{*}{\parbox{2.5cm}{\centering Server-centric}}
      & SGDClip (mild straggler) & 97.7 & 40 &83.4&21 & 29.6 & 49 \\
      & \textbf{SD-SGDClip (mild straggler)} & 98.2 & 41 &82.0&24 & 29.9 & 51 \\
      & SGDClip (large straggler) & 97.7 & 44 &82.2&36 & 29.7 & 46 \\
      & \textbf{SD-SGDClip (large straggler)} & 98.2 & 44 &81.9&22 & 29.6 & 48  \\ 
      \midrule
      \multirow{4}{*}{\parbox{2.5cm}{\centering Client-centric}}
      & SGDClip (mild straggler)  & 98.4 & 36 &83.0&20 &30.1&34 \\
      & \textbf{SD-SGDClip (mild straggler)} & 98.2 & 36 &82.3&17 &30.3&36  \\
      & SGDClip (large straggler)  & 97.9 & 40  &83.4&26 & 28.9 & 36  \\
      & \textbf{SD-SGDClip (large straggler)}  & 98.0 & 39 &82.8&30 & 29.8 & 34 \\
      \bottomrule
    \end{tabular}
    }
    \vspace{\baselineskip}
\end{table}

\begin{table}[!h]
\vspace{\baselineskip}
\caption{\small Best Accuracy and Runtime comparisons of Async $Clip^2$, Async SD-$Clip^2$, and Async DC-$Clip^2$.}
    \label{tab:4}
    \centering
    \resizebox{0.8\textwidth}{!}{
    
    \begin{tabular}{c|c|c|c|c|c|c|c}
     \toprule
      Async Mode & Methods   & \multicolumn{2}{c|}{CIFAR-10} & \multicolumn{2}{c}{GLUE} & \multicolumn{2}{c}{TED}\\
      &         & Acc. & runtime & Acc. & runtime & BLEU & runtime\\
      \midrule
      
      % First multi-row cell spanning 6 rows
      \multirow{6}{*}{\parbox{2.5cm}{\centering Server-centric}}
      & \textit{Clip$^2$} (mild straggler) & 98.3 & 40 &82.9& 23 & 29.1 & 43 \\
      & \textbf{SD-\textit{Clip$^2$} (mild straggler)} & 98.2 & 39 &82.7&22 & 28.7 & 40 \\
      & \textbf{DC-\textit{Clip$^2$} (mild straggler)} & 98.4 & 42 &81.8&24 & 29.3 & 40  \\
      & \textit{Clip$^2$} (large straggler) & 98.2 & 42 &82.0&21 & 29.1 & 45 \\
      & \textbf{SD-\textit{Clip$^2$} (large straggler)} & 98.2 & 43 &81.6&18 & 28.9 & 45 \\
      & \textbf{DC-\textit{Clip$^2$} (large straggler)} & 98.5 & 45 &81.5&25 & 29.0 & 43 \\
      \midrule
      % Second multi-row cell spanning 6 rows
      \multirow{6}{*}{\parbox{2.5cm}{\centering Client-centric}}
      & \textit{Clip$^2$} (mild straggler) & 98.1 & 37 & 83.9 & 26 & 28.9 & 37 \\
      & \textbf{SD-\textit{Clip$^2$} (mild straggler)} & 98.0 & 37 & 84.0 & 23 & 28.7 & 36 \\
      & \textbf{DC-\textit{Clip$^2$} (mild straggler)} & 98.3 & 36 & 84.3 & 33 & 29.2 & 31 \\
      & \textit{Clip$^2$} (large straggler) & 98.3 & 39 & 83.3 & 31 & 29.3 & 36 \\
      & \textbf{SD-\textit{Clip$^2$} (large straggler)} & 97.9 & 39 & 84.3 & 41 & 28.8 & 36 \\
      & \textbf{DC-\textit{Clip$^2$} (large straggler)} & 98.3 & 39 & 83.4 & 90 & 29.0 & 32  \\
      \bottomrule
    \end{tabular}}
    \vspace{\baselineskip}
\end{table}

For baseline comparison results presented in Table~\ref{tab:cifar_baseline}, \ref{tab:ted_baseline},~\ref{tab:glue-baseline}, we try our best to tune hyperparameters for the baseline methods. E.g., for FADAS, we tune client- and server-side learning rates from (0.1, 0.01, 0.001, 0.0001) and the delay threshold in their paper $\tau_c$ from (1, 4, 8, 10). For DN-DyLU, we tune learning rates in the same way and $N$ from (4, 8, 16). For RANSGDm, we also tune the learning rate in the same way and apply the other optimal hyperparameters in the original paper \citep{wu2026optimalasynchronousstochasticnonconvex}.

\begin{table}[!h]
\vspace{\baselineskip}
\caption{\small Accuracies (in \%) Comparison for SD-$Clip^2$ and DC-$Clip^2$ with FADAS \citep{wang2024fadas}, DN+DyLU \citep{liu2024asynchronous}, and RANSGDm \citep{wu2026optimalasynchronousstochasticnonconvex} on CIFAR-10.}
\label{tab:cifar_baseline}
\centering
\resizebox{0.45\textwidth}{!}{
\begin{tabular}{c|c|c}
     \toprule
      Async Mode & Methods   & Acc. \\
      \midrule
      
      % First multi-row cell spanning 6 rows
      \multirow{6}{*}{\parbox{2.5cm}{\centering Server-centric}}
      & \textit{Clip$^2$} (mild straggler) & 98.3 \\
      & \textbf{SD-\textit{Clip$^2$} (mild straggler)} & 98.2  \\
      & \textbf{DC-\textit{Clip$^2$} (mild straggler)} & 98.4  \\
      & \textit{Clip$^2$} (large straggler) & 98.2 \\
      & \textbf{SD-\textit{Clip$^2$} (large straggler)} & 98.2 \\
      & \textbf{DC-\textit{Clip$^2$} (large straggler)} & 98.5 \\
      \midrule
      % Second multi-row cell spanning 6 rows
      \multirow{6}{*}{\parbox{2.5cm}{\centering Client-centric}}
      & \textit{Clip$^2$} (mild straggler) & 98.1 \\
      & \textbf{SD-\textit{Clip$^2$} (mild straggler)} & 98.0 \\
      & \textbf{DC-\textit{Clip$^2$} (mild straggler)} & 98.3 \\
      & \textit{Clip$^2$} (large straggler) & 98.3 \\
      & \textbf{SD-\textit{Clip$^2$} (large straggler)} & 97.9 \\
      & \textbf{DC-\textit{Clip$^2$} (large straggler)} & 98.3 \\
      \midrule
      % Baselines: one multi-row cell spanning 6 rows
      \multirow{6}{*}{\parbox{2.5cm}{\centering N/A}}
      & \textit{FADAS} \citep{wang2024fadas} (mild straggler) & 93.5 \\
      & \textit{FADAS} \citep{wang2024fadas} (large straggler) & 94.3 \\
      & \textit{DN+DyLU} \citep{liu2024asynchronous} (mild straggler) & 97.3 \\
      & \textit{DN+DyLU} \citep{liu2024asynchronous} (large straggler) & 97.5 \\
      & \textit{RANSGDm} \citep{wu2026optimalasynchronousstochasticnonconvex} (mild straggler) & 97.8 \\
      & \textit{RANSGDm} \citep{wu2026optimalasynchronousstochasticnonconvex} (large straggler) & 97.8 \\
      \bottomrule
    \end{tabular}}

\vspace{\baselineskip}
\end{table}

\begin{table}[!h]
\vspace{\baselineskip}
\caption{\small Accuracies (in \%) Comparison for SD-$Clip^2$ and DC-$Clip^2$ with RANSGDm \citep{wu2026optimalasynchronousstochasticnonconvex} on TED.}
\label{tab:ted_baseline}
\centering
\resizebox{0.45\textwidth}{!}{
\begin{tabular}{c|c|c}
     \toprule
      Async Mode & Methods   & BLEU \\
      \midrule
      
      % First multi-row cell spanning 6 rows
      \multirow{6}{*}{\parbox{2.5cm}{\centering Server-centric}}
      & \textit{Clip$^2$} (mild straggler) & 29.1 \\
      & \textbf{SD-\textit{Clip$^2$} (mild straggler)} & 28.7  \\
      & \textbf{DC-\textit{Clip$^2$} (mild straggler)} & 29.3  \\
      & \textit{Clip$^2$} (large straggler) & 29.1 \\
      & \textbf{SD-\textit{Clip$^2$} (large straggler)} & 28.9 \\
      & \textbf{DC-\textit{Clip$^2$} (large straggler)} & 29.0 \\
      \midrule
      % Second multi-row cell spanning 6 rows
      \multirow{6}{*}{\parbox{2.5cm}{\centering Client-centric}}
      & \textit{Clip$^2$} (mild straggler) & 28.9 \\
      & \textbf{SD-\textit{Clip$^2$} (mild straggler)} & 28.7 \\
      & \textbf{DC-\textit{Clip$^2$} (mild straggler)} & 29.2 \\
      & \textit{Clip$^2$} (large straggler) & 29.3 \\
      & \textbf{SD-\textit{Clip$^2$} (large straggler)} & 28.8 \\
      & \textbf{DC-\textit{Clip$^2$} (large straggler)} & 29.0 \\
      \midrule
      % Baselines: one multi-row cell spanning 6 rows
      \multirow{2}{*}{\parbox{2.5cm}{\centering N/A}}
      & \textit{RANSGDm} \citep{wu2026optimalasynchronousstochasticnonconvex} (mild straggler) &  27.3\\
      & \textit{RANSGDm} \citep{wu2026optimalasynchronousstochasticnonconvex} (large straggler) & 27.6 \\
      \bottomrule
    \end{tabular}}

\vspace{\baselineskip}
\end{table}

\subsection{The Effect of M} \label{app:sub:m_effect}

To explore the effect of $M$ for different methods, we run the image classification experiment on the dataset CIFAR-10 with $M\in\{1,10, 20, 30\}$ using the optimal hyperparameters determined by the sweep explained in Appendix~\ref{app:sub:hyp_sweep} when $M=4$. Notice that in general there is no clear dominance of server-centric or client-centric regimes over the other, so we run  the server-centric version for asynchronous methods. The results are presented in Table~\ref{tab:different_M_cifar10}.

% \begin{table}[tb]

% \resizebox{\textwidth}{!}{
% \begin{tabular}{lcccccccccc}
% \hline\hline
% Algorithm & straggler mode &  M=1 Loss  & M=1 Runtime  &  M=10 Loss  & M=10 Runtime &  M=20 Loss  & M=20 Runtime &  M=30 Loss  & M=30 Runtime \\ \hline

% Sync SGDClip& large & 0.136 & 1328  & 0.154  & 4715 & 0.160 & 5163 & 0.170 & 5393  \\
% Async SGDClip& large & 5.660 & 10  & 0.081  & 115 & 0.076 & 289 & 0.073 & 832  \\

% SD SGDClip& large  & 39853.162 & 12   & 0.128  & 122 & 0.132 & 291 & 0.111 & 869   \\ \hline

% Sync SGDClip& mild  &0.111 & 559  & 0.079 & 1036  & 0.084 & 1081 & 0.086 & 1090  \\

% Async SGDClip& mild & 46.389 & 11  & 0.108  & 104 & 0.076 & 248 & 0.080 & 541  \\

% SD SGDClip& mild  &85627.005 & 12   & 0.123  & 98 &  0.137 & 248 & 0.123 & 541 \\ \hline 

% Sync $Clip^2$& large & 0.688 & 1424  & 0.550  & 4645 & 0.617 & 5447 & 0.608 & 5420   \\ 

% Async $Clip^2$& large & 0.083 & 14  & 0.066  & 117 & 0.065 & 282 & 0.060 & 824    \\

% SD $Clip^2$& large & 0.073 & 12  & 0.069  & 119 & 0.065 & 285 & 0.064 & 813\\ \hline

% Sync $Clip^2$& mild  &0.836 & 489  & 0.694  & 1027 & 0.780 & 1096 & 0.721 & 1100  \\ 
% Async $Clip^2$& mild & 0.094 & 12  & 0.060  & 99 & 0.057 & 247 & 0.058 & 538  \\
% SD $Clip^2$& mild & 1.211 & 11  & 1.527  & 1.01 & 0.068 & 244 & 0.056 & 538 \\
% \hline 

% DC $Clip^2$& large & 0.079 & 11  & 0.063 & 115 & - & - & - & -  \\

% DC $Clip^2$& mild & 0.086 & 10  & 0.071 & 102 & 0.065 & 249 & 0.061 & 544    \\
% \hline

% \hline\hline
% \end{tabular}
% }
% \caption{Loss and Runtime for different M's}
% \label{tab:optimal_hyp}
% \end{table}

\begin{table}[!h]
\vspace{\baselineskip}
\caption{\small Accuracies (in \%) and runtimes for different $M$ values on CIFAR-10.}
\label{tab:different_M_cifar10}
\centering
\resizebox{\textwidth}{!}{
\begin{tabular}{lcccccccccc}
\toprule
Algorithm & Straggler & M=1 Acc. & M=1 Runtime & M=10 Acc. & M=10 Runtime & M=20 Acc. & M=20 Runtime & M=30 Acc. & M=30 Runtime \\ 
\midrule

Sync SGDClip   & large & 97.4 & 1328 & 97.0 & 4715 & 97.0 & 5163 & 96.9 & 5393 \\
Async SGDClip  & large & 10.3 & 10   & 97.3 & 115  & 97.5 & 289  & 97.7 & 832  \\
SD SGDClip     & large & 10.3 & 12   & 97.4 & 122  & 97.7 & 291  & 97.8 & 869  \\ 
\midrule

Sync SGDClip   & mild  & 96.8 & 559  & 98.5 & 1036 & 98.3 & 1081 & 98.4 & 1090 \\
Async SGDClip  & mild  & 9.8  & 11   & 98.3 & 104  & 98.5 & 248  & 98.4 & 541  \\
SD SGDClip     & mild  & 9.8  & 12   & 97.1 & 98   & 97.6 & 248  & 97.9 & 541  \\ 
\midrule

Sync $Clip^2$  & large & 94.6 & 1424 & 94.5 & 4645 & 95.0 & 5447 & 95.0 & 5420 \\
Async $Clip^2$ & large & 97.5 & 14   & 98.2 & 117  & 98.3 & 282  & 98.3 & 824  \\
SD $Clip^2$    & large & 98.0 & 12   & 98.0 & 119  & 98.2 & 285  & 98.4 & 813  \\ 
\midrule

Sync $Clip^2$  & mild  & 93.6 & 489  & 95.1 & 1027 & 94.8 & 1096 & 94.9 & 1100 \\
Async $Clip^2$ & mild  & 97.8 & 12   & 98.4 & 99   & 98.4 & 247  & 98.5 & 538  \\
SD $Clip^2$    & mild  & 92.4 & 11   & 79.5 & 101  & 98.1 & 244  & 98.5 & 538  \\ 
\midrule

DC $Clip^2$    & large & 98.0 & 11   & 98.3 & 115  &  98.3  &  281  &  98.4  &  544  \\
DC $Clip^2$    & mild  & 97.7 & 10   & 98.0 & 102  & 98.2 & 249  & 98.2 & 544  \\

\bottomrule
\end{tabular}
}
\vspace{\baselineskip}
\end{table}

Notably, we see that non-synchronous $SGDClip$ with both large and mild stragglers give significantly lower accuracies ($10.3$ and $9.8$) when $M=1$. We believe this stems from interplays between heavy-tailed noise and asynchrony: When $M=1$, the server updates its parameter whenever a client sends its results back, and the server immediately sends its updated parameters to that client to start a new round of local updates. This can be detrimental when the server receives delayed updates from an extreme straggler, which comes from an old model that is drastically different from the current global model. In the meantime, fast clients will more frequently update the model. Therefore, the negative effect of the delayed updates from stragglers is exaggerated compared to a larger $M$, thus causing the model to not converge, especially with the existence of heavy-tailed noise. We see that both server-side clipping and simply using synchronous training effectively address the problem of low accuracy. This suggests that our proposed stateless-aware aggregation method would achieve better performance when used together with $Clip^2$ than $SGDClip$, in extreme asynchronous cases.

\subsection{Additional Results} \label{app:sub:add}

\begin{figure}[!h]
\vspace{\baselineskip}
    \centering
    \includegraphics[width=\linewidth]{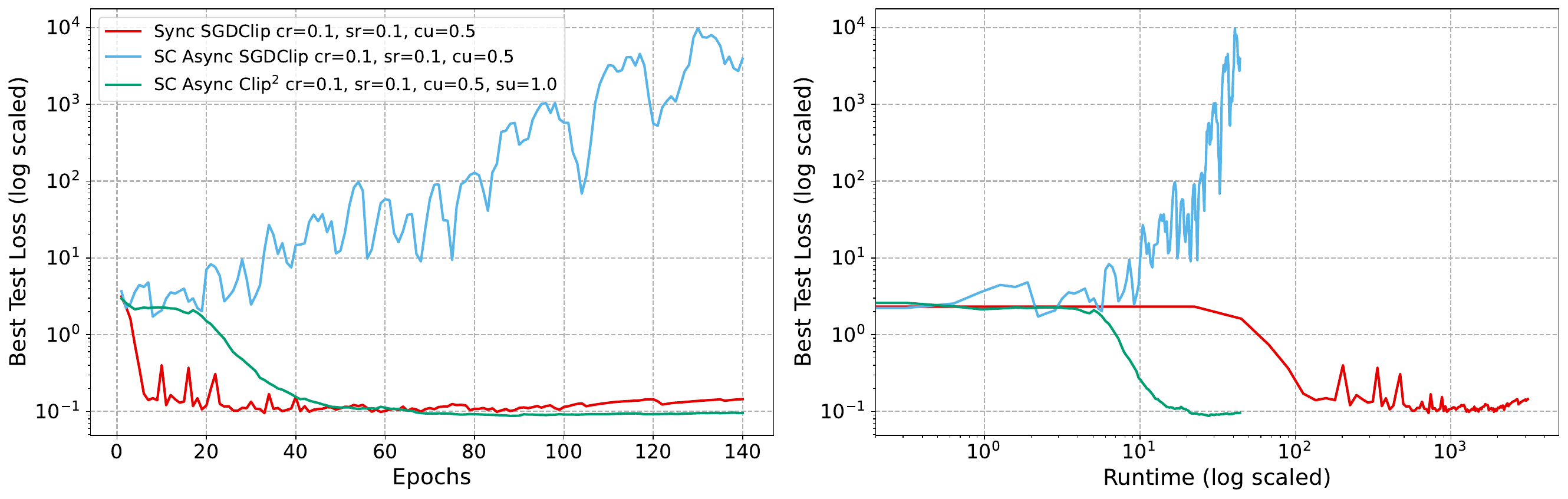}
    \caption{\small Test loss versus \#epochs and runtime of Sync SGDClip, Async SGDClip and $Clip^2$ under a specific hyperparameter choice on CIFAR-10. `cr, sr, cu, su' denotes client learning rate, server learning rate, client upper clipping bound, and server upper clipping bound, respectively. We see that vanilla asynchronous could cause divergence, whereas $Clip^2$, which is originally designed to handle heavy-tailed noise, can help convergence.% \dixi{Do we need to explain the meanings of notation of cr, sr, cu in captions are somewhere?} 
    } 
    \label{fig:extreme_async}
    \vspace{\baselineskip}
\end{figure}

Please see Figure~\ref{fig:normal_async} which shows the loss/epochs and loss/runtime tradeoffs on the mild straggler setting on the CIFAR-10 dataset.
\begin{figure}[h!]
\vspace{\baselineskip}
    \centering
    \includegraphics[width=\linewidth]{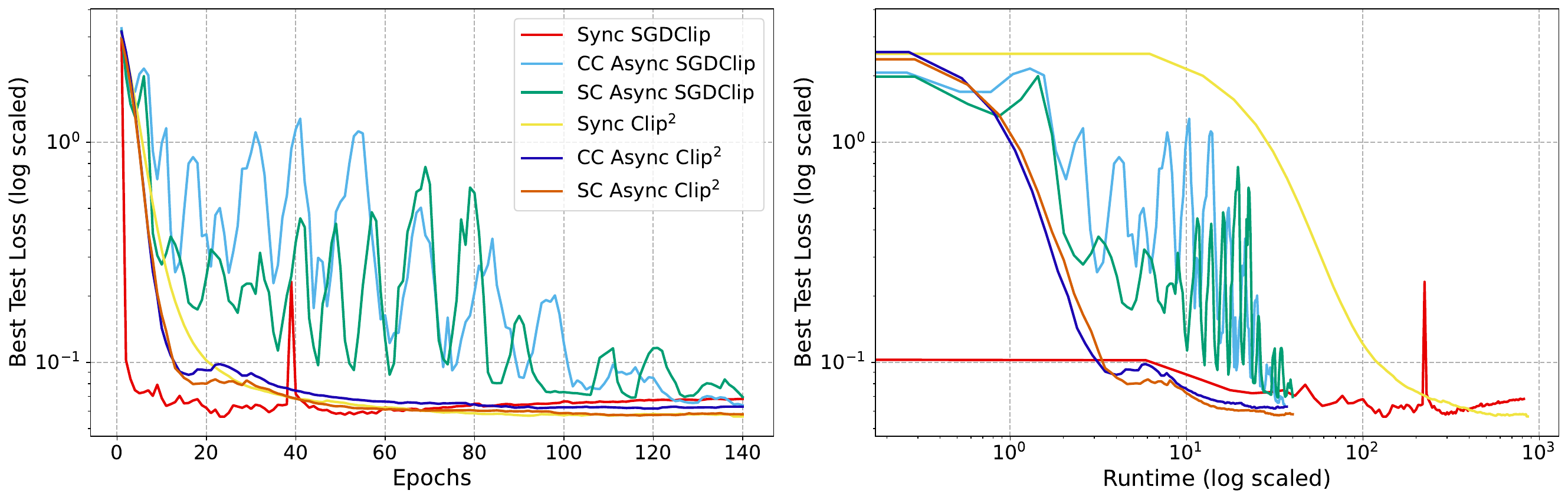}
    \caption{\small Best Test Loss v.s. Epochs and Runtime of Sync SGDClip/$Clip^2$ and Async SGDClip/$Clip^2$ under the mild straggler setting on CIFAR-10.}
    \label{fig:normal_async}
    \vspace{\baselineskip}
\end{figure}

% \begin{figure}[!h]
%     \centering
%     \includegraphics[width=0.6\linewidth]{images/clip2_advantage.pdf}
%     \caption{Best Test Loss v.s. Epochs and Runtime of Sync SGDClip, Async SGDClip and $Clip^2$ under a specific hyperparameter choice.} 
%     \label{fig:clip2_advantage}
% \end{figure}

\subsection{Detailed Experiment Results over Each Task in GLUE Benchmark}
In this section, we report the experimental results on the GLUE benchmark for each individual task. Specifically, we evaluate three settings: (1) synchronous and asynchronous fine-tuning without staleness-aware downplaying or delay compensation (Table~\ref{tab:glue-vanila}); (2) fine-tuning with staleness-aware downplaying 86.0 (Table~\ref{tab:glue-sa}); and (3) fine-tuning with delay compensation (Table~\ref{tab:glue-dc}).

\begin{table}[!h]
\vspace{\baselineskip}
\caption{\small Utility Metrics and Runtime comparison of synchronous, server-centric-asynchronous, and client-centric-asynchronous model across SGDClip and $Clip^2$ over different circumstances of stragglers.}
\label{tab:glue-vanila}
\resizebox{\textwidth}{!}{
\begin{tabular}{ccccccccccccccccccc}
\toprule
&  & Tasks  & \multicolumn{2}{c}{MNLI}  & \multicolumn{2}{c}{QNLI} & \multicolumn{2}{c}{QQP} & \multicolumn{2}{c}{RTE} & \multicolumn{2}{c}{SST-2} & \multicolumn{2}{c}{MPRC} & \multicolumn{2}{c}{CoLA} & \multicolumn{2}{c}{STS-B} \\ 
\midrule
Algorithm  & Straggler & & Acc. & Runtime & Acc. & Runtime & Acc. & Runtime & Acc. & Runtime & Acc. & Runtime & Acc. & Runtime & Acc.  & Runtime  & Pearson  & Runtime  \\ 
\midrule
\multirow{6}{*}{SGDClip} & \multirow{3}{*}{Mild}  & Sync.  & 83.15  & 96  & 86.58 & 102  & 86.98 & 96  & 57.76  & 105  & 91.86      & 88  & 84.80 & 94 & 81.11  & 63 &87.32&103\\& & Server & 82.50  & 19  & 84.72  & 18  & 86.33  & 23&69.26 & 26& 92.09  & 16 & 82.11 & 24      & 81.59& 19 &88.92&22\\
&  & Client & 82.18& 19& 86.91& 18& 81.91& 15& 70.04& 16& 91.51& 18& 82.35& 32& 80.15& 23& 89.33& 19\\ \cline{2-19} 
& \multirow{3}{*}{Large} & Sync.  & 82.88  & 451 & 86.53   & 510  & 86.92  & 509  & 63.18  & 443 & 92.20  & 516 & 83.58 & 303       & 81.21 & 369 &83.65& 457\\
&  & Server & 82.58&21&85.78&22&86.09&16&60.29&16&91.97&17&80.39&19&81.69&31 &88.42&26\\
&  & Client & 82.21& 21& 87.48& 27& 85.03& 27& 70.07& 16& 90.94& 12& 82.60& 70& 80.35& 23& 88.22& 13\\ 
\midrule
\multirow{6}{*}{$Clip^2$} & \multirow{3}{*}{Mild}  & Sync.  &  83.67 &97&87.33&88&79.86&86&64.30&104&92.88& 95 &86.03 & 87 & 82.08& 105 &86.54&99\\
&  & Server & 83.54&28&84.72   &18&86.33&23& 60.29&26& 92.66   & 22 &86.03&16& 81.11 & 26&88.14 &22\\
&  & Client & \multicolumn{1}{c}{81.87} & 31& 86.86& 35& 85.37& 20& 70.04  & 27& 91.63& 14& 87.01& 18& 81.02& 23& 87.15& 40\\ \cline{2-19} 
& \multirow{3}{*}{Large} & Sync.  &83.30&419&87.22&511&83.03&406&64.30&480&92.78&422&85.05&281&79.77&58&86.43&461 \\
&  & Server & 82.58&21&85.78&22&86.09&16&59.95&16&91.97&17&83.58 &19&80.44&31&85.63 &23\\
&   & Client & \multicolumn{1}{c}{82.19} & 14& 86.89& 22& 85.39& 41& 67.87& 45& 92.29& 69& 85.05& 19& 80.15& 23& 86.92& 17\\ 
\bottomrule
\end{tabular}
}
\vspace{\baselineskip}
\end{table}

\begin{table}[!h]
\vspace{\baselineskip}
\caption{\small Utility Metrics and Runtime comparison of server-centric-asynchronous and client-centric-asynchronous model across SGDClip and $Clip^2$ over different circumstances of stragglers where we apply staleness-aware downplaying.}
\label{tab:glue-sa}
\resizebox{\textwidth}{!}{
\begin{tabular}{ccccccccccccccccccc}
\toprule
&  & Tasks  & \multicolumn{2}{c}{MNLI} & \multicolumn{2}{c}{QNLI} & \multicolumn{2}{c}{QQP} & \multicolumn{2}{c}{RTE} & \multicolumn{2}{c}{SST-2} & \multicolumn{2}{c}{MPRC} & \multicolumn{2}{c}{CoLA} & \multicolumn{2}{c}{STS-B} \\ \hline
Algorithm  & Straggler &  & Acc. & Runtime & Acc. & Runtime  & Acc. & Runtime  & Acc. & Runtime  & Acc.  & Runtime  & Acc. & Runtime & Acc.  & Runtime  & Pearson  & Runtime      \\ 
\midrule
\multirow{4}{*}{SGDClip} & \multirow{2}{*}{Mild}  & Server & 81.57&23& 84.40 &26&  85.35 &23 &  61.39  &20& 91.17  & 25& 83.82  &24& 82.07&22&86.36&25\\
 &  & Client & 80.14 & 23  & 86.53  & 13  & 82.78  & 11 & 64.26 & 9  & 90.94 & 22   & 83.33 & 26   & 81.40  & 13   & 88.65 & 20            \\ \cline{2-19} 
& \multirow{2}{*}{Large}  & Server & 81.50 &26& 81.99  &23& 85.26 &20& 60.3 & 21 & 92.10 &23& 85.29 &22& 81.40 &20&87.36&22\\
 &  & Client & 82.20 & 14  & 87.42& 43  & 84.43  & 13 & 63.90  & 13   & 91.17 & 25  & 83.58 & 33 & 81.02  & 28 & 88.46 & 68            \\ 
 \midrule
\multirow{4}{*}{$Clip^2$}   & \multirow{2}{*}{Mild}  & Server & 82.18 & 21 & 87.24 &26  & 84.48 & 17 & 64.32 &21&93.11&20&82.11&20&80.47& 22&87.33&25\\
& & Client & \multicolumn{1}{c}{83.42} & 14 & 86.91  & 20  & 86.10 & 37  & 70.76  & 30  & 92.20 & 39  & 85.78& 13 & 80.35  & 20 & 86.19  & 13   \\ \cline{2-19} 
& \multirow{2}{*}{Large} & Server & 82.50& 19&87.59& 14& 84.47 & 21 & 54.51 & 23 & 92.78 & 15& 84.31 & 17 & 81.40  & 15 & 85.60 & 21\\
&   & Client & \multicolumn{1}{c}{83.35} & 16  & 86.55  & 55  & 86.96  & 40 & 72.56 & 74  & 91.51  & 65  & 86.52 & 18  & 80.73  & 38          & 86.60  & 21   \\ 
\bottomrule
\end{tabular}
}
\vspace{\baselineskip}
\end{table}

\begin{table}[!h]
\vspace{\baselineskip}
\caption{\small Utility Metrics and Runtime comparison of server-centric-asynchronous and client-centric-asynchronous model across $Clip^2$ over different circumstances of stragglers where we apply delay compensation.}
\label{tab:glue-dc}
\resizebox{\textwidth}{!}{
\begin{tabular}{ccccccccccccccccccc}
\toprule
&   & Tasks  & \multicolumn{2}{c}{MNLI} & \multicolumn{2}{c}{QNLI} & \multicolumn{2}{c}{QQP} & \multicolumn{2}{c}{RTE} & \multicolumn{2}{c}{SST-2} & \multicolumn{2}{c}{MPRC} & \multicolumn{2}{c}{CoLA} & \multicolumn{2}{c}{STS-B} \\ \hline
Algorithm  & Straggler  &  & Acc. & Runtime & Acc. & Runtime   & Acc. & Runtime  & Acc. & Runtime & Acc.  & Runtime  & Acc.       & Runtime & Acc. & Runtime & Pearson  & Runtime  \\ 
\midrule
\multirow{4}{*}{$Clip^2$}  & \multirow{2}{*}{Mild} & Server &83.30&30&86.66&31  &86.52&31&56.32 &19&93.23&22&82.35&21&79.87&17&86.36&20\\
&   & Client & \multicolumn{1}{c}{81.56} & 29 & 86.53   & 23 & 86.57 & 36 & 70.40  & 33  & 92.09  & 21 & 86.76  & 31  & 81.88 & 53  & 88.37 & 38  \\ \cline{2-19} 
& \multirow{2}{*}{Large} & Server &83.50&30&86.47&27&86.92&34&58.14&19&92.78&19&77.45&26&80.57&21&85.60&22\\
&   & Client & \multicolumn{1}{c}{82.29} & 23& 86.36 & 36 & 84.44  &59  & 66.43  & 48  & 91.63  & 25  & 86.03 & 42  & 81.78  & 370       & 87.91  & 113            \\ 
\bottomrule
\end{tabular}
}

\vspace{\baselineskip}
\end{table}

Again, for FADAS in Table~\ref{tab:glue-baseline}, we try our best to tune hyperparameters: we tune client- and server-side learning rates from (0.1, 0.01, 0.001, 0.0001) and the delay threshold in their paper $\tau_c$ from (1, 4, 8, 10).

\begin{table}[!h]
\vspace{\baselineskip}
\caption{\small Utility Metrics comparison of (SD/DC-)$Clip^2$ with   FADAS and RANSGDm.}
\label{tab:glue-baseline}
\resizebox{\textwidth}{!}{
\begin{tabular}{cccccccccccc}
\toprule
Algorithm  & Straggler & Async Mode  & MNLI  & QNLI & QQP & RTE & SST-2 & MPRC & CoLA & STS-B & Avg \\ 
\midrule
\multirow{2}{*}{$Clip^2$} & 
\multirow{2}{*}{Large} & Server & 82.58&85.78&86.09&59.95&91.97&83.58 &80.44&85.63 &82.00 \\
&   & Client & \multicolumn{1}{c}{82.19} &  86.89&  85.39&  67.87&  92.29&  85.05&  80.15&  86.92 &81.33\\ 
\multirow{2}{*}{SD-$Clip^2$} & 
 \multirow{2}{*}{Large} & Server & 82.50&87.59&  84.47 & 54.51  & 92.78 &  84.31  & 81.40  & 85.60 &81.65 \\
&   & Client & \multicolumn{1}{c}{83.35}   & 86.55   & 86.96   & 72.56   & 91.51   & 86.52   & 80.73          & 86.60   &84.35  \\ 
\multirow{2}{*}{DC-$Clip^2$} & 
 \multirow{2}{*}{Large} & Server &83.50&86.47&86.92&58.14&92.78&77.45&80.57&85.60 & 81.43\\
&   & Client & \multicolumn{1}{c}{82.29} & 86.36& 84.44   & 66.43   & 91.63  & 86.03  & 81.78    & 87.91     &83.36       \\ 
\multirow{1}{*}{FADAS \citep{wang2024fadas}} & 
 \multirow{1}{*}{Large} & N/A & 79.28 & 80.58 & 83.82 & 55.60 & 89.22 & 77.23 & 79.44 & 79.60 & 78.10            \\ 
RANSGDm~\citep{wu2026optimalasynchronousstochasticnonconvex}&Large&Server&38.58&55.29&71.93&57.04&82.34&74.51&78.04&84.40&67.77\\
\bottomrule
\vspace{\baselineskip}
\end{tabular}
}

\end{table}

\end{document}